%% file: main_icml.tex
\documentclass{article}

\usepackage{microtype}
\usepackage{graphicx}
\usepackage{subcaption}
\usepackage{booktabs} 

\usepackage{hyperref}
\makeatletter
\g@addto@macro{\UrlBreaks}{\UrlOrds}
\makeatother

\usepackage[preprint]{icml2026}

\usepackage{amsmath}
\usepackage{amssymb}
\usepackage{mathtools}
\usepackage{amsthm}
\usepackage{thmtools}
\usepackage{titletoc}

\usepackage[capitalize,noabbrev]{cleveref}

\hypersetup{hypertexnames=false} 

\theoremstyle{plain}
\newtheorem{theorem}{Theorem}[section]

\newtheorem{lemma}[theorem]{Lemma}
\newtheorem{corollary}[theorem]{Corollary}
\theoremstyle{definition}
\newtheorem{definition}[theorem]{Definition}
\newtheorem{assumption}[theorem]{Assumption}
\theoremstyle{remark}
\newtheorem{remark}[theorem]{Remark}

\usepackage{etoolbox}
\BeforeBeginEnvironment{proof}{\vspace{-0.5\baselineskip}}

\theoremstyle{plain}
\newtheorem{reptheoreminner}{\textbf{Theorem}}
\newenvironment{reptheorem}[1]{%
  \renewcommand\thereptheoreminner{#1}%
  \reptheoreminner
}{\endreptheoreminner}
\newenvironment{proofsketch}{\begin{proof}[Sketch of proof]}{\end{proof}}

\usepackage[textsize=tiny]{todonotes}

\newcommand{\ourtitle}{
Transformers Provably Learn Algorithmic Solutions for Graph Connectivity, But Only with the Right Data
}
\icmltitlerunning{\ourtitle}

\input{headers}

\input{math_commands}

\begin{document}

\twocolumn[
  \icmltitle{\ourtitle}



  \newcommand{\uscsymbol}{{\ensuremath{\textcolor[HTML]{990000}{\spadesuit}}}}
  \newcommand{\dukesymbol}{{\ensuremath{\textcolor[HTML]{00539B}{\clubsuit}}}}
  \icmlsetsymbol{equal}{$\star$}
  \icmlsetsymbol{usc}{\uscsymbol}
  \icmlsetsymbol{duke}{\dukesymbol}

  \begin{icmlauthorlist}
    \icmlauthor{Qilin Ye}{equal,usc,duke}
    \icmlauthor{Deqing Fu}{equal,usc}
    \icmlauthor{Robin Jia}{usc}
    \icmlauthor{Vatsal Sharan}{usc}
  \end{icmlauthorlist}


  \icmlcorrespondingauthor{Qilin Ye}{qilin.ye@duke.edu}
  \icmlcorrespondingauthor{Deqing Fu}{deqingfu@usc.edu}

  \icmlkeywords{Machine Learning, ICML}

  \vskip 0.3in
]



\printAffiliationsAndNotice{
    \icmlEqualContribution 
    \textsuperscript{\ensuremath{\textcolor[HTML]{990000}{\spadesuit}}}Thomas Lord Department of Computer Science, University of Southern California. \textsuperscript{\ensuremath{\textcolor[HTML]{00539B}{\clubsuit}}}Department of Computer Science, Duke University}  

\begin{abstract}
  \input{sections/0A_Abstract}
\end{abstract}

\input{sections/01_Intro}
\input{sections/02_Related}

\input{sections/03_ProblemSetup}
\input{sections/04_Theory}

\input{sections/05_Experiments}

\input{sections/06_Discussion}


\section*{Impact Statement}
This paper presents work whose goal is to advance the field of machine learning, specifically our understanding of when neural networks learn generalizable algorithms versus brittle heuristics. Our findings have potential implications for improving the reliability and interpretability of machine learning systems deployed in safety-critical applications, where algorithmic correctness is paramount. 
By identifying data distribution properties that encourage algorithmic learning, this work may inform better training practices. We do not foresee direct negative societal consequences from this foundational research.

\ifpreprint
\section*{Acknowledgments}
The authors acknowledge the Center for Advanced Research Computing (CARC) at the University of Southern California for providing computing resources that have contributed to the research results reported within this publication. We also acknowledge the use of the USC NLP cluster provided by USC NLP Group. This work used the Delta system at the National Center for Supercomputing Applications through allocation CIS250737 from the Advanced Cyberinfrastructure Coordination Ecosystem: Services \& Support (ACCESS) program, which is supported by National Science Foundation grants \#2138259, \#2138286, \#2138307, \#2137603, and \#2138296. 
DF and RJ were also supported by gifts from the USC-Capital One Center for Responsible AI and Decision Making in Finance (CREDIF) and the USC-Amazon Center on Secure and Trusted Machine Learning. RJ was also supported by the National Science Foundation
under Grant No. IIS-2403436. 
VS was supported by
an NSF CAREER Award CCF-2239265, an Amazon Research Award, a Google Research Scholar Award and
an Okawa Foundation Research Grant. 
The work was done in part while DF and VS were visiting the Simons Institute
for the Theory of Computing. Any opinions, findings, and conclusions or recommendations expressed in
this material are those of the author(s) and do not reflect the views of the funding agencies.
\fi

\bibliography{reference}
\bibliographystyle{icml2026}

\newpage
\onecolumn
\section*{Appendix}
\appendix

\input{sections/AA_Definition}
\input{sections/AB_Capacity}

\input{sections/AC_TrainingDynamics}
\input{sections/AD_Experiments}
\input{sections/AE_RelatedWork}

\end{document}

%% file: headers.tex
\usepackage{amssymb}
\usepackage[svgnames]{xcolor}
\usepackage[framemethod=TikZ]{mdframed}
\usepackage{aliascnt}
\usepackage{subcaption}
\usepackage{enumitem}
\usepackage{wrapfig}
\usepackage{etoolbox}
\usepackage{url}

\newif\ifpreprint
\ifdefined\ispreprint
  \preprinttrue
\else
  \preprintfalse
\fi

\AtBeginEnvironment{proof}{%
  \setlength{\topsep}{0pt}%
  \setlength{\partopsep}{0pt}%
}
\AtBeginEnvironment{proofsketch}{%
  \setlength{\topsep}{0pt}%
  \setlength{\partopsep}{0pt}%
}

\newcommand{\PP}{\mathbb{P}}
\newcommand{\TF}{\mathsf{TF}}

\newcommand{\ReLU}{\mathsf{ReLU}}
\newcommand{\ER}{\mathsf{ER}}
\newcommand{\TChain}{\mathsf{2Chain}}
\newcommand{\TClique}{\mathsf{2Clique}}

\newcommand{\diam}{\operatorname{diam}}

\newcommand{\Sn}{\mathcal{S}_n}

\usepackage[suppress]{color-edits}
\addauthor[Deqing]{df}{magenta}
\addauthor[Qilin]{yql}{blue}
\addauthor[Robin]{rj}{red}
\addauthor[Vatsal]{vs}{purple}

\newcommand{\vsn}{\vspace{-4pt}}
\usepackage{titlesec}
\titlespacing\paragraph{0pt}{0pt}{0pt plus 1pt minus 1pt}
\titlespacing\section{0pt}{0pt}{0pt}
\titlespacing\subsection{0pt}{0pt}{0pt}
\setlist{leftmargin=*, labelindent=0pt, topsep=0pt,itemsep=0pt}

\newcommand*{\Scale}[2][4]{\scalebox{#1}{$#2$}}%

%% file: math_commands.tex
\usepackage{amsmath,amsfonts,bm}

\def\eqref#1{equation~\ref{#1}}

\def\1{\bm{1}}

\DeclareMathAlphabet{\mathsfit}{\encodingdefault}{\sfdefault}{m}{sl}
\SetMathAlphabet{\mathsfit}{bold}{\encodingdefault}{\sfdefault}{bx}{n}

\newcommand{\R}{\mathbb{R}}



%% file: sections/0A_Abstract.tex
Transformers often fail to learn generalizable algorithms, instead relying on brittle heuristics. 
Using graph connectivity as a testbed, we explain this phenomenon both theoretically and empirically.
We consider a simplified Transformer architecture, the Disentangled Transformer, and prove that an $L$-layer model can compute connectivity in graphs with diameters up to $3^L$, implementing an algorithm equivalent to computing powers of the adjacency matrix.
By analyzing training dynamics, we prove that whether the model learns this strategy hinges on whether most training instances are within this model capacity.
Within-capacity graphs (diameter $\leq 3^L$) drive the learning of the algorithmic solution while beyond-capacity graphs drive the learning of a simple heuristic based on node degrees.
Finally, we empirically show that restricting training data to stay within a model's capacity makes both standard and Disentangled Transformers learn the exact algorithm.


%% file: sections/01_Intro.tex
\section{Introduction}
Large language models (LLMs) based on the Transformer architecture have demonstrated remarkable capabilities, yet their success is often shadowed by failures on tasks that demand robust, algorithmic reasoning. A growing body of evidence shows that, instead of learning generalizable algorithms, these models frequently rely on brittle shortcuts and spurious correlations that exploit statistical cues in the training data \citep{niven-kao-2019-probing,geirhos2020shortcutlear,tang2023lazylearners,yuan2024shortcutsuite,zhou2024conceptspurious,ye2024spurioussurvey}. This shortcut reliance contributes to poor out-of-distribution (OOD) generalization, brittleness under superficial input changes, and unreliability on multi-step reasoning tasks \citep{zou2023universalattack,deng2024contamination,li2024latesteval,mirzadeh2025gsmsymbolic}. 
On deterministic tasks like shift ciphers, LLMs favor high-probability outputs over correct solutions \citep{mccoy2023embersllm};
in mathematical problem solving, strong in-distribution scores often fail to transfer as problem structure or size shifts \citep{saxton2018analysing,kao2024complexmath,zhou2025mathforai}. This motivates a foundational question: 
\vsn\vsn
\begin{center}
    \textit{\textbf{When and why do Transformers learn heuristics over verifiably correct algorithms, even when the task admits an algorithmic solution?}}
\end{center}
\vsn\vsn


To study this question, we adopt \textit{graph connectivity} as a controlled testbed.
Connectivity is a fundamental problem in computational complexity \citep{wigderson1992complexity} that combines three properties suited to our analysis.
First, it admits an unambiguous algorithmic solution: reachability equals the transitive closure and is computable via matrix powering \citep{warshall1962,floyd1962}. Second, recent theory shows that Transformers with depth $L = \Theta(\log n)$ can express this algorithm through matrix powering constructions \citep{merrill2025logdepth}, so the architecture is provably expressive enough. Third, connectivity admits simple heuristics based on node degrees that correlate with the correct answer on typical random graphs but fail on adversarial instances. This makes it possible to test whether training recovers the algorithm or settles for the shortcut.

Despite these expressivity guarantees, whether gradient descent actually finds the algorithmic solution remains open. Our experiments reveal that standard Transformers achieve perfect in-distribution accuracy on random graphs yet fail catastrophically on simple OOD instances such as two disjoint chains (see \S\ref{ssec:prelim} and \Cref{fig:roberta}). To understand both the failure mode and how to correct it, we analyze the \emph{Disentangled Transformer} \citep{friedman2023ltp,nichani2024transformerslearncausalstructure}, a simplified architecture that is more amenable to theoretical analysis while preserving the essential computations. We summarize our contributions below.

\paragraph{An $L$-layer model solves connectivity up to diameter $3^L$, but no further. }
We prove tight bounds that characterize model capacity in terms of graph diameter rather than the number of nodes. Let $\diam(G)$ denote the maximum shortest-path distance between any two connected nodes. On the expressivity side, we show that an $L$-layer Disentangled Transformer can solve connectivity on all graphs with $\diam(G) \le 3^L$ by implementing a matrix powering algorithm (\Cref{thm:expressivity}). On the limitations side, we prove a matching upper bound: for any choice of nonnegative weights, there exists a graph with diameter $3^L + 1$ on which the model fails (\Cref{thm:capacity}). Together, these results establish that $3^L$ is the maximum diameter an $L$-layer model can handle perfectly. We call this quantity the model's \emph{capacity}. We empirically validate this diameter-depth scaling on both disentangled and standard Transformers.

\paragraph{Learned weights decompose into algorithmic and heuristic channels. } We prove that under natural symmetry assumptions, specifically invariance to relabeling of graph vertices, the weights of a trained Disentangled Transformer decompose into two functionally distinct components (\Cref{thm:w_decomposition}). The \emph{algorithmic channel} preserves locality and implements multi-hop composition by computing powers of the adjacency matrix. The \emph{heuristic channel} ignores graph structure and instead computes global statistics based on node degrees, predicting connectivity from whether two nodes both have high degree. We verify empirically that trained models satisfy the required symmetry, making this decomposition applicable to practical settings (\S\ref{story:algo-heur}).


\paragraph{Training dynamics select between channels based on the data distribution. }
Our analysis of the training dynamics reveals a sharp dichotomy driven by the data distribution. For graphs within the model's capacity (diameter $\leq 3^L$), population gradients suppress the heuristic channel and favor the algorithmic channel that implements matrix powering (\Cref{thm:training-suppresses-J-debug}). Conversely, when the distribution contains a significant share of beyond-capacity graphs (diameter $> 3^L$), the gradients instead strengthen the heuristic channel, promoting the simple degree-counting shortcut (\Cref{thm:dynamics}). This precise characterization hinges on our exact $3^L$ capacity bound; an asymptotic one, such as the $\mathcal O(\exp(L))$ result from \citet{merrill2025logdepth}, would not yield such clear predictive implications.

\paragraph{Transformers learns algorithmic solution \emph{with the right data}. } These theoretical insights point to a direct mitigation strategy we call the \textit{data lever}: restricting the training data exclusively to within-capacity graphs. Our experiments in \S\ref{sec:exp} confirm the effectiveness of this approach, showing that it boosts the algorithmic component and improves OOD robustness (\Cref{fig:restrict_dt}), and that these benefits transfer successfully to standard Transformer models (\Cref{fig:roberta_restrict_diam}).

Our results, validated on both Disentangled and standard Transformers  (\Cref{fig:capacity,fig:capacity_roberta_2layer,fig:roberta_restrict_diam}), provide a precise account of how Transformers learn algorithms given the right data.



\begin{figure}[t]
  \centering
  \includegraphics[width=0.95\linewidth]{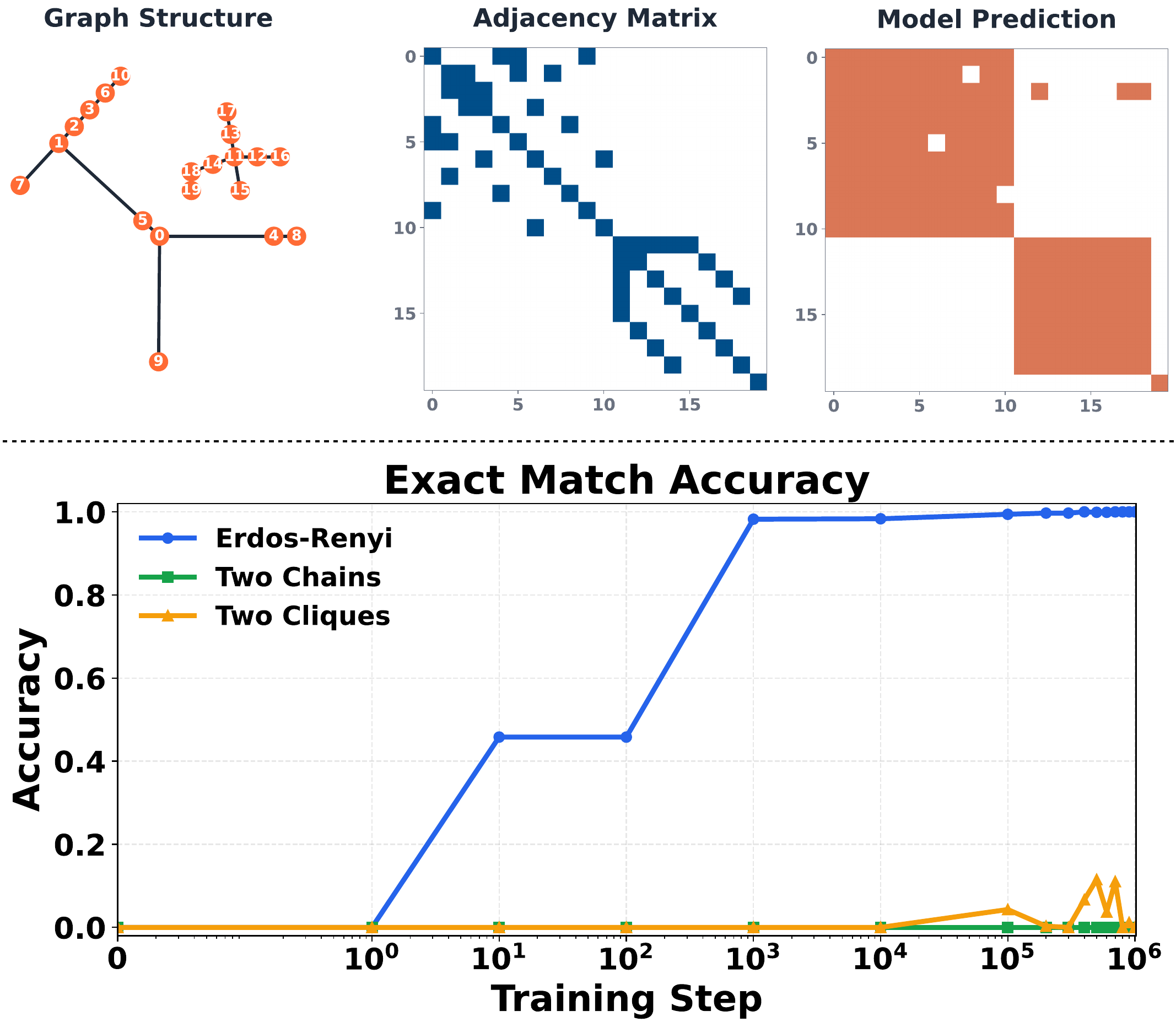}
  \caption{We train 2-layer Transformer models on Erd\H{o}s-R\'enyi graphs. (\textbf{Top}) Visualization of a graph, its input adjacency matrix, and the model's (mis-)prediction of its connectivity. (\textbf{Bottom}) Although trained models are able to predict perfectly on every edge within distribution, they failed to generalize to out-of-distribution graphs such as graphs with two isolated chains or cliques.}
  \vsn
  \label{fig:roberta}
\end{figure}

%% file: sections/02_Related.tex
\section{Related Work}
Theoretical analyses aim to define what Transformers can and cannot compute.
Although Transformers are universal approximators for continuous sequence-to-sequence functions \citep{yun2020universal}, they also face sharp complexity-theoretic limits.
Fixed-depth attention struggles with periodic or hierarchical patterns \citep{hahn2020}, and standard Transformers are restricted to the complexity class $\mathsf{TC}^0$ \citep{merrill2023parallelism}, with hard-attention variants also confined to low-level circuit classes \citep{hao2022hardattention,barcelo2024hardattention}. 
Allowing model depth to scale logarithmically with input length enables solving graph connectivity \citep{merrill2025logdepth,sanford2024graphreasoning}.
Chain-of-thought also increases expressivity \citep{merrill2024cot,feng2023towards} but we exclude it here to focus only on what the base architecture learns through gradient descent.
Programmatic abstractions like $\mathsf{RASP}$ offer another lens, identifying which algorithms can be implemented in a length-generalizing way \citep{weiss2021rasp,zhou2023lengthgen}.
Empirically for the graph connectivity problem, \citet{fu2024isobench} shows frontier LLMs can reach almost perfect performance on small graphs and \citet{saparov2025search} shows Transformers have greater difficulty in learning the task when graph size increases.
Our matrix powering view is analogous to message passing in GNNs, where $L$ layers reach $L$-hop neighborhoods \citep{hamilton2020graph}; attention achieves $3^L$ hops by tripling the range at each layer.
We include more related work in the appendix \S\ref{app:related_work}.

%% file: sections/03_ProblemSetup.tex
\section{Problem Setup and Preliminary Study} \label{sec:problem}
\subsection{Graph Connectivity Task}
\begin{definition}[Self-loop-augmented adjacency matrix]
Let $G=(V,E)$ be a graph with $n$ vertices. We define the \textbf{self-loop-augmented adjacency matrix} $A \in \{0,1\}^{n \times n}$ as
$A_{i,j} = 1$ if $\{v_i, v_j\} \in E$ or $i=j$, and 0 otherwise.
\end{definition}
This definition is equivalent to taking the standard adjacency matrix and adding the identity matrix. A key consequence is that the $(i,j)$-th entry of the matrix power $A^k$ counts the number of walks of length $k$ from $v_i$ to $v_j$. With self-loops, these walks may stay at the same vertex from one step to the next. Henceforth, ``adjacency matrix'' will refer to this self-loop-augmented version.

\begin{definition}[Connectivity]
    For any graph $G = (V, E)$ with $n$ nodes, we define the connectivity matrix $R \in \{0,1\}^{n\times n}$ as follows: $
    R_{i,j} = 1$ if there is a path between $v_i$ and $v_j$ and $0$ otherwise.
    In particular, $
    R_{i,j} = 1 \textrm{ if and only if } [A^n]_{i,j} > 0
    $.
\end{definition}

Our learning objective is to learn models $\mathcal M: \{0,1\}^{n \times n} \rightarrow \mathbb R^{n \times n}$. For a graph $G$ with adjacency matrix $A$ and  connectivity matrix $R$, if $\mathcal M$ satisfies $[\mathcal M(A)]_{i,j} > 0 \Leftrightarrow R_{i,j} = 1$, then we say $\mathcal M$ is \emph{perfect} on graph $G$.  
We train on Erd\H{o}s-R\'enyi $\ER(n,p)$ graphs, i.e., graphs on $n$ vertices where each edge is present independently with probability $p$.

\subsection{Transformer Architectures}
We first introduce our setups on standard transformer models without causal attention masking.
\begin{definition}[Transformers for Graph Connectivity]
\label{def:Transformers}
Let $A$ be the self-loop-augmented adjacency matrix of $G$.
Fix depth $L$ and hidden width $d>n$.
Define the linear read-in and read-out maps
\[
\begin{aligned}
\mathsf{ReadIn}(X) &:= X W_{\mathrm{in}},\quad W_{\mathrm{in}}\in\mathbb R^{n\times d},\\
\mathsf{ReadOut}(H) &:= H W_{\mathrm{out}}^\top,\quad W_{\mathrm{out}}\in\mathbb R^{n\times d}.
\end{aligned}
\]
An $L$-layer single-head transformer model acts as
\[
\;\TF^L_\Theta(A) \;:=\; \mathsf{ReadOut}\!\left(\mathsf{Transformers}^L\!\big(\mathsf{ReadIn}(A)\big)\right)\;
\]
where $\mathsf{Transformers}^L$ is a standard pre-norm Transformer with self-attention and with \textit{no} causal attention masks. There is no additional positional encoding since $I_n$ is already added to the input $A$ as the absolute positional encoding. A full definition can be found in \Cref{def:roberta_full}. 
\end{definition}


\subsection{Preliminary Study} \label{ssec:prelim}

We train 2-layer Transformer models on $\ER(n=20, p=0.08)$ graphs and test them on two out-of-distribution datasets: (1) $\TChain(n = 20, k=10)$ graphs with $n$ nodes consisting of two isolated chains each with $k$ nodes, and (2) $\TClique(n =20, k=10)$ graphs with $n$ nodes consisting of two isolated $k$-Cliques. We measure the performance of model $\mathcal M$ via an exact match accuracy on our graph distribution $\mathcal G$, i.e., the fraction of graphs on which $\mathcal{M}$ is perfect, or
$
    \mathsf{ExactMatchAcc}(\mathcal M, \mathcal G) = \operatorname{\mathbb E}_{G = (V, E) \in \mathcal G} \left[\prod_{v_i, v_j \in V} \mathbf 1\left\{[\mathcal M(A_G)]_{i,j} = [R_G]_{i,j} \right\}\right].
$

\textbf{Transformers Fail to Generalize.} As shown in \Cref{fig:roberta}, the 2-layer Transformer model is able to achieve almost perfect exact match accuracy on the held-out set of the training distribution. However, it  fails to learn an algorithmic solution that transfers to other distributions. When the model is tested on the $\TChain$ and $\TClique$ distributions, its exact match accuracy falls to nearly zero, indicating over-fitted heuristics have dominated the model prediction. We repeat the experiments via extensive hyperparameter search and scaling up the number of layers, but all models fail to generalize. This motivates us to investigate why transformers prefer to learn brittle heuristics and how we can encourage them to learn algorithmic solutions instead.

%% file: sections/04_Theory.tex
\section{Theory} \label{sec:theory}

\subsection{Disentangled Transformer}
To understand the generalization failure in \S\ref{ssec:prelim} theoretically, we pivot to a  simplified \emph{\textbf{Disentangled Transformer}}; this helps us not only with expressivity/capacity analysis in \S\ref{ssec:expressivity_capacity} but also with training dynamics analysis in \S\ref{ssec:training_dynamics}. In the Disentangled Transformer, each attention block appends its output as a new coordinate slice of the residual stream rather than summing, so the representation dimension grows with depth and the read/write pathways become traceable \citep{friedman2023ltp,nichani2024transformerslearncausalstructure}. 
This model serves as a reasonable proxy for its standard counterpart: \citet{nichani2024transformerslearncausalstructure} show that any standard attention-only Transformer can be re-expressed as a disentangled model by specializing attention to implement feature concatenation, and \citet{chen2024unveiling} adopt this architecture precisely because it preserves the computations of interest while being markedly more amenable to theoretical analysis. 
We now formalize the model. 
\begin{definition}[Disentangled Transformer for Graphs]
\label{def:disentangle}
    Let $n$ be the number of nodes for any graph $G$ with adjacency matrix $A \in \mathbb R^{n \times n}$. Let $L$ be the depth of the Disentangled Transformer, and $\{d_0, d_1, \cdots, d_L\}$ be the set of dimensions of its hidden states with $d_\ell = 2^{\ell+1} n$. Let $\{W_\ell\}_{\ell=1}^L$ be the attention matrices with $W_\ell \in \mathbb R^{d_{\ell-1} \times d_{\ell-1}}$. Let $W_O \in \mathbb R^{n \times d_L} = [I_{n}, \cdots, I_n]$ be the output matrix. Let $\Theta = \{W_\ell\}_{\ell=1}^L$. An $L$-layer Disentangled Transformer $\TF^L_\Theta$ acts on any graph's self-loop augmented adjacency matrix $A$ by
    \begin{equation*}
        \begin{aligned}
          &\textrm{\textbf{Input hidden state }} h_0 := [I_n, A] \in \mathbb R^{n \times d_0} \\
          &\textrm{\textbf{$\ell$-th Hidden state }}  h_{\ell} := \left[h_{\ell - 1}, \mathrm{Attn}\left(h_{\ell-1}; W_\ell\right)\right] \in \mathbb R^{n \times d_\ell}\\
           &\textrm{\textbf{Output layer }} \TF^L_\Theta(A) := h_L W_O^\top 
        \end{aligned}
    \end{equation*}
    where $\mathrm{Attn}\left(h_{\ell-1}; W_\ell\right) := \frac{1}{n} \ReLU\left(h_{\ell-1} W_\ell h_{\ell-1}^\top\right) h_{\ell-1}$. We remark that $h_\ell \in \mathbb R^{n \times d_\ell}$ where $d_\ell = 2^{\ell+1}n$ grows exponentially with respect to $\ell$.  
\end{definition}

\subsection{Expressivity and Capacity} \label{ssec:expressivity_capacity}
If a 2-layer Transformer fails to generalize in \S\ref{ssec:prelim}, should we attribute this to the architecture's expressivity? We argue not. \Cref{thm:expressivity} shows that an $L$-layer Disentangled Transformer can implement the matrix powering algorithm and is perfect on graphs of diameter at most $3^L$. Moreover, \Cref{thm:capacity} shows this $3^L$ threshold is tight and exact\footnotemark. To make this precise, we first formalize graph distance and diameter in \Cref{def:diameter}.

\footnotetext{Because the $\TChain(n = 20, k=10)$ graphs have maximum path length $9$, they should be theoretically learnable by $2$-layer models, unlike in \S\ref{ssec:prelim}.}

\begin{definition}[Graph distances and diameter]
\label{def:diameter}
Let $G=(V,E)$ be a finite, simple, undirected graph. Following standard definitions, for $u,v\in V$, we let $d_G(u,v)$ be the \emph{shortest-path distance} between $u,v$, which is finite if they are connected and infinite otherwise. For a connected component, we define its \emph{diameter} to be the longest path length within the component.

Throughout, we define the diameter of a graph, denoted $\mathrm{diam}(G)$, to be the maximum diameter among its connected components. Note this differs from the common convention on disconnected graphs, where the latter sets $\max_{u,v}d_G(u,v) = \infty$; ours is always finite.
\end{definition}

We begin by establishing the expressive power of the Disentangled Transformer, showing that with sufficient depth, it can implement the correct matrix powering algorithm to solve connectivity.

\begin{theorem}[Expressivity of $\TF_{\Theta }^L$]
\label{thm:expressivity}
    There exists an $L$-layer Disentangled Transformer that makes perfect predictions for every graph $G$ satisfying $\mathrm{diam}(G) \leq 3^L$.
\end{theorem}
\vsn
\begin{proofsketch}
    For all $\ell$, setting $W_\ell = I_{d_{\ell-1}}$ suffices. These choices of weights implement the matrix powering algorithm $\sum_{j=0}^{\diam(G)} \alpha_j A^j$ with positive coefficients $\alpha_j$.
\end{proofsketch}
\vsn

We next show a capacity bound that reveals the model's inherent limitations. We prove a tight, non-asymptotic upper bound on the graph diameter an $L$-layer model can handle, linking model depth directly to instance difficulty. To state this precisely, we define capacity as follows.

\begin{definition}[Model Capacity]
The \emph{capacity} of an $L$-layer Disentangled Transformer $\TF_{\Theta }^L$ is the largest integer $d$ such that there exist weights achieving perfect predictions on every graph $G$ with $\diam(G) \le d$.
\end{definition}

Informally speaking, the capacity of an $L$-layer Disentangled Transformer $\TF_{\Theta }^L$ with nonnegative weights is $3^L$. We formalize this claim as follows.
\begin{theorem}[Capacity of $\TF_{\Theta }^L$]
\label{thm:capacity}
    Let $\TF_{\Theta }^L$ be an $L$-layer Disentangled Transformer on $n = \Omega(3^L)$ nodes\footnotemark. Further assume that the weights $W_{\ell} \geq 0$ for each $\ell$. Then there exists a graph $G$ on which $\TF_{\Theta }^{L}(A)$ does not equal the connectivity matrix $R$. Further, $G$ has diameter $3^{L} + 1$. In other words, $\TF_{\Theta }^{L}$ cannot master graph connectivity beyond path length $3^L$, and they cannot fully solve the task on $\Omega(3^L)$-node graphs. 
\end{theorem}
\footnotetext{In particular, taking $n\geq (7 /3) \cdot 3^{L} + 2$ suffices.}
\vsn

\begin{figure}[t]
    \centering
    \includegraphics[width=0.95\linewidth]{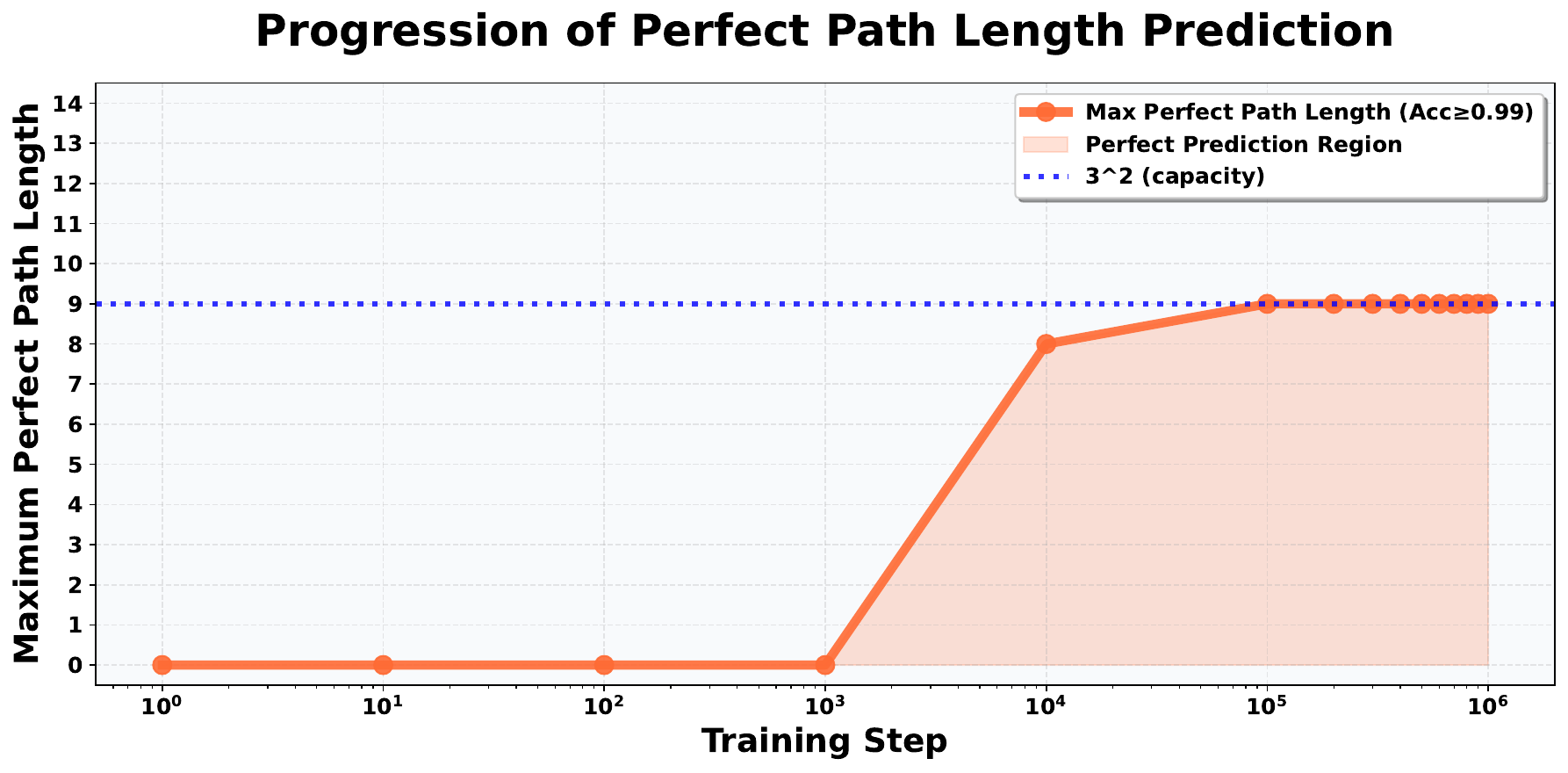}\\
    \includegraphics[width=0.95\linewidth]{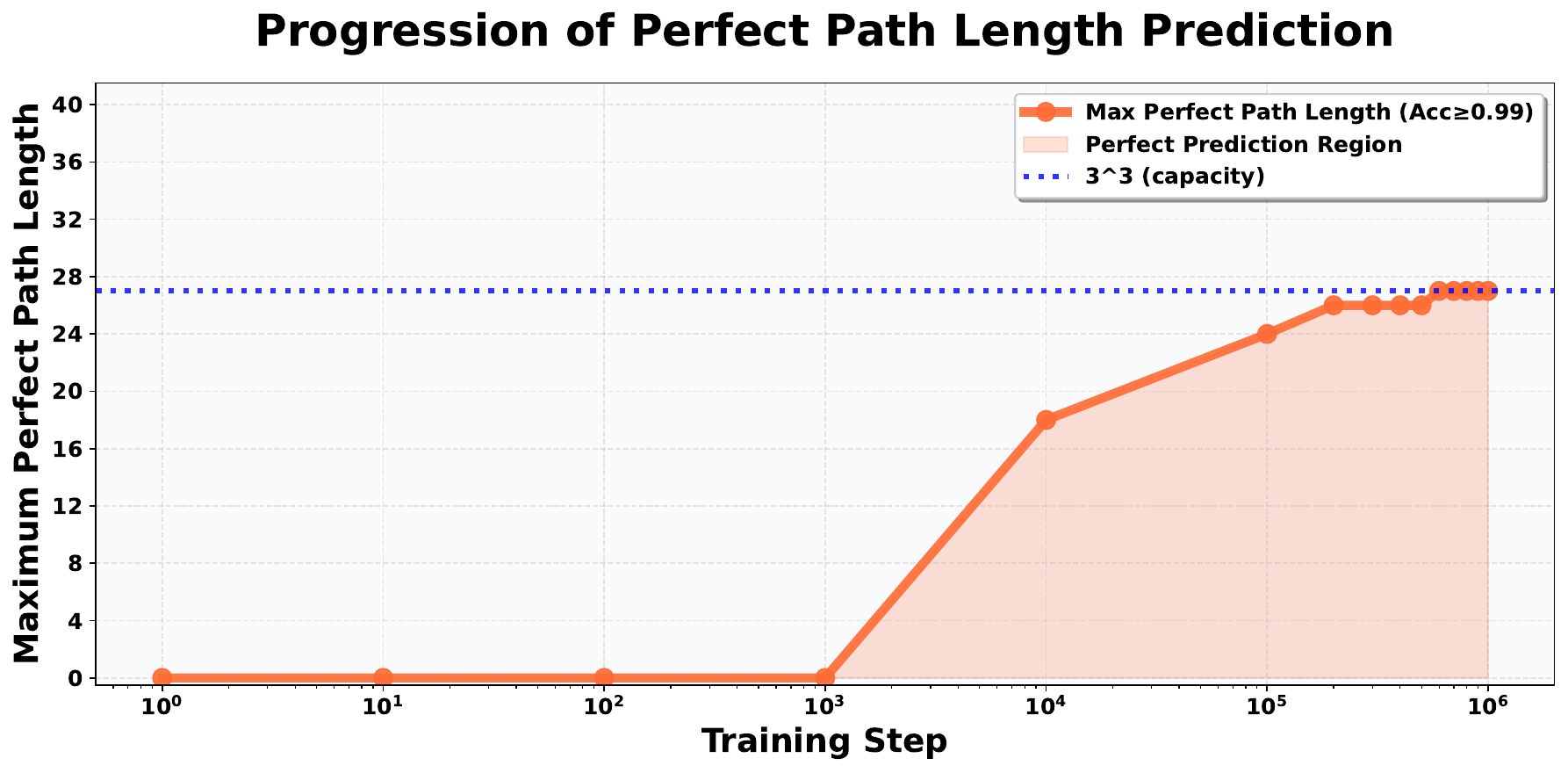}
    \vsn\caption{\textbf{Capacity of Disentangled Transformers.} We train 2-layer (\textbf{top}) and 3-layer (\textbf{bottom}) Disentangled Transformers on $\ER(n=24)$ and $\ER(n=64)$ graphs respectively. When evaluated on hold-out sets, both models can only make reliable predictions ($\geq 99\%$ accuracy) on node pairs $u, v$ if and only if $d_G(u,v) \leq 3^L$. These findings resonate with our theoretical observations in \Cref{thm:capacity}.}\vsn
    \label{fig:capacity}
\end{figure}

\begin{proofsketch}
We split by whether a false positive across different connected components occurs at some intermediate layer; the full proof can be found in \Cref{app:capacity}.

\textit{Case 1 (False positive occurs \Cref{lem:teleportation}).} Suppose for a graph $H$, an intermediate layer contains a false positive in its hidden states. That is, for two disconnected nodes $u,v$, the corresponding $(u,v)$ entry in some $\ell^{\text{th}}$ layer hidden states is positive (and this will propagate to the final output, whereas the $(u,v)$ entry in the ground-truth connectivity entry isn't). We isolate two sets of nodes that contribute to this false positive by backtracking the computation DAG. By making appropriate changes, we argue for the existence of a graph $G$ that (i) preserves this false positive on $(u,v)$ and (ii) contains a manually created path of length $3^L + 1$.


\textit{Case 2 (No false positives \Cref{lem:no-teleportation})}. Suppose now that no intermediate layer has false positives for any $n$-node graph. We show that ``information'' spreads no faster than $3^L$ so that it never predicts ``Yes'' on node pairs with distance beyond $3^L$. We first apply the no-false-positives assumption on the empty (self-loops-only) graph. Inductively, each column of each hidden states is supported on exactly one row, which ranges from $1$ to $n$. This naturally gives a ``label'' for each column in each hidden states. The crux of the proof is to inductively show that at layer $\ell$, two columns can ``share'' information, thereby creating a positive score on $(u,v)$, only if their labels, interpreted as graph nodes, are within distance $3^\ell$. Consequently, a no-false-positive model cannot recognize a connected pair with distance $3^L + 1$.

In both cases one can construct graphs with diameters $3^L + 1$ on which $\TF_\Theta^L$ is not perfect.
\end{proofsketch}

Given the tight $3^L$ capacity bounds for Transformers, it is natural and crucial to introduce a dichotomy around the $3^L$ capacity. For any connected node pair $(u,v)$, they are said to be within capacity if $d_G(u,v) \leq 3^L$ and beyond capacity otherwise. Formally, we define the dichotomy as follow:
\begin{definition}[Within-capacity and beyond-capacity pairs at depth $L$]
\label{def:capacity_dichotomy}
Fix a graph $G$ and a depth $L$. We say a pair of nodes $(i,j)$ is \textbf{within capacity} if $[A^{3^L}]_{i,j} > 0$ and \textbf{beyond capacity} otherwise. In other words, a pair $(i,j)$ is within capacity iff their shortest-path distance is $\leq 3^L$.

Overloading the notation, we say $G$ is a within-capacity graph if $\mathrm{diam}(G) \le 3^L$ and beyond-capacity otherwise. 
\end{definition}

\subsection{Training Dynamics} \label{ssec:training_dynamics}
\begin{figure}[t]
    \centering
    \includegraphics[width=0.95\linewidth]{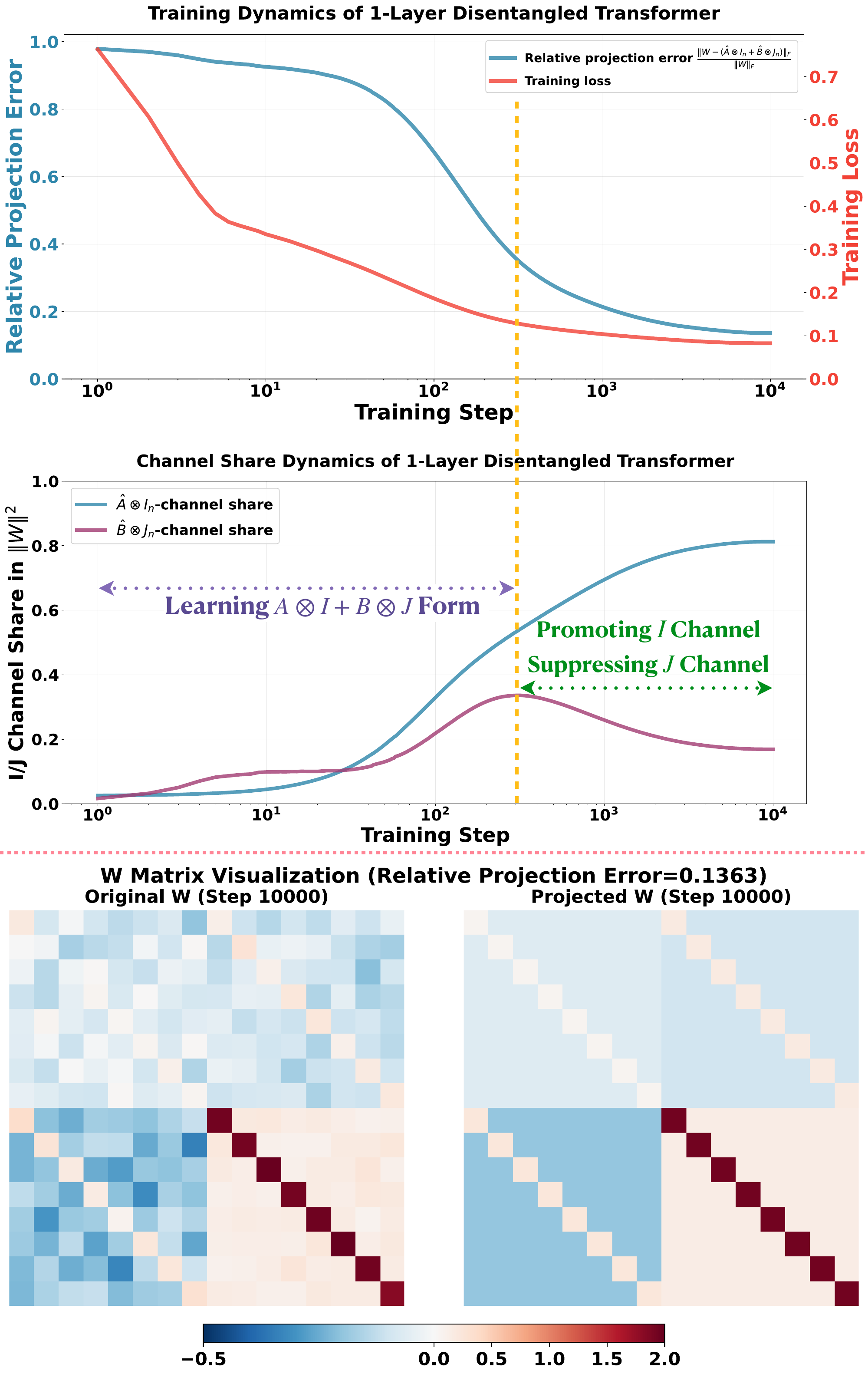}
    \vsn
    \caption{\textbf{Training Dynamics of Disentangled Transformers.} We train a 1-layer Disentangled Transformer on graphs from $\ER(n=8, p=0.2)$ distribution. Weight $W$ will approximately approach to $A \otimes I_n + B \otimes J_n$ form. (\textbf{Top}) There are two major phases during training, where during Phase 1, model focuses on learning the equivariant parameterizations so both $I$ and $J$ channel's share of energy (see \S\ref{ssec:expt_capacity}) in $W$ increases, and during Phase 2, the algorithmic $I$-channel is promoted and the heuristic $J$-channel is suppressed. (\textbf{Bottom}) Visualization of the learned weights and its projection to the closest $\hat{W} = \hat{A} \otimes I_n + \hat{B} \otimes J_n$ form. } 
    \label{fig:1-layer-dynamics}
    \vsn
\end{figure}

If a capable 2-layer Transformer is able to perfectly predict connectivity up to path length $3^2=9$, and the $\TChain(n = 20, k=10)$ dataset does not contain longer paths, why didn't the Transformer model in \S\ref{ssec:prelim} learn the algorithm? In this section, we show that this is because the training distribution contains too many graphs beyond the $3^L$ capacity, and those samples reward a learning a global heuristic over an algorithm. Equipped with \Cref{thm:w_decomposition}, we can analyze the gradient dynamics in the two-channel parameterization (a superposition of heuristic and algorithmic channels). \Cref{thm:training-suppresses-J-debug} and \Cref{thm:dynamics} give a simple criterion made possible by the exact $3^L$ characterization: if within-capacity pairs dominate, the algorithmic channel wins; if beyond-capacity pairs prevail, the heuristic wins.

\textbf{Parameterizing Model Weights.} To analyze the gradient dynamics, we first identify the relevant parameter space. Our data and targets are symmetric under node relabeling: the ground truth mapping maps $A\mapsto R$ and $PAP^\top \mapsto PRP^\top$ for any permutation $P$. We also empirically observe that models trained from scratch on these graphs rapidly converge to a layerwise equivariant state (\Cref{fig:dt_n64_model_behavior}). Based on these observations, we analyze the dynamics within the subspace of layerwise permutation-equivariant weights. The following theorem formally defines layerwise equivariance and characterizes exactly what weights look like.

\begin{theorem}[Layerwise Permutation-Equivariant Parameterization]
\label{thm:w_decomposition}
    Suppose an $L$-layer Disentangled Transformer $\TF_\Theta^L$ has non-negative weights. Let $K_{\ell-1} = 2^\ell$. Then $\TF_\Theta^L$ is layer-wise permutation equivariant, i.e., for each $\ell$, any hidden states $h \in \mathbb{R}^{n \times d_{\ell-1}}$, and any permutation $P\in S_n$,
		\[
        \Scale[0.9]{
				\mathrm{Attn} (P h(I_{K_{\ell-1}}\otimes P^\top); W_\ell) = P \; \mathrm{Attn} (h; W_\ell)\; (I_{K_{\ell-1}}\otimes P^\top),
        }
		\] 
        if and only if each layer weight $W_\ell$ decomposes as $ W_\ell = A_\ell \otimes I_n + B_\ell \otimes J_n$ for some $A_\ell, B_\ell \in \mathbb R^{2^\ell \times 2^\ell}$ for all $\ell$, where $\otimes$ denotes the Kronecker product and $J_n = \mathbf{1}\mathbf{1}^T$ the all-ones $(n\times n)$ matrix.
\end{theorem}

\ifpreprint
\begin{proofsketch}
    Sufficiency is immediate: If $W_\ell$ admits this form, then conjugating by any node permutation $P$ leaves both factors invariant, as $PI_nP^\top = I_n$ and $P J_nP^\top = J_n$. 

    For necessity, the key observation is that with ReLU inactive due to non-negativity assumption, the layer map becomes bilinear in $h$. With algebra, the equivariance assumption can be shown to imply a conjugation-alike identity on \textit{weights}: Writing $\sigma(P) = I_{K_{\ell-1}}\otimes P$ and $\Delta = \sigma(P) W_\ell \sigma(P)^\top - W_\ell$, the following must hold:
    \[
        h \Delta h^\top h \sigma(P) = 0 \qquad \text{ for all } h \ge 0.
    \]
    We then argue that this forces $\Delta = 0$, i.e., $W_\ell$ is conjugation-invariant, by testing the above equation with special matrices $P$. Next, we inspect small blocks $W_\ell [u,v]$ of size $n\times n$ in $W_\ell \in \mathbb{R}^{(2^\ell n) \times (2^\ell n)}$, and argue that $W_\ell[u,v]$ must commute with all permutations $P$. This forces each $W_\ell[u,v]$ to lie in $\mathrm{span}(I_n, J_n)$. Aggregating all subblocks, $W_\ell$ can therefore be decomposed as $A_\ell \otimes I_n + B_\ell \otimes J_n$.
\end{proofsketch}
\fi

It immediately follows that this subspace contains a canonical algorithmic solution (e.g. the identity construction $W = I$ used in \Cref{thm:expressivity}). Furthermore, \Cref{thm:good-model-parameterization} (in Appendix) shows that for \textit{any} capacity-reaching model, the symmetric component of the weights, which drives the attention mechanism, must lie purely in the $I_n$-channel. Finally, algebraically this parameterization is closed under gradients (\Cref{thm:population-gradient}), and this allows us to decompose the learning process into the competition between two functionally distinct channels, discussed next. 

\label{story:algo-heur}
Under the conditions of \Cref{thm:w_decomposition}, each layer weight decomposes as $W_\ell = A_\ell\otimes I_n + B_\ell\otimes J_n$. This decomposition separates the model's computation into two functionally distinct channels with different roles.

\paragraph{The $I_n$-channel ($A_\ell\otimes I_n$) compute algorithms.} 
    
    The term $A_\ell \otimes I_n$ preserves locality: when applied within the attention mechanism, it only combines features between nodes that share graph neighbors. Across layers, this channel composes multi-hop information by effectively computing powers of the adjacency matrix $A$. After $L$ layers, the readout aggregates a weighted sum $\sum_{j=0}^{3^L} \alpha_j A^j$ with nonnegative coefficients, which is precisely the matrix powering algorithm that solves connectivity (\Cref{thm:expressivity}).
\paragraph{The $J_n$-channel ($B_\ell\otimes J_n$) collect heuristics.} The term $B_\ell \otimes J_n$ broadcasts information globally, ignoring graph structure. To see why, observe that $J_n$ is rank-one: for any vector $x \in \mathbb{R}^n$, we have $J_n x = (\mathbf{1}^\top x)\mathbf{1}$, which computes the sum of entries and broadcasts it uniformly to all nodes. When composed with the adjacency matrix, $A J_n = \mathbf{d}\mathbf{1}^\top$ where $\mathbf{d} = A\mathbf{1}$ is the degree vector. Thus, this channel computes global statistics, specifically products of node degrees and their higher-order generalizations. Such statistics correlate with connectivity on dense random graphs but fail on adversarial instances. For example, two disjoint cliques both have high-degree nodes, yet no edge connects them.

\textbf{Training Dynamics.} Under this algorithmic-herustic dual-channel view, we can track the evolution of the two parameters, $A_\ell$ and $B_\ell$. To rigorously analyze the gradient descents, we adopt the following assumptions.

\begin{assumption} \label{asm:training_dynamics}
We make assumptions on data distribution, model parameterization and the training loss. 
\begin{enumerate}
    \item \textbf{Data Distribution.} Let $\ER(n,p)$ be the Erd\H{o}s-R\'enyi distribution with edge-probability $p\in(0,1)$. Then $\PP_{G\sim\ER(n,p)}\{G\ \text{is disconnected}\} > 0$.
    \item \textbf{Nonnegativity \textit{\&} Equivariant Parameterization.} For each layer $\ell$, assume $W_\ell \geq 0$ can be decomposed as $ W_\ell \;=\; A_\ell\otimes I_n \;+\; B_\ell\otimes J_n$ for some $A_\ell,B_\ell\in\R^{2^\ell \times 2^\ell}.$
    \item \textbf{Surrogate Loss.} Given scores $Z := \TF_\Theta^L(\cdot) \in\R_{\ge 0}^{n\times n}$, define the link $\phi(z):=1-e^{-\alpha z}$ with $\alpha>0$;\footnotemark  the entrywise cross-entropy with respect to the connectivity matrix $R$ is $
\mathcal{L}(Z;R)\ :=\ -\sum_{i,j}\big(R_{i,j}\log \phi(Z_{i,j}) + (1-R_{i,j})\log (1-\phi(Z_{i,j}))\big).$
Its gradient with respect to $Z$ is $\frac{\partial \mathcal{L}}{\partial Z}\ =\ \alpha\,(1 - R / \phi(Z))\ \in \ \R^{n\times n}.$
\end{enumerate}
\end{assumption}
\footnotetext{It is possible that $R_{i,j} = 1$ while $Z_{i,j} = 0$, resulting in undefined gradient $\partial \mathcal{L} / \partial Z$. To circumvent this, we approximate via $\phi_\epsilon = 1 - (1-\epsilon) e^{-\alpha z}$.  All subsequent analyses hold verbatim by replacing $\phi$ with $\phi_\epsilon$.}

Under these assumptions, we can characterize the convergence and limiting behavior of gradient descent.

\begin{theorem}[Convergence to KKT Points]
\label{thm:convergence}
Let $\mathcal{R}(\Theta) := \mathbb{E}_{G \sim \ER(n,p)}[\mathcal{L}(\TF_{\Theta}^{L}(A_G); R_G)]$ denote the population risk. For $\lambda > 0$, define the regularized objective $\mathcal{R}_\lambda(\Theta) := \mathcal{R}(\Theta) + \frac{\lambda}{2}\|\Theta\|_F^2$. Let $\mathcal{C} := \{(A_\ell, B_\ell)_{\ell} : A_\ell \ge 0, B_\ell \ge 0, \forall \ell\}$ denote the constraint set, and consider the sequence $\{\Theta^{(k)}\}_{k \ge 0}$ generated by projected gradient descent on $\mathcal{R}_\lambda$:
\begin{equation}
    \Theta^{(k+1)} = \Pi_{\mathcal{C}}\left(\Theta^{(k)} - \eta \nabla \mathcal{R}_\lambda(\Theta^{(k)})\right),
\end{equation}
with step size $\eta > 0$ sufficiently small and initialization $\Theta^{(0)} \in \mathcal{C}$ of the form $W_\ell = A_\ell \otimes I + B_\ell \otimes J$. Then every limit point $\Theta^*_\lambda \in \mathcal{C}$ satisfies the KKT conditions:
\begin{equation}
\begin{aligned}
    \nabla_{B_\ell} \mathcal{R}(\Theta^*_\lambda) + \lambda B_\ell^* \ge 0, \quad B_\ell^* \ge 0, \\ \left(\nabla_{B_\ell} \mathcal{R}(\Theta^*_\lambda) + \lambda B_\ell^*\right) \odot B_\ell^* = 0,
\end{aligned} 
\end{equation}
and analogously for $A_\ell^*$. 
Moreover, the iterates converge to a KKT point at the standard $\mathcal{O}(1/\epsilon)$ rate for projected gradient descent.
\end{theorem}

The proof adapts standard convergence analysis for projected gradient descent on smooth nonconvex functions \citep{Bertsekas01031997,Beck2017}. See \Cref{app:convergence} for details. In the next result, we analyze the structure of the KKT points of the objective. 

\begin{theorem}[Learning the Algorithm, informal version of \Cref{thm:training-suppresses-J-debug} and \Cref{cor:regularized-suppression}]
\label{thm:char_kkt_informal}
    Under suitable conditions where the gradient penalty from disconnected graphs outweighs the gradient reward from connected graphs, the only KKT-compliant value for the heuristic channel is $B_\ell^* = 0$, i.e., the model converges to learning the fully algorithmic matrix-powering algorithm. 
\end{theorem}


While \Cref{thm:convergence} guarantees that training converges to limit points $\Theta_\lambda^*$ of the objective $\mathcal{R}_\lambda(\Theta)$, \Cref{thm:char_kkt_informal} further characterizes a sufficient condition for algorithmic alignment. The formal versions for the un-regularized and regularized objectives can be found at \Cref{thm:training-suppresses-J-debug} and \Cref{cor:regularized-suppression}, respectively.  With convergence guaranteed, our analysis of gradient descents reveals that the training process consists of two distinct phases:



\textbf{Phase 1: Both channels pick up easy examples}. In early updates, both channels quickly ramp up mass because there are plenty of within-component, within-capacity pairs. Concretely, the local $I$-channel composes neighborhood information, while the global $J$-channel can also boost under-predicted positives without facing much penalty (\Cref{rem:early-training}). Phase $1$ is transient and ends once those easily connected pairs are mostly saturated. In \Cref{fig:1-layer-dynamics} (top), it only occupies around $2\cdot 10^2$ steps out of $10^4$ total.

\textbf{Phase 2: Data determines which channel wins}. Once in this regime, the growth of $B_\ell$ is determined by the population-level balance (see \Cref{thm:training-suppresses-J-debug} for full definitions and details). Informally speaking, the derivative of $B_\ell$ is driven by a competition between a suppression force and a promotion force. The suppression force arises from disconnected graphs, where the heuristic generates false positives, and they push $B_\ell$ to zero. The promotion force arises from connected graphs, where the heuristic helps minimize loss by correctly identifying connected pairs, generating a reward gradient.

There are two outcomes, best understood through the distinction between within-capacity and beyond-capacity graphs (\Cref{thm:dynamics}). If batches are dominated by within-capacity graphs ($\diam(G) \leq 3^L$), the algorithmic $I$-channel is consistently rewarded. Disconnected graphs in this regime penalize the heuristic $J$-channel for predicting false connections, driving $B_\ell \to 0$ (\Cref{thm:training-suppresses-J-debug}(ii)) and leaving only the algorithmic solution. Conversely, if the distribution contains a significant share of connected beyond-capacity graphs ($\diam(G) > 3^L$), the algorithmic channel cannot bridge the distance. These samples force the model to rely on the global $J$-channel to minimize loss, promoting the degree-counting shortcut. Thus, the final 
learned solution depends directly on the proportion of beyond-capacity connected graphs in the training distribution (\Cref{rmk:population-push}), as visualized in \Cref{fig:beyond_capacity_proportion}.

%% file: sections/05_Experiments.tex
\section{Experiments} \label{sec:exp}
We test the theory in two parts. First, in \S\ref{ssec:expt_capacity}  we verify the $3^L$ capacity threshold (\Cref{thm:expressivity,thm:capacity}) by measuring the maximum reliable path length one model can handle perfectly. Then, we trace training dynamics by projecting learned weights onto the algorithmic $A \otimes I$ channel and the heuristic $B \otimes J$ channel (\Cref{thm:w_decomposition}). Next, in \S\ref{ssec:learn_algo} we introduce a simple data lever that up-weights within-capacity graphs, and show that this simple method suppresses the heuristic and  promotes the algorithmic channel, as predicted by \Cref{thm:training-suppresses-J-debug,thm:dynamics}. Finally, we show this data lever prescribed by our theoretical analysis on Disentangled Transformers can transfer back to standard Transformers and boost their generalization capabilities.

\subsection{Capacity and Training Dynamics} \label{ssec:expt_capacity}
\paragraph{$L$-layer Transformers Hit Their Capacity at Exactly $3^L$.} We train Disentangled Transformers with 2 layers or 3 layers on Erd\H{o}s-R\'enyi graphs with 24 or 64 nodes respectively. As shown in \Cref{fig:capacity}, neither of the two models could make reliable predictions on node pairs $(u,v)$ with $d_G(u,v) > 3^L$ but their predictions on node pair with $d_G(u,v) \leq 3^L$ are almost perfect with an $> 99\%$ accuracy. As shown in \Cref{fig:capacity_roberta_2layer}, a 2-layer standard Transformer model also has the same empirical capacity. These results resonate with our exact capacity bound of Disentangled Transformers in \Cref{thm:capacity} and justifies our dichotomy in \Cref{def:capacity_dichotomy}. Overall, for any graph $G = (V, E)$, the decisive factor for Transformer model depth is not simply the asymptotic $\Theta(\log |V|)$ relation to the number of nodes $n = |V|$ but more importantly the non-asymptotically exact relation to $\log_3 \diam(G)$. 
We also note that we do not enforce non-negativity of model weights during training, but our theoretical analysis remains predictive of model capacity. 

\paragraph{Transformers Learn an Algorithm-Heuristic Mixture.} To empirically understand training dynamics, we first train a 1-layer Disentangled Transformer model on $\ER(n=8, p=0.2)$ graphs, without enforcing any parameterization assumptions. As shown in \Cref{fig:1-layer-dynamics}, a randomly initialized $W$ converges to a matrix approximately of the form $A \otimes I_n + B \otimes J_n$ for some matrices $A, B \in \mathbb R^{2 \times 2}$. Deeper models also converge to such solutions, as shown in \Cref{fig:dt_weights}. These results show the applicability of our decomposition \Cref{thm:w_decomposition}. Then, we project the final weight $W$ onto this algebra as $\hat{W} = \hat{A} \otimes I_n + \hat{B} \otimes J_n$ by minimizing $\|W - \hat{W}\|_F$. We observe that the share (see \S\ref{app:experiments}) of $\hat{A} \otimes I_n$ in $\|W\|_F^2$ increases as training progresses but the share of $\hat{B} \otimes J_n$ first increases and then decreases. These provide empirical evidence supporting our two-phase story in \S\ref{ssec:training_dynamics}.

\begin{figure}[t]
    \centering
    \includegraphics[width=0.95\linewidth]{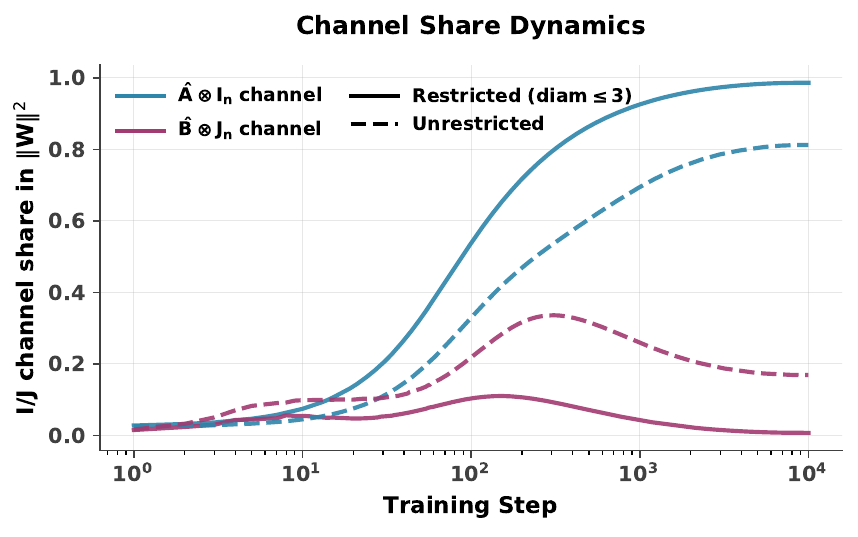}
    \caption{Following insights from \Cref{thm:training-suppresses-J-debug,thm:dynamics}, we repeat the same experiment setup as in \Cref{fig:1-layer-dynamics} but only training on within-capacity graphs (see \Cref{def:capacity_dichotomy}). As shown in the \textbf{solid} lines, restricting training samples by capacity pushes the energy share of the {\color{DodgerBlue}\textbf{algorithmic}} mechanism (the $A \otimes I_n$ channel) further to nearly 100\% in the weight $W$. It simultaneously prevents the growth of the {\color{purple}\textbf{heuristic}} portion (the $B\otimes J_n$ channel).}
    \label{fig:restrict_dt}
    \vsn
\end{figure}

\ifpreprint
\begin{figure}[t]
    \centering
    \includegraphics[width=\linewidth]{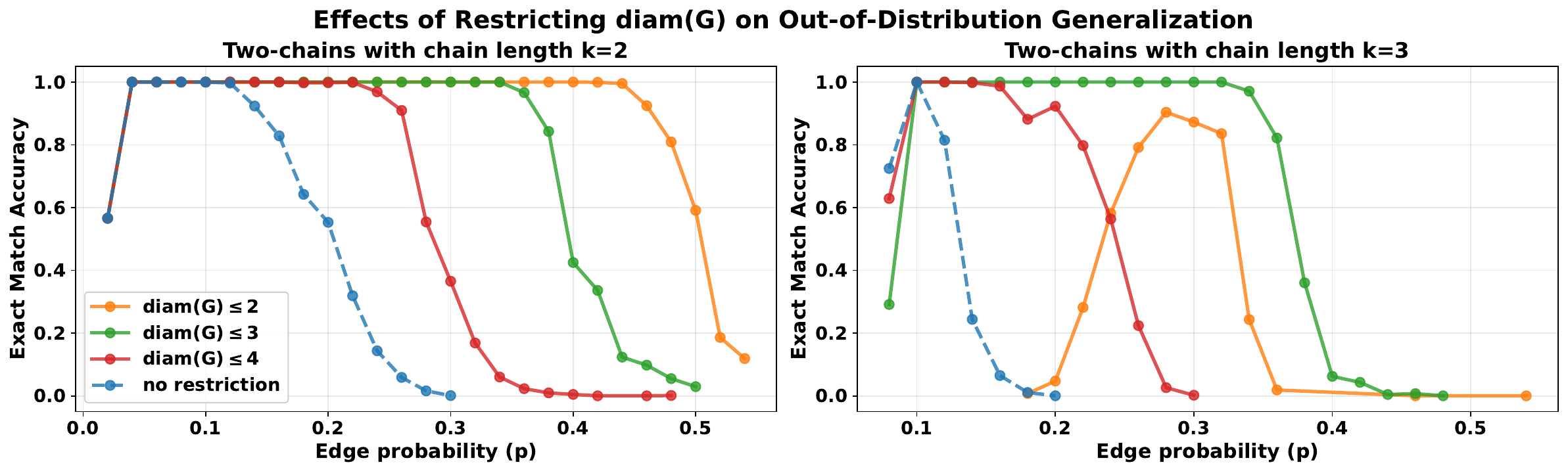}
    \caption{With 1-layer Disentangled Transformers with capacity $\mathsf{Cap} = 3$ following \Cref{thm:capacity}, we vary $d$ such that we restrict our training graphs to have $\diam(G) \leq d$. We also vary the edge probability of our training distribution $\ER(n=8, p = \cdot)$ for generality. We test on $\TChain(n=8, k=\cdot)$ graphs with $k=2$ or $3$ and show the exact match accuracy on configurations where the accuracy is non-zero for readability. We find if the training $d \leq \mathsf{Cap}$, models still learns the algorithmic solution up to problem size $d$ (see $d=2, k=2$ case on the left in {\color{orange}orange}) but \emph{fails to length generalize} (see $d=2, k=3$ in {\color{orange}orange} on the right). On the other hand, if the training $d > \mathsf{Cap}$, model struggles to learn the algorithmic solution (see $d=4$ cases in {\color{red}red} on both $k = 2$ or $3$). The best case overall is when setting $d = \mathsf{Cap} $, i.e., preventing the model from seeing beyond-capacity samples but still preserving at-capacity samples for better generalization. As shown in the {\color{ForestGreen}green} lines, with $d=3$, model achieves balanced testing accuracy on both $k=2$ and $3$.}
    \label{fig:varying_diam}
\end{figure}
\fi

\begin{figure}[t]
    \centering
    \includegraphics[width=0.95\linewidth]{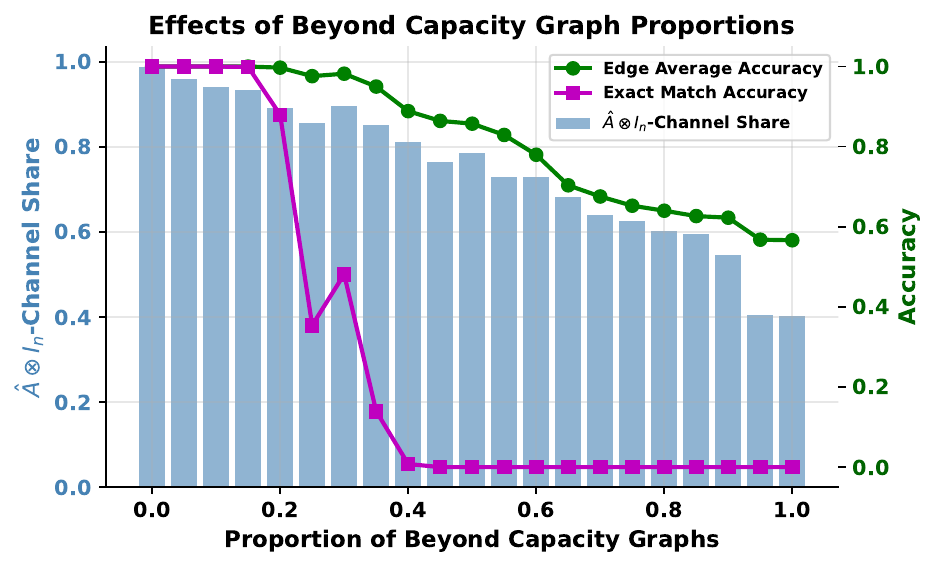}
    \caption{We vary the proportion of beyond-capacity graphs, and train the same Disentangled Transformer on stratified $\ER$ distribution and test on the same OOD $\TChain$ distribution. We find that Transformers are robust towards a small amount of noises (beyond-capacity graphs). Although the $W$ is not exactly in the $A \otimes I_n$ form, the model still performs perfectly when the energy share of $I$-channel dominates (beyond roughly 90\%). 
    }
    \label{fig:beyond_capacity_proportion}
    \vsn
\end{figure}

\subsection{Encouraging Transformers to Learn Algorithms}\label{ssec:learn_algo}
Now that we understand \textit{why} Transformers and Disentangled Transformer models learn heuristics that hurt their algorithmic computations (as shown in \Cref{fig:roberta,fig:1-layer-dynamics}), a natural question is whether we can mitigate this unwanted behavior and encourage the models to up-weight the algorithmic channel.

\paragraph{Mitigation via the Data Lever.}  We propose a data-centric method: instead of training on all graphs from the $\ER$ distribution, we up-weigh graphs whose $\diam(G)$ is within capacity following the dichotomy in \Cref{def:capacity_dichotomy}. We dissect the training distribution $\mathcal G$ into two sub-distributions: $\mathcal G_{\leq} = \{G \in \mathcal G : \diam(G) \leq 3^L \}$ only includes graphs containing no beyond capacity node pairs, and $\mathcal G_{>}$ includes the rest. In \Cref{fig:restrict_dt}, we only train the 1-layer Disentangled Transformer on $\ER_{\leq}$, and then find the algorithmic $\hat A \otimes I_n$ channel is significantly promoted so that the learned weight only contains the algorithm channel. Furthermore, we find at-capacity graphs are crucial. In the case of \Cref{fig:varying_diam,fig:at-capacity}, where no graphs are to have $\diam(G) > 2$, the model also fails to learn generalizable solutions due to Transformers' poor length generalization abilities. These results imply that simply scaling up the model depth does not naturally equip it with algorithmic capabilities . 

\paragraph{Robustness to Noise. } To test the predictiveness of our theory, we evaluate  if a small amount of beyond-capacity node pairs is enough to encourage the model learning a heuristic-dominated method. We define 
$\rho(\mathcal G) = \mathbb E_{G \in \mathcal G} {|\{(u,v) \in V, d_G(u,v) > 3^L\}|} /{n^2}$ to be the fraction of beyond-capacity node pairs in a graph distribution $\mathcal G$. In practice, $\rho$ can be controlled via stratified sampling from the mixture distribution $\mathcal G_q = q \mathcal G_{\leq} + (1-q) \mathcal G_{>}$.  In \Cref{fig:beyond_capacity_proportion}, we performed stratified sampling between $\ER_{\leq}$ and $\ER_{>}$ and found that with a small $\rho(\mathcal G)$, the model is still able to maintain high energy in the algorithm channel and make perfect predictions on out-of-distribution $\TChain$ graphs. It suggests that there exists a small  $\rho^\star > 0$ such that the model can still rely on the algorithm-channel to make predictions if the training distribution satisfies $\rho(\mathcal G) \leq \rho^\star$.

\begin{figure}[t]
  \centering
  \includegraphics[width=\linewidth]{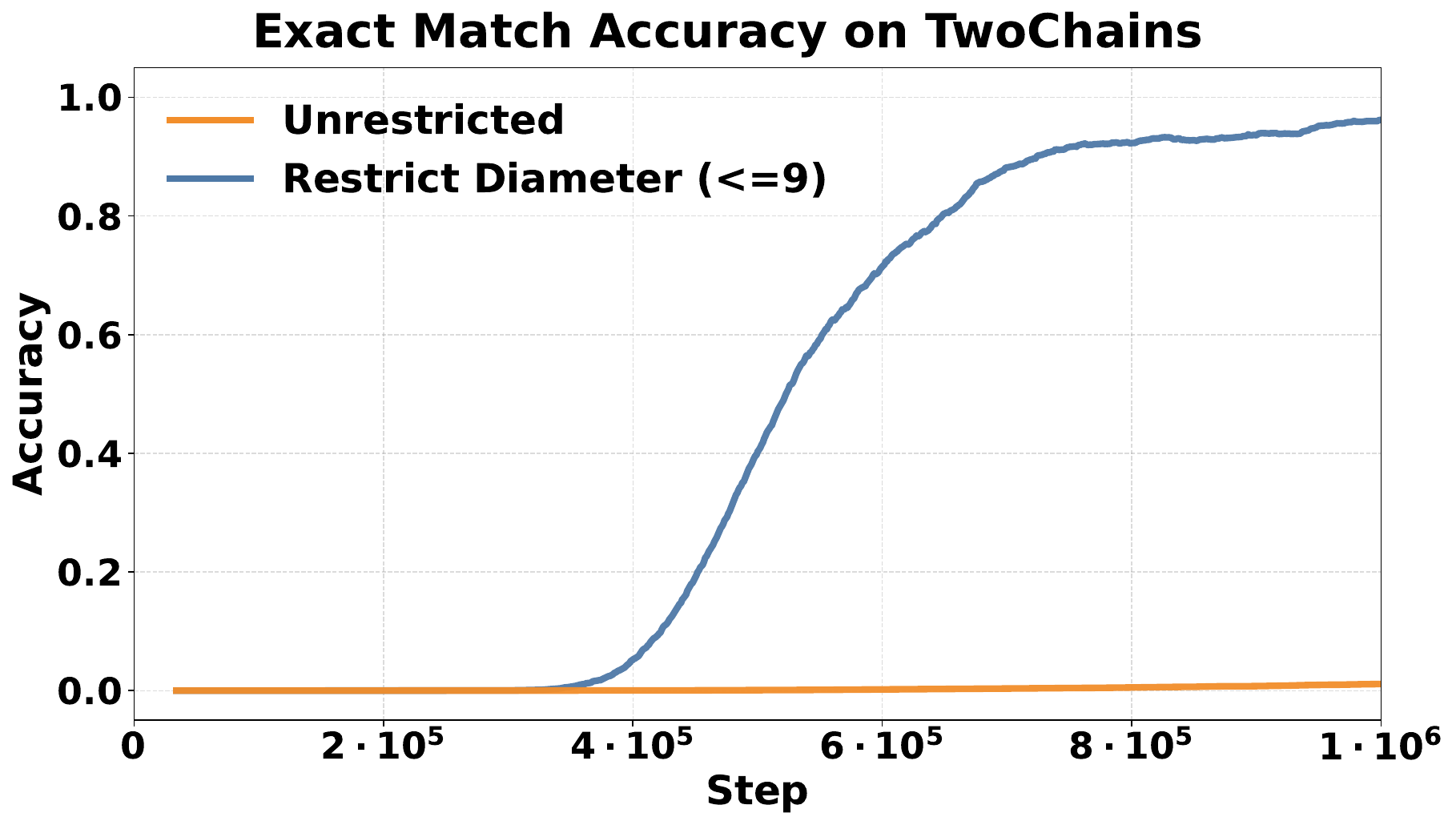}
  \caption{\textbf{Standard Transformer models learn generalizable solutions from within capacity data.} We train a 2-layer standard Transformer model on $\ER(n=20)$ graphs, with and without restricting graph diameters. When tested on OOD $\TChain$ graphs, the one trained with \textit{the right data} is able to generalize.}
  \label{fig:roberta_restrict_diam}
\end{figure}

\paragraph{Transferability to Standard Transformers models.} Our theory from \S\ref{ssec:training_dynamics} makes a prescriptive suggestion to remove beyond capacity graphs to reduce dependence on heuristics, and \Cref{fig:restrict_dt} demonstrated the effectiveness of this approach on   Disentangled Transformers. We now evaluate this on standard Transformers. 
We train the same 2-layer Transformers model as in our preliminary study in \S\ref{ssec:prelim} but this time, we train on the restricted distribution $\ER_\leq$ instead where all graphs $G \in \ER_\leq$ have $\diam(G) \leq 3^2=9$. As shown in \Cref{fig:roberta_restrict_diam}, when tested on the OOD $\TChain$ dataset with maximum chain length 10, the one trained on $\ER_\leq$ can successfully generalize but the one trained on unconstrained distribution $\ER$ cannot. It also helps model generalize to OOD $\TClique$ graphs as well (see \Cref{fig:roberta_restrict_diam_2clique}).

%% file: sections/06_Discussion.tex
\section{Discussion and Conclusion}
In this paper, we separate expressivity from capacity and training dynamics for Transformers on graph connectivity. We prove that an 
$L$-layer model can implement matrix powering and is perfect on graphs with $\diam(G) \leq 3^L$, and we show this $3^L$ threshold is tight. The failures in \S\ref{ssec:prelim} are explained by a capacity mismatch: training mass beyond $3^L$ steers learning toward a global shortcut rather than the intended multi-hop algorithmic computation. Our two-channel view makes this explicit and turns generalization into a property of the data distribution: when within-capacity pairs dominate, the algorithmic channel is selected. Experiments confirm the threshold and show that a simple capacity-aware data lever that up-weights within-capacity graphs suppresses the shortcut, promotes out-of-distribution generalization, and transfers to standard Transformers. By pinpointing when a model reaches for a shortcut and showing how simple data choices can steer it towards the true algorithmic solution, we outline a path to systematically control training data and model capacity to enable Transformers to learn solutions that generalize better.

%% file: sections/AA_Definition.tex
\section{Additional Details on Problem Setups and Preliminary Studies}
\begin{definition}[Transformer for Graph Connectivity: full specification]
\label{def:roberta_full}
\phantom{}

\textbf{Input and output}. Given a simple graph on $n$ nodes with adjacency matrix $A \in \{0,1\}^{n\times n}$, let $\bar{A} = A + I_n$ be its self-loop augmented adjacency matrix. We treat $\bar{A}$ as input embedding: row $i$ is the token for node $i$; column $j$ indexes a feature tied to node $j$. To ease notation, we will simply write $A$ in place of $\bar{A}$ and always assume the adjacency matrix is self-loop augmented. The model outputs an $n\times n$ score matrix $\mathrm{TF}_\Theta^L(A)$, the predicted connectivity matrix.

\textbf{Dimensions and parameters}. Fix depth $L$, hidden dimension $d > n$, number of heads $H$ with $d = Hd_h$, and feed-forward width $d_{\mathrm{ff}}$. The parameters we need include:
\[
		W_{\mathrm{in}}, W_{\mathrm{out} }\in \mathbb{R}^{n\times d}, \qquad W_{\ell,h}^{Q}, W_{\ell,h}^{K}, W_{\ell,h}^{V} \in \mathbb{R}^{d \times d_h}, \qquad W_\ell^{O} \in \mathbb{R}^{H d_h \times d}
\] 
\[
		W_\ell^{(1)}\in \mathbb{R}^{d \times d_{\mathrm{ff} }}, \; b_\ell^{(1)} \in \mathbb{R}^{d_{\mathrm{ff} }}, \qquad W_\ell^{(2)} \in \mathbb{R}^{d_{\mathrm{ff} }\times d}, \; b_\ell^{(2)} \in \mathbb{R}^{d}
\] 
for $\ell = 1,\hdots, L$ and heads $h = 1,\hdots, H$. We use pre-norm residual blocks with LayerNorm ($\mathrm{LN} $) and GeLU activations. We do not use attention masks or any extra positional encoding; the identity in $\overline{A}$ already pins each token to a node.

\textbf{The forward map}. The read-in is linear: $h^{(0)} = \overline{A} W_{\mathrm{in}}\in \mathbb{R}^{n\times d}$. From there, for each $\ell = 1,\hdots, L$, let $\tilde{h} = \mathrm{LN} (h^{(\ell-1)})$ as we use pre-norm. Within each block:
\begin{align*}
		\textbf{Multi-head self-attention} && Q_{\ell,h} = \tilde{h} W_{\ell,h}^{Q}, \quad K_{\ell,h} = \tilde{h} W_{\ell,h}^{K}, \quad V_{\ell,h} = \tilde{h} W_{\ell,h}^{V} \\
		\textbf{Attention scores} && \alpha_{\ell,h} = \frac{1}{n} \mathrm{ReLU}  (1 / \sqrt{d_h} \cdot Q_{\ell,h} K_{\ell,h}^{\top}), \quad z_{\ell,h} = \alpha_{\ell,h} V_{\ell,h} \\
		\textbf{Concatenation \textit{\&} residual} && z_\ell = [z_{\ell,1}\mid \hdots \mid z_{\ell,H}] W_\ell^{O} \in \mathbb{R}^{n\times d}, \qquad u_\ell = h^{(\ell-1)} + z_\ell \\
		\textbf{Feed-forward} && \hat{u}_\ell = \mathrm{LN} (u_\ell), \qquad \mathrm{FFN}_\ell(\hat{u}_\ell) = \mathrm{GeLU} (\hat{u}_\ell W_\ell^{(1)} + b_\ell^{(1)}) W_\ell^{(2)} + b_\ell^{(2)}  \\
\end{align*}
and finally $h^{(\ell)} = u_\ell+ \mathrm{FFN}_\ell(\hat{u}_\ell) \in \mathbb{R}^{n\times d}$. The read-out is linear: $\mathrm{TF} _\Theta ^L(A) = h^{(L)} W_{\mathrm{out} }^{\top} \in \mathbb{R}^{n\times n}$.
\end{definition}

\textbf{Metrics for Permutation Equivariance.} Let $P\in S_n$ be the corresponding permutation matrix for any $\sigma \in \Sn$. For a given graph adjacency matrix $A$, we compute the model's prediction in respect to $P_\sigma$ as $\mathcal M(P_\sigma A P_\sigma^\top)$. Now we define a equivariance consistency metric, \textbf{Equivariance Consistency via Frobenius Cosine Similarity}:
\begin{equation}
    \mathsf{ConsFrob}(\mathcal M) = \mathbb E_{\sigma \in \Sn, A \in \mathcal G} \left[\frac{\langle \mathcal M(P_\sigma A P_\sigma^\top), P_\sigma \mathcal M( A ) P_\sigma^\top \rangle_F}{\|\mathcal M(P_\sigma A P_\sigma^\top)\|_F \|P_\sigma \mathcal M( A ) P_\sigma^\top\|_F}\right]
\label{eqn:metric_equivariance}
\end{equation}

When measuring intermediate model computations, this metric is modified depending on the model type. For standard Transformer models, $\mathcal M^\ell$ computes the final $\mathsf{Readout}$ to the hidden states at layer $\ell$. For Disentangled Transformers we are computing $P_\sigma \mathcal M^\ell(A) (P_\sigma \otimes I_n)^\top$. 

%% file: sections/AB_Capacity.tex
\clearpage
\section{Details for Capacity } \label{app:appendix_b}

\subsection{Expressivity} \label{app:expressivity}
\begin{reptheorem}{\ref{thm:expressivity}}
    There exists an $L$-layer Disentangled Transformer that makes perfect predictions for every graph $G$ satisfying $\mathrm{diam}(G) \leq 3^L$.
\end{reptheorem}
\begin{proof}
Set $W_\ell = I_{d_{\ell-1}}$ for all layers and note that all matrices are entrywise nonnegative, so ReLU and the factor $1/n$ never changes supports. With $h_0 = [I\mid A] = [A^0\mid A^1]$ and update $h_\ell = [h_{\ell-1}\mid (h_{\ell-1} h_{\ell-1}^\top)h_{\ell-1}/n]$, we can show by induction that every $n\times n$ block of $h_\ell$ lies in $\mathrm{span}\{A^0, \hdots, A^{3^\ell}\}$, and that some block contains $A^{3^\ell}$ with a positive coefficient. Indeed, the base case holds trivially; for the inductive step, if a block within $h_{\ell-1}$ contains $A^m$, then $(h_{\ell-1} h_{\ell-1}^\top) h_{\ell-1}$ contains $A^{2m} A^m = A^{3m}$. Finally, the readout simply sums over all these blocks, so $\mathrm{supp}(\mathrm{TF}_\Theta^L(A)) = \mathrm{supp}(A^{3^L})$. 

Finally, because $A$ has self-loops, supports are monotone in power and stabilizes at $t\geq \mathrm{diam}(G)$. Thus, if $\mathrm{diam}(G) \leq 3^L$ we get $\mathrm{supp}(\mathrm{TF}_\Theta^L(A)) = \mathrm{supp}(A^{\mathrm{diam}(G)})$.
\end{proof}


\subsection{Capacity} \label{app:capacity}

\begin{reptheorem}{\ref{thm:capacity}}
		Fix $L\geq 1$ and let $\TF_{\Theta }^L$ be an $L$-layer Disentangled Transformer on $n = \Omega(3^L)$ nodes. Further assume that the weights $W_{\ell} \geq 0$ for each $\ell$. Then there exists a graph $G$ with diameter $>3^{L}$ on which $\TF_{\Theta }^{L}(A)$ is not perfect. In other words, diameter $3^{L}$ upper bounds the capacity of any $L$-layer Disentangled Transformer. In particular, taking $n\geq (7/3) \cdot 3^{L} + 2$ suffices.
\end{reptheorem}

For each layer $\ell$ we define the post-ReLU score $R_\ell = \mathrm{ReLU} (h_{\ell-1} W_\ell h_{\ell-1}^\top )$. The proof of the theorem will be partitioned into two branches: whether some intermediate $R_{\ell}$ gives a false positive on some graph, or all $R_\ell$'s are free of false positives on all graphs. We say a pair of nodes $(u,v)$ from $G$ is a \textit{witness} to false positives if they belong to different connected components while $R_\ell(G)_{u,v} > 0$. Throughout this section, we set $n\geq (7 /3) \cdot 3^{L} + 2$.

\vspace{5pt}

\begin{lemma}
		Assume the setup in \Cref{thm:capacity}. Suppose there exist some $n$-node graph, a layer index $\ell^* \in \{1, \hdots, L\}$, and vertices $u,v$ belonging to different connected components of the graph such that $(R_{\ell^*})_{u,v} > 0$. Further assume that $\ell^*$ is globally minimal, in the sense that for \textit{all} $n$-node graphs and all $\ell < \ell^*$, the corresponding $R_\ell$ has no false positive entries across components. Then there exists a graph $G$ such that $\mathrm{diam} (G) = 3^{L} + 1$, $\TF_{\Theta }^L(A(G))_{u,v} > 0$, where $u,v$ lie in different connected components of $G$. 
		\label{lem:teleportation}
\end{lemma}

\begin{proof}
	The proof roughly partitions into two parts. In the first half, we backtrack the computation DAG, tracing the ``sources'' that contribute to the false positiveness of $(u,v)$. This gives us subgraphs which we call \textit{certificates} that, if kept untouched, suffice to guarantee a false positiveness of $(u,v)$. In the second half, we construct a graph $G$ that preserves these certificates while also containing a path of length $>3^L$ disjoint from both certificates, and we show that $\TF_\Theta^L$ preserves false positiveness of $(u,v)$ on $G$, thereby proving the claim.
	
	\vspace{5pt}

	\textsc{Step 1. Constructing the certificates}. Intuitively, since $R_{\ell^*}(H)_{u,v} > 0$, there exist column indices $p,q$ with $(W_{\ell^*})_{p,q} > 0, h_{\ell^*-1}(H)_{u,p} > 0$, and $h_{\ell^*-1}(H)_{v,q} > 0$. We will backtrack the entries that contribute to the positiveness of the hidden states entries $h_{\ell^*-1}(H)_{u,p}$ and $h_{\ell^*-1}(H)_{v,q} > 0$ in the computation DAG, iteratively visiting previous layers. Formally, we define a \textit{certificate} for an entry $h_t(H)_{i,c} > 0$ to be a small tree whose nodes are triples [of form (layer, row, column)] recording earlier entries that must be positive to guarantee that the current one is positive. The root is $(t,i,c)$ and we build it top-down by repeating one of the two rules until we hit the first layer, which we know looks like $[I_n\mid A]$. We now describe how to backtrack.  Since $h_t = [h_{t-1}\mid \mathrm{Attn} (h_{t-1}; W_t)]$, we split the recursion on layer $t$ into two cases: whether the entry lies in the first half ($h_{t-1}$) or the second half ($\mathrm{Attn} (h_{t-1}; W_t)$).
	
	\begin{itemize}
		\item (First half) If column $c$ is in the inherited block of $h_t$, add a single child $(t-1, i,c)$ to $(t,i,c)$, as the value is simply copied from the previous layer $h_{t-1}$.
		\item (Second half) If column $c$ is in the newly appended block, then by definition
		\[
		h_t(H)_{i,c} = \frac{1}{n} \sum_{k}^{} R_t(H)_{i,k} h_{t-1}(H)_{k,c'}\qquad \text{ for some }c',
		\] 
		and since this is a sum of nonnegative terms, there exists at least one $k$ with $R_t(H)_{i,k} > 0$ and $h_{t-1}(H)_{k,c'} > 0$. In turn,
		\[
		R_t(H)_{i,k} = \sum_{r,s}^{} h_{t-1}(H)_{i,r} (W_t)_{r,s} h_{t-1}(H)_{k,s} > 0,
		\] 
		which implies that there exist indices $r,s$ with $(W_t)_{r,s} > 0, h_{t-1}(H)_{i,r} > 0$, and $h_{t-1}(H)_{k,s} > 0$. Thus, for such $(t,i,c)$, we create three children:
		\[
		(t-1, k, c'), \qquad (t-1, i,r), \qquad (t-1, k,s).
		\] 
	\end{itemize}
	
	Now let $s(t)$ denote the maximal number of vertices needed to realize a single certificate for some entry $h_t(\boldsymbol\cdot )>0$ by the recursive procedure above. At $t=0$ we may assume $s(0)\leq 2$. The recursion gives $s(t) \leq 3s (t-1)$, so $s(t) \leq 2\cdot 3^{t}$. Since $u,v$ lie in different connected components of $H$, and $\ell^*$ is minimal, every index $k$ selected by a certificate at any layer $t \leq \ell^*-1$ stays within the same component as its $i$, so the two certificates induce trees $T_u, T_v$ that occupy disjoint vertex sets $S_u, S_v$, with $\lvert S_u \cup S_v\rvert \leq 4 \cdot 3^{\ell^*-1} \leq 4 \cdot 3^{L-1}$ vertices.

	\textsc{Step 2. Building a new graph.} Initialize $G$ to the edgeless graph, keeping node isolated. We then embed $T_u, T_v$ onto $G$ by adding edges according to the trees. Finally, we connect $3^L + 2$ vertices outside $S_u \cup S_v$ arbitrarily into a long chain of path length $3^L + 1$. This is always possible because there are $ n - |S_u \cup S_v|$ vertices outside the union, which is $ \geq \big((7/3) \cdot 3^L + 2\big) - 4 \cdot 3^{L-1} = 3^L + 2$, and this is why require $n \geq (7/3) \cdot 3^L + 2$ in this Lemma.

	We claim $G$ is the graph we seek. On one hand, every sum used by the certificates is a sum of nonnegative terms, and we have preserved a strictly positive summand at each step that appears in the tree. Hence $R_{\ell^*}(G)_{u,v} > 0$ with $u,v$ also disconnected in $G$. On the other hand, under the choice of $n$ specified by \Cref{thm:capacity}, there exist at least $3^L+2$ vertices outside $S_u\cup S_v$, so connecting them into a long path guarantees $\mathrm{diam}(G) = 3^L + 1$. The claim then follows.
\end{proof}

\begin{lemma}
		Assume the setup in \Cref{thm:capacity}. Further assume that for every $n$-node graph $G$ and every layer $\ell\in \{1,\hdots, L\}$, the post-ReLU scores $R_\ell(G)$ has no positive entry between distinct connected components of $G$. Then, for every graph $G$ and every $u,v\in V(G)$, if $\TF _\Theta ^L(A(G))_{u,v} > 0$, we must have $\mathrm{dist} _G(u,v) \leq 3^L$. Consequently, if $G$ contains a connected component of diameter $3^L + 1$ then $\TF _{\Theta }^L$ is not perfect on $G$.
		\label{lem:no-teleportation}
\end{lemma}

\begin{proof}
	Under the no-false-positives assumption, the idea is to show that ``information'' spreads no faster than power base $3$ so $\TF_\Theta^L$ never predicts ``Yes'' on node pairs with distance beyond $3^L$.	Concretely, columns exchange information as attention scores are calculated. We first define the ``distances'' between columns by giving each column a label $\in \{1,\hdots, n\}$, and then show that by layer $\ell$, two columns can ``share'' information if and only if their labels, \textit{interpreted as graph nodes,} are within distance $3^\ell$. 
	
	\textsc{Step 1. Giving each column a label.} We first consider trivial graph $G_0$ with $n$ isolated nodes: immediately $h_0(G_0) = [I_n \mid I_n]$ and, by hypothesis, every $R_\ell(G_0)$ have no off-diagonal positives. Inductively this shows that every column of $h_\ell(G_0)$ has support in exactly one row. We define the label of this column to be the row index $\in \{1,\hdots, n\}$ where the unique support is. With labels defined, the remaining proof is based on establishing the following locality claim.
	
	\vspace{5pt} 
	
	\textsc{Claim}. Fix graph $G$, layer $\ell$, and $i,j\in \{1,\hdots, n\}$. If column $c$ of $h_\ell(G)$ has label $j$ and if $h_\ell(G)_{i,c} > 0$, then $\mathrm{dist} _G(i,j) \leq 3^{\ell}$. In other words, \textit{every column spreads at most $3^{\ell}$ hops away from its label by depth $\ell$.}

	\vspace{5pt} 
	
	\textsc{Step 2. Establishing the claim.} We prove this claim via induction. The base case $\ell=0$ directly follows from the fact that $h_0(G) = [I_n\mid A(G)]$. For the inductive step, we assume that the claim holds at depth $\ell-1$ with radius $3^{\ell-1}$. As in \Cref{lem:teleportation}, there are two column types in $h_\ell$: inherited or newly appended columns. The former case is easy; if $c$ is inherited from $h_{\ell-1}$, then $h_{\ell}(G)_{i,c} = h_{\ell-1}(G)_{i,c}$, so the bound follows from the inductive hypothesis. We now assume $c$ is newly appended.

	Suppose $(R_\ell(G) h_{\ell-1}(G))_{i,c} > 0$ for a column $c$ with label $j$. Then there exists a row $k$ with $R_\ell(G)_{i,k} > 0$ and $h_{\ell-1}(G)_{k,c} > 0$. By the IH, $\mathrm{dist} _G(k,j) \leq 3^{\ell-1}$. Then we expand $R_\ell(G)_{i,k} > 0$ to obtain column witnesses $p,q$, with $h_{\ell-1}(G)_{i,p} > 0$, $h_{\ell-1}(G)_{k,q} > 0$, and $(W_\ell)_{p,q} > 0$, as in \Cref{lem:teleportation}. Let $a,b$ be the labels of $p, q$, respectively. By IH again, $\mathrm{dist}_G(i,a) \leq 3^{\ell-1}$ and $\mathrm{dist} _G(k,b) \leq 3^{\ell-1}$. We now split the analysis into two cases. 
	
	\begin{itemize}
		\item If $a \neq b$, we derive a contradiction to the no-false-positives assumption by reusing the certificate procedure from \Cref{lem:teleportation}. Because $W_\ell\geq 0$ entrywise, every positive entry in $h_t(\boldsymbol\cdot )$ admits a certificate supported on at most $s(t) \leq 2\cdot 3^{t}$ vertices. In particular, there exist certificates witnessing $h_{\ell-1}(G)_{i,p} > 0$ (labeled $a$) and $h_{\ell-1}(G)_{k,q}> 0$ (labeled $b$). Let $S_a, S_b$ be the corresponding certificate vertex sets. Form a new graph $G'$ on the same $n$ vertices whose connected components are two disjoint induced copies $S_a', S_b'$ of the subgraphs on $S_a, S_b$ (leaving all other vertices outside $S_a \cup S_b$ isolated). Note this is feasible because $\lvert S_a\rvert + \lvert S_b\rvert \leq 4 \cdot 3^{L-1} \leq n$ assumed by \Cref{thm:capacity}. By construction, there exist $i'\in S_a'$ and $k'\in S_b'$ with $h_{\ell-1}(G')_{i', p}>0$ and $h_{\ell-1}(G')_{k', q} > 0$. Thus,
		\[
		(h_{\ell-1}(G')W_\ell h_{\ell-1}(G')^\top )_{i',k'} \geq h_{\ell-1}(G')_{i',p} (W_\ell)_{p,q} h_{\ell-1}(G')_{k', q} > 0,
		\]
		meaning $R_\ell(G')_{i', k'} > 0$. But $i', k'$ belong to different connected components in $G'$, contradiction! Therefore,
		\item $a=b$. Triangle inequality gives $\mathrm{dist} _G(i,j) \leq \mathrm{dist} _G(i,a) + \mathrm{dist} _G(a,k) + \mathrm{dist} _G(k,j) \leq 3 \cdot  3^{\ell-1}  = 3^{\ell}$, completing the induction.  \hfill \textsc{End proof of claim / Step 2.}
	\end{itemize}
	
	The model's output $\TF _{\Theta }^L(A(G)) = h_L(G) W_O^\top $ is an entrywise nonnegative sum over the $n\times n$ blocks of $h_L(G)$. Since each block respects the $3^{L}$ locality bound, we have $\TF _{\Theta }^L(A(G))_{u,v} = 0$ whenever $\mathrm{dist} _G(u,v) \geq  3^{L} + 1$. Hence, on any graph whose largest component has diameter $3^{L} + 1$, the model will inevitably miss a pair $(u,v)$ of nodes realizing this diameter.
\end{proof}

\begin{proof}[Proof of Theorem \ref{thm:capacity}]
Combine \Cref{lem:teleportation,lem:no-teleportation}.
\end{proof}


%% file: sections/AC_TrainingDynamics.tex
\clearpage

\section{Details for Training Dynamics} \label{app:appendix_c}

\subsection{Characterizing Block Weights \texorpdfstring{$W_\ell$}{Wl}}
\label{app:characterization}

As discussed in \Cref{ssec:training_dynamics}, due to the symmetric nature of the graph connectivity problem, it is natural to demand that a ``good'' model should map not only adjacency matrices $A$ to connectivity matrices $R$, but also $PAP^\top$ to $PRP^\top$ for any permutation $P$. We further generalize equivariance. Observe that given a permutation matrix $P$ and any hidden states $h\in \mathbb{R}^{n\times (kn)}$ consisting of $k$ consecutive $n\times n$, the mapping $h\mapsto P h (I_K\otimes P^\top)$ relabels both rows and columns within each $n\times n$ block in a way that is consistent with the effects of $P$. Hence, the notion of equivariance can be generalized to any (nonnegative) hidden states, beyond just the ones induced by adjacency matrices.

Similarly, we are now also able to define an $L$-layer Disentangled Transformer on arbitrary inputs of appropriate dimensions. For any nonnegative initial state $h_0 \in \mathbb{R}^{n \times 2n}$, recursively define $h_\ell = [h_{\ell-1}\mid \mathrm{Attn} (h_{\ell-1}; W_\ell)]$ for $\ell = 1,\hdots, L$. Let $\mathrm{Sum}(h)$ denote the sum of the consecutive left-aligned $n\times n$ blocks of $h$. Then the generalized output is $\TF_\Theta^L(h_0) = \mathrm{Sum}(h_L)$. We define two equivariance-related conditions. The first one is a direct generalization of $P \TF_\Theta^L(A) P^\top = \TF_\Theta^L(PAP^\top)$; the second one, as discussed in \Cref{ssec:training_dynamics}, makes theoretical analysis significantly more tractable while also being supported by empirical evidence.

\begin{definition}[Output Equivariance and Layerwise Attention Equivariance]
    \label{def:equivariance}
    Let $\TF_\Theta^L$ be an $L$-layer Disentangled Transformer with nonnegative weights. Let $K_\ell = 2^{\ell+1}$.
    \begin{itemize}
        \item [(i)] For $h_0\in \mathbb{R}^{n\times 2n}_{\geq 0}$ and for any $P$, define $h_0^P = P h_0 (I_{K_0}\otimes P^\top)$. We say $\TF_\Theta^L$ is \textbf{output-level value equivariant} iff $P\TF_\Theta^L(h_0)P^\top = \TF_\Theta^L(h_0^P)$ holds for all $P$ and all $h_0\in \mathbb{R}^{n\times 2n}_{\geq 0}$.
        \item [(ii)] We say $\TF_\Theta^L$ is \textbf{layer-wise attention equivariant} iff for each $\ell$ and any hidden states $h \in \mathbb{R}^ {n \times d_{\ell-1}}$ (i.e., any hidden states of dimension feasible for layer $\ell$),
		\[
				\mathrm{Attn} (P h(I_{K_{\ell-1}}\otimes P^\top); W_\ell) = P \; \mathrm{Attn} (h; W_\ell)\; (I_{K_{\ell-1}}\otimes P^\top),
		\] 
     \end{itemize}
\end{definition}






\begin{theorem}[Parameterization of ``Good'' Models]		
    \label{thm:good-model-parameterization}
	Let $n = \Omega(3^{L})$ as in \Cref{thm:capacity}. Fix an $L$-layer Disentangled Transformer $\TF _{\Theta }^{L}$ with nonnegative weights. Suppose that
	\begin{itemize}
		\item [(i)] $\TF _\Theta ^{L}$ is output-level value-equivariant, and
		\item [(ii)] $\TF _{\Theta }^{L}$ reaches its capacity bound of $3^{L}$, i.e., for every graph, we have $\mathrm{supp} (\TF _\Theta ^{L}(A)) = \mathrm{supp} (A^{3^{L}})$.
	\end{itemize}
    
    Then, either $\TF_{\Theta}^L$ or a functionally equivalent version of it satisfies the following: for each layer $\ell$, there exists a nonnegative matrix $\Lambda_\ell\in \mathbb{R}^{K_{\ell-1} \times K_{\ell-1}}$ such that $W_\ell + W_\ell^\top = \Lambda_\ell \otimes I_n$. In other words, $W_\ell$ can be decomposed into this form up to an antisymmetric part.
	
	\vspace{5pt}
	
	(Note that the theorem is a direct generalization of equivariance under all graph permutations; replacing $h_0$ by $[I_n\mid A]$ gives the desired result for a fixed graph with adjacency matrix $A$.)
\end{theorem}

\begin{proof}{}{}
		To prove the claim, it suffices to show that if we partition $W_\ell$ into $K_{\ell-1} \times K_{\ell-1}$ contiguous sub-blocks of size $n\times n$, then each block must be diagonal, with symmetry conditions meeting $W_\ell + W_\ell^\top = \Lambda_\ell \otimes I_n$.

        To do so, the proof is split into two parts: we prove that each $n\times n$ block must be diagonal using (ii) and \Cref{lem:no-teleportation}, and that the diagonal entries must realize the said forms by examining the forward maps under a curated, parameterized class of initial hidden states. 

		\vspace{5pt} 

		\textsc{Step 1: each block must be diagonal}. In this step, we argue that if a block admits a positive off-diagonal entry, then the certificate trick from \Cref{lem:teleportation} will create a false positive entry on some output, contradicting (ii).
        
        Formally, let $R_\ell = \ReLU (h_{\ell-1} W_\ell h^\top _{\ell-1})$. If for some graph and some $\ell$, there exists a false positive entry $(R_{\ell})_{i,k} > 0$ for some $i,k$ across different connected components, then the false positiveness would persist to the output, contradicting (ii). Hence $\TF _\Theta ^{L}$ must have no false positives. 

		Consider feeding the graph $G_0$ of $n$ isolated vertices into $\TF _\Theta ^{L}$, so that $h_0(G_0) = [I_n \mid I_n]$. The premises of \Cref{lem:no-teleportation} hold, so every column of every $h_\ell(G_0)$ is supported in exactly one row, which we called its label in $\{1,\hdots, n\}$. Hence, if we write $h_\ell(G_0) = [X_1^{(\ell)}\mid \hdots \mid X_{K_{\ell}}^{(\ell)}]$ of contiguous $n\times n$ blocks, then each such block $X_r ^{(\ell)}$ must be nonnegative and diagonal. Now expand
		\[
				R_\ell = \ReLU (h_{\ell-1}W_\ell h_{\ell-1}^\top ) = h_{\ell-1} W_\ell h_{\ell-1}^\top  =  \sum_{r,s}^{} X_r ^{(\ell-1)} W_\ell [r,s] (X_s ^{(\ell-1)})^\top .
		\] 
		We first claim that every $n\times n$ sub-block $W_\ell[r,s]$ is diagonal. Suppose not, that there exist indices $r,s$ and distinct nodes $i \neq k$ such that $(W_\ell[r,s])_{i,k} > 0$. For a node $i$ and a block $r$, we say $(i,r)$ is \textit{activatable} at depth $\ell-1$ if there exists \textit{some} graph $G$ such that $X_r ^{(\ell-1)}(G) [i,i] = h_{\ell-1} [i, (r-1)n + i] > 0$. Two cases:
		\begin{itemize}
			\item If at least one of $(i,r)$ or $(k,s)$ is not activatable, then for every graph $G$, at least one factor $X_r ^{(\ell-1)}(G)[i,i]$ or $X_s ^{(\ell-1)}(G)[k,k]$ is zero, and thus $(W_{\ell}[r,s])_{i,k}$ is functionally inert and never contributes to any $R_\ell$ entry. Hence we may simply set it to $0$ without altering the model's output on any graph. 
			\item If both $(i,r)$ and $(k,s)$ are activatable, take graphs $G_i, G_k$ that make $X_r ^{(\ell-1)} (G_i)[i,i] > 0$ and $X_x ^{(\ell-1)}(G_k)[k,k] > 0$. Using the certificate mechanism in \Cref{lem:teleportation}, each positiveness admits a finite certificate subgraph with at most $2 \cdot 3^{\ell-1}$ vertices. We then create a new graph $G'$ and disjointly embed both certificates into it, leaving all other vertex isolated. The two labels $i,k$, viewed as nodes, now lie in different components. But then the product
					\[
							X_r(G')[i,i] \cdot (W_\ell[r,s])_{i,k} \cdot X_k (G') [k,k] > 0,
					\] 
					making $(R_{\ell}(G'))_{i,k} > 0$, contradiction.
	\end{itemize}
	Therefore $W_\ell[r,s]$ is diagonal for all block indices $(r,s)$. This concludes \textsc{Step 1}. 

	\vspace{5pt}

	\textsc{Step 2. $W_\ell$ is node-symmetric.}  Given a triplet $(\ell, r, s)$, we can now write $W_\ell[r,s]$ as $\mathrm{diag} (w_{\ell,r,s}(1), \hdots, w_{\ell,r,s}(n))$. Our goal is to show that for each $(\ell, r, s)$, $w_{\ell,r,s}(j) + w_{\ell,s,r}(j) = w_{\ell,r,s}(k) + w_{\ell,s,r}(k)$ for all $j,k\in [n]$. We formalize this in matrix form: For each node $i\in [n]$ and each layer $\ell$, let $\Lambda _\ell ^{(i)} = [w_{\ell,r,s}(i)] _{r,s} \in \mathbb{R}^{K_{\ell-1} \times K_{\ell-1}}$ and define the symmetric part $\mathrm{Sym} (\Lambda_\ell^{(i)}) = ( \Lambda_\ell ^{(i)} + \Lambda_\ell ^{(i)T}) /2$; the goal is to show that given $\ell$, all $\Lambda_\ell^{(i)}$ are the same, so that $\mathrm{Sym} (W_\ell) = \Lambda_\ell \otimes I_n$ or equivalently, $W_\ell + W_\ell^\top = \Lambda_\ell \otimes I_n$, as claimed. 

	\vspace{5pt} 

	Throughout out this step, we will use a family of special hidden states parameterized by a scalar $\lambda > 0$ and a vector $u=(u_1, u_2)\in \mathbb{R}^2_{\geq 0}$. Fix distinct nodes $j \neq k$. For $\lambda, u$, define the initial state $h_0(\lambda, u) \in \mathbb{R}^{n \times 2n}$ by setting exactly four entries nonzero:
	\[
			\begin{cases}
					h_0(\lambda,u) [j,j] = \lambda u_1 \qquad h_0 (\lambda, u) [j,n+j] = \lambda u_2 \\
					h_0(\lambda,u) [k,k] = \lambda u_1 \qquad h_0 (\lambda, u) [k,n+k] = \lambda u_2.
			\end{cases}
	\] 
	Note that $h_0(\lambda,u)$ is invariant under the transposition $P = (j,k)$, i.e., $P h_0 (\lambda, u) (I_{K_0}\otimes P^\top) = h_0(\lambda, u)$. Therefore, by assumption (i), we must have $\TF_\Theta^L(h_0)_{j,j} = \TF_\Theta^L(h_0)_{k,k}$. Let $h_\ell(\lambda, u)$ be the network state at depth $\ell$. Because of \textsc{Step 1}, there is no cross-row interaction for this input at any depth. Writing the row-$i$ vector as $v_\ell^{(i)}(\lambda,u)\in \mathbb{R}^{K_\ell}$, recursion gives, for $i\in \{j,k\}$,
	\[
			v_0 ^{(i)}(\lambda, u) = \lambda u,\qquad  v_\ell^{(i)}(\lambda, u) = [v_{\ell-1}^{(i)}(\lambda, u)\mid q_\ell^{(i)}(\lambda, u) v_{\ell-1}^{(i)}(\lambda, u)]
	\] 
	where
	\[
			q_\ell^{(i)}(\lambda, u) = \frac{1}{n} \cdot  v_{\ell-1}^{(i)}(\lambda, u)^\top  \; \mathrm{Sym} (\Lambda_\ell ^{(i)}) \; v_{\ell-1}^{(i)}(\lambda, u).
	\]

	Taking $\ell_1$-norms gives
	\begin{equation}
			\|v_\ell^{(i)}(\lambda, u)\| = (1 + q_\ell^{(i)}(\lambda, u)) \|v_{\ell-1}^{(i)}(\lambda, u)\| \quad \text{and} \quad \|v_L^{(i)}(\lambda, u)\| = \|v_0^{(i)}(\lambda, u)\| \prod_{\ell=1}^{L} (1 + q_\ell^{(i)}(\lambda, u)).
            \label{eq:l1-norm-proof}
	\end{equation} 
	Because the readout weight $W_O$ is a concatenation of $I_n$'s, and under our specific input $h_0(\lambda, u)$, every nonzero row $i$ lies in columns with indices $i$ modulo $n$, the $(i,i)$ output numerically equals $\|v_L ^{(i)}(\lambda, u)\|$. Hence, assumption (i) requires $\|v_L^{(j)}(\lambda, u)\| = \|v_L^{(k)}(\lambda, u)\|$.

	\vspace{5pt}

	Let $\ell^*$ be the minimal layer such that $\mathrm{Sym} (\Lambda_{\ell^*}^{(j)}) \neq  \mathrm{Sym} (\Lambda_{\ell^*}^{(k)})$. If no such $\ell^*$ exists for all $j\neq k$, then all $\Lambda_\ell^{(i)}$'s are the same given any fixed $\ell$, and \textsc{Step 2} holds. Otherwise, for every $\ell < \ell^*$, the symmetric parts coincide, and $v_{\ell-1}^{(j)}(\lambda, u) = v_{\ell-1}^{(k)}(\lambda, u)$ and $q_\ell^{(j)}(\lambda, u) = q_\ell^{(k)}(\lambda, u)$ for all $\lambda, u$. We may use $v_{\ell^*-1}(\lambda, u)$ to denote both $v_{\ell^*-1}^{(j)}(\lambda, u)$ and $v_{\ell^*-1}^{(k)}(\lambda, u)$ for they are now equal. 

	Because of the structure of $h_0(\lambda, u)$, by induction, the row vectors of each hidden state admits an odd power expansion 
	\[
			v_{\ell-1}^{(i)}(\lambda, u) = \lambda u + \lambda^3 \xi _{1, \ell-1}(u) + \lambda^{5} \xi _{2, \ell-1}(u) + \hdots
	\] 
	from which we conclude $q_\ell^{(i)}(\lambda, u) = O(\lambda^2)$ for every $\ell$. In particular, at $\ell=\ell^*$,
	\[
			q_{\ell^*}^{(j)}(\lambda, u) - q_{\ell^*}^{(k)}(\lambda, u) = \frac{1}{n} \cdot  v_{\ell^*-1} (\lambda, u)^\top  (\mathrm{Sym} (\Lambda_{\ell^*}^{(j)}) -  \mathrm{Sym} (\Lambda_{\ell^*}^{(k)}))  v_{\ell^*-1} (\lambda, u) = \lambda^{2m}c(u) + o(\lambda^{2m})		
	\] 
	for some $m\geq 1$ and some nondegenerate polynomial $c(u)$ as $\lambda \searrow 0$. In particular, 
	\[
			q_{\ell^*}^{(j)}(\lambda, u) - q_{\ell^*}^{(k)}(\lambda, u) = \Theta (\lambda^{2m}).
	\] 

	We now put this back into the comparison between the output's $(j,j)$ and $(k,k)$ entry. Recall that $v_0^{(j)} (\lambda, u) = v_0^{(k)}(\lambda, u) = \lambda u$. Further, since $v_{\ell-1} = \lambda u + O(\lambda ^3)$, we know $q_\ell (\lambda, u) = O(\lambda^2)$ for every $\ell$ and every $i$. We drop $\lambda, u$ for notational simplicity. It follows from \eqref{eq:l1-norm-proof} that
	\begin{align*}
			\|v_L^{(j)}\| - \|v_L^{(k)}\| &=  \lambda \|u\| \cdot  \Bigg[ \prod_{l < \ell^*}^{} (1+q_\ell)  \Bigg] \cdot  \Bigg[ \Big[1 + q_{\ell^*}^{(j)}\Big] \prod_{\ell > \ell^*}^{} \Big[1 + q_\ell^{(j)}\Big] - \Big[1 + q_{\ell^*}^{(k)}\Big] \prod_{\ell>\ell^*}^{} \Big[1 + q_\ell^{(k)}\Big] \Bigg] \\
			&= \lambda \|u\| \cdot \Bigg[ \prod_{\ell<\ell^*}^{} (1 + O(\lambda^2)) \Bigg] \cdot \Big[q_{\ell^*}^{(j)} - q_{\ell^*}^{(k)}\Big] \cdot  \Bigg[ \prod_{\ell > \ell^*} (1 + O(\lambda^2))  \Bigg] \\
			&= \lambda \|u\| \Theta (\lambda ^{2m}) (1 + o(1)) = \Theta (\lambda ^{2m+1}) 
	\end{align*}
	which is nonzero for small $\lambda$. Hence the $(j,j)$ and $(k,k)$ entries can be made different, contradicting assumption (i), and the proof is complete! \qedhere

\end{proof}

\begin{reptheorem}{\ref{thm:w_decomposition}}
        \label{thm:w_decomposition_appendix}
	   Suppose an $L$-layer Disentangled Transformer $\TF _{\Theta }^{L}$ has nonnegative parameters. Suppose $\TF _{\Theta }^{L}$ is layerwise permutation equivariant, i.e., for each $\ell$, any hidden states $h \in \mathbb{R}^ {n \times d_{\ell-1}}$, and any permutation $P\in S_n$,
		\[
				\mathrm{Attn} (P h(I_{K_{\ell-1}}\otimes P^\top); W_\ell) = P \; \mathrm{Attn} (h; W_\ell)\; (I_{K_{\ell-1}}\otimes P^\top),
		\] 
		then each block $W_{\ell} = A_\ell \otimes I_n + B_\ell \otimes J_n$ for some $A_\ell, B_\ell \in \mathbb{R}^{K_{\ell-1}, K_{\ell-1}}$. In other words, each block-aligned $n\times n$ submatrix of $W_{\ell}$ necessarily lies in $\mathrm{span} \{I_n, J_n\}$.
\end{reptheorem}

\begin{remark}
		The equivariance condition presented in the theorem is strictly harder than what we need for graph-level, layerwise equivariance: 
		\[
		\mathrm{Attn} (h_{\ell-1}(PAP^\top ); W_\ell) = P \;\mathrm{Attn}(h_{\ell-1}(A); W_\ell) \;(I_{K_\ell-1} \otimes P^\top ).
		\]
		For graphs, it suffices to assume that the hidden states are induced by some $n$-node graph. 
\end{remark}

\begin{proof}
		\textsc{Step 1. Relating to weight conjugation.} Fix a layer $\ell$. Write $K = K_{\ell-1}$, $h = h_\ell$, $W=W_\ell$, and let $\sigma(P) = I_K\otimes P$. The first step is to relate the conjugation of hidden states, $T_P(h): h\mapsto P h (I_K\otimes P^\top)$, to a conjugation of layer weights, $W_\ell \mapsto \sigma(P) W_\ell \sigma(P)^\top$. 
        
        Concretely, since $W_\ell\geq 0$, $\ReLU $.  Hence
		\begin{align*}
				\mathrm{Attn} (T_P(h); W) &=  \frac{1}{n} \ReLU [(P h \sigma(P)) \; W \; (\sigma(P)^\top  h^\top  P^\top )] (P h \sigma(P)) \\
				&= \frac{1}{n} P [h\sigma(P)\; W \; (\sigma(P)^\top  h^\top )] (h \sigma(P)) 
		\end{align*}
		and
		\[
				T_P(\mathrm{Attn}(h; W) ) = \frac{1}{n} P (h W h^\top ) h \sigma(P).
		\] 
		Layer-wise attention equivariance requires the two quantities above to equal for all $h$, and left multiplication by $P^{-1}$ gives
		\[
				h \Delta h^\top  h \;\sigma(P) = 0 \qquad \text{ for all }h \geq 0 \qquad \text{ where }\qquad \Delta := \sigma(P) W \sigma(P)^\top  - W. \tag{*}
		\] 

		\textsc{Step 2. Proving $\Delta = 0$.} To do so, we consider special hidden states, with only two nonzero entries $h_{i,p} = 1$ and $h_{j,q} = t$. Equivalently, pick columns $p\neq q$ and rows/nodes $i\neq k$ and set $h_{i, \boldsymbol\cdot } = e_p^\top $, $h_{j, \boldsymbol\cdot } = t e_q ^\top $, and $h=0$ everywhere else, where $e_p$ is standard basis vector pivoted at $p$.

		Because $h$ only uses columns $p$ and $q$, the matrix $h \Delta h^\top $ can be embedded on rows/columns $\{i,j\}$ with values
		\[
				h \Delta h^\top  = 
				\begin{pmatrix} \Delta_{p,p} & t \Delta_{p,q} \\ t \Delta_{q,p} & t^2 \Delta_{q,q} \end{pmatrix}.
		\]
		Recall $\sigma(P)$ is a permutation on columns; let $\pi $ be the permutation induced by it. Since $h \sigma(P)$ has the same two nonzero rows with $(h \sigma(P))_{i,\boldsymbol\cdot } = e^\top _{\pi (p)}$ and $(h \sigma(P))_{j, \boldsymbol\cdot } = t e^\top _{\pi (q)}$, we get that $(h \Delta h^\top ) (h \sigma(P))$ only has rows $i$ and $j$ potentially nonzero: 
		\[
				\begin{cases}
						\text{row }i: \Delta_{p,p} e^\top _{\pi (p)} + t^2 \Delta_{p,q} e^\top _{\pi (q)} \\[0.5em]
						\text{row }j: t \Delta_{q,p} e^\top _{\pi (p)} + t^2 \Delta_{q,q} e^\top _{\pi (q)}.
				\end{cases}
		\] 
		But recall (*): $(h \Delta h^\top ) (h\sigma(P)) = 0$ for all $t > 0$. The two standard basis vectors $e_{\pi (p)}, e_{\pi (q)}$ are linearly independent, so the coefficients must be uniformly zero! Hence $\Delta_{p,p} = \Delta_{p,q} = \Delta_{q,p} = \Delta_{q,q} = 0$. Finally, because $p\neq q$ were arbitrary, this forces $\Delta = 0$ entrywise. and that $\sigma(P) W_{\ell} \sigma(P)^\top  = W_{\ell}$ for this $P$. And because $P$ is arbitrary, we conclude that $\sigma(P) W \sigma(P)^\top  = W$ for every permutation $P$. 

        \vspace{5pt}
        
		\textsc{Step 3. Relating to $n\times n$ blocks.} Consider any $n\times n$ block $W[u,v]$ of $W$ where $1\leq u,v\leq K_{\ell}$. Using $\sigma(P) = I_{K_\ell} \otimes P^\top $ and taking the $(u,v)$ block on both sides,
		\[
				(\sigma(P) W \sigma(P)^\top )[u,v] = \sum_{a,b}^{} (I_{K_\ell})_{u,a} P^\top  W[a,b] P (I_{K_\ell})_{b,v} = P^\top  W[u,v] P.
		\] 
		The LHS equals $W[u,v]$, so we conclude that 
		\[
				P^\top  W[u,v] P = W[u,v] \qquad \text{ for all }P\in S_n.
		\] 

		In other words, layerwise equivariance implies each block must be invariant under $P^\top  (\boldsymbol\cdot ) P$. Taking any transposition forces all diagonal entries of a block to equal, while for any $i\neq j, k\neq \ell$, any arbitrary permutation mapping $\pi (i)=k, \pi (j)=\ell$ forces entries $(i,j)$ and $(k,\ell)$ to be equal. This implies that each block lies in $\mathrm{span} \{I_n, J_n\}$ as claimed. 
\end{proof}

\subsection{Population Gradient Lives in the Equivariant Algebra}

\label{app:gradient}

\begin{theorem}[Population gradient lives in the equivariant algebra]
		\label{thm:population-gradient}
		Under \Cref{asm:training_dynamics}, in particular using layerwise parameterization $W_\ell = A_\ell \otimes I_n + B_\ell \otimes J_n$, fix a layer $\ell$ and let $K = K_{\ell-1}$. Then the population gradient with respect to $W_\ell$ lies in $M_{K}(\mathbb{R}) \otimes \mathrm{span} \{I_n, J_n\}$: there exist matrices $G_\ell^{(I)}, G_\ell^{(J)}\in \mathbb{R}^{K\times K}$ such that
		\begin{equation}
				\mathbb{E} \Big[ \frac{\partial \mathcal{L}}{\partial W_\ell}  \Big] = G_\ell ^{(I)} \otimes I_n + G_\ell^{(J)} \otimes J_n.
		\end{equation}
\end{theorem}

\begin{proof}{}{}
		We let $S_n$ act on node indices. Since $W_\ell$ can be parametrized as $W_\ell = A_\ell \otimes I_n + B_\ell \otimes J_n$, the attention map is equivariant under left-right action:
		\[
				\mathrm{Attn} (P h (I_K \otimes P^\top ); W_\ell) = P \; \mathrm{Attn} (h; W_\ell) (I_K \otimes P^\top),
		\] 
		and so is the full map $A\mapsto Z$. For any fixed permutation $P$, the data $\ER (n,p)$ is permutation-invariant, i.e., $A$ and $PAP^\top$ are identically distributed. Because the model map and the loss are equivariant under $A\mapsto PAP^\top$ with $R \mapsto PRP^\top$, the sample gradient covaries as
        \[
				\nabla _{W_\ell} \mathcal{L}(PAP^\top ) = (I_k\otimes P)  \nabla_{W_\ell} \mathcal{L}(A) (I_K\otimes P^\top).
		\] 
        Taking expectation over $A$ gives
        \[
				\mathbb{E}_A [\nabla _{W_\ell} \mathcal{L}(PAP^\top )] = (I_k\otimes P)  \mathbb{E}_A[\nabla_{W_\ell} \mathcal{L}(A)] (I_K\otimes P^\top)
		\] 
        for every $P$. Hence the population gradient lies in the commutant of $\{I_K\otimes P: P\in S_n\}$. It remains to identify this commutant. View $G_\ell$ as a $K\times K$ block matrix with $n\times n$ sub-blocks. The relation $(I_K\otimes P)^\top  G_\ell (I_K\otimes P) = G_\ell$ says each $n\times n$ block $B$ satisfies $P^\top BP$ for all permutations $P$, so the block must have one value on the diagonal and one on the off-diagonals. It is well known that the fixed-point algebra of conjugation on $n\times n$ matrices is $\mathrm{span} (I_n, J_n)$. Hence every block lies in this span, i.e., $G_\ell \in M_K(\mathbb{R}) \otimes \mathrm{span}  \{I_n, J_n\}$. 
\end{proof}

\vspace{5pt}

\subsection{Which Conditions Encourage \texorpdfstring{$W_\ell \approx A_\ell \otimes I_n$?}{AI}}
\label{app:dynamics-1}

To facilitate the following analyses, it will be beneficial to first (re)introduce some notations. 

Throughout the analysis of training dynamics, we inherit the notations used in \Cref{asm:training_dynamics}: we use $Z$ to denote the model output, $R$ the reachability matrix, $A$ the adjacency matrix, $\mathcal{L} = \mathcal{L}(Z; R)$ the loss, and $\mathcal{R}(\Theta)$ the population risk $\mathcal{R}(\Theta) := \mathbb{E}_{G\sim \mathrm{ER}(n,p)} [\mathcal{L} (\TF_\Theta^L(A_G); R_G)]$.

Fix a layer $\ell$ and a nonnegative direction $\Delta \geq 0$ in the $J$-channel. Write $D = \frac{\partial Z}{\partial B_\ell} [\Delta]$ (more details in \Cref{thm:training-suppresses-J-debug}). We say a node pair $(i,j)$ is \textbf{active} for $\Delta$ if $D_{i,j} > 0$. In particular, we say $\Delta$ is active on cross-component pairs if $D_{i,j} > 0$ for some $(i,j)$ belonging to different connected components (note $\Delta$ could also be active on within-component pairs).

Because we constrain $W_\ell \geq 0$, under the parameterization $W_\ell = A_\ell \otimes I_n + B_\ell \otimes J_n$, we must also have $B_\ell \geq 0$. Then, the appropriate notion of stationarity is KKT: in our setting, this reduces to
\[
    \nabla_{B_\ell} \mathcal{R}(\Theta) \geq 0, \qquad B_\ell \geq 0, \qquad \text{ and } \qquad \nabla_{B_\ell} \mathcal{R} (\Theta) \odot B_\ell = 0
\]
which we use in the Theorem below.

\vspace{5pt}

\begin{theorem}[Population Training Conditionally Suppresses the $J$-Channel]
		\label{thm:training-suppresses-J-debug}
		Assume \Cref{asm:training_dynamics}. Fix any layer $\ell$ and decompose $W_\ell = A_\ell \otimes I_n + B_\ell \otimes J_n$. Let $Z$ be the output, $R$ the reachability matrix (ground truth), $\mathcal{L} = \mathcal{L}(Z; R)$ the loss, and $\mathcal{R}(\Theta )$ the population risk. 
		\begin{enumerate}
				\item (\textbf{Directional derivative on nonnegative $J$-channel directions.}) Let $\Delta\in \mathbb{R}^{K_{\ell-1} \times K_{\ell-1}}$ be entrywise nonnegative and define the one-sided Fr\'echet derivative $D := \frac{\partial Z}{\partial B_\ell} [\Delta] \in \mathbb{R}^{n\times n}.$
				Then $D \geq 0$ entrywise, and the population directional derivative satisfies
                    \begin{equation}
                            D_{B_\ell} \mathcal{R}(\Theta ) [\Delta]  = \mathbb{E}\left< \Big[ \frac{\partial \mathcal{L}}{\partial Z} , D \Big]\right> _F = \alpha \cdot  \mathbb{E} \Bigg[ \underbrace{\sum_{R_{i,j}=0}^{} D_{i,j}}_{\substack{\text{cross component}\\ \text{penalty}}} - \underbrace{\sum_{R_{i,j}=1}^{} \frac{1-\phi_\epsilon(Z_{i,j})}{\phi_\epsilon(Z_{i,j})}D_{i,j}}_{\substack{\text{within-component} \\ \text{reward}}} \Bigg].
                            \label{eq:B_ell_dynamics}
                    \end{equation}
                    In particular, $D_{B_\ell}\mathcal{R}(\Theta )[\Delta] \geq 0$ iff the Population-Level Dominance Condition holds (i.e. \Cref{eq:B_ell_dynamics} is positive). Throughout this Appendix, we will use ``cross component penalty'' and ``within-component reward'' to denote these two competing terms. 
                \item (\textbf{Consequences for KKT stationary points.}) Assume $\Theta$ is KKT-stationary for $B_\ell \geq 0$:
                \begin{equation}
                    \nabla_{B_\ell} \mathcal{R}(\Theta) \geq 0, \qquad B_\ell \geq 0, \qquad \text{ and } \qquad \nabla_{B_\ell} \mathcal{R} (\Theta) \odot B_\ell = 0
                \end{equation}

                Let $\Delta = \lvert B_\ell\rvert$ (entrywise absolute value) and let $D = \frac{\partial Z}{\partial B_\ell} [\lvert B_\ell\rvert ]$. If, with positive probability under $\ER (n,p)$, $\Delta$ activates at least one cross-component pair, and if the Population-Level Dominance Condition holds, then $B_\ell = 0$. Equivalently, under activation at $\Delta = \lvert B_\ell \rvert$ and strict dominance by cross-component penalty, the only KKT stationary point in the $J_n$-channel is $B_\ell = 0$.
                
                
		\end{enumerate}
\end{theorem}

\begin{lemma}[Monotonicity in the $J$-channel]
		\label{lem:monotonicity-in-J}
		Fix $\ell$ and hold all parameters except $B_\ell$. Write $h_{\ell-1} = [X_1 \mid \hdots \mid X_{K_{\ell-1}}]$ and $u_p = X_p \textbf{1} \in \mathbb{R}^{n}_{\geq 0}$. Then
		\begin{equation}
				h_{\ell-1} W_\ell h_{\ell-1}^\top  = \sum_{p,q}^{} (A_\ell)_{p,q} X_p X_q^\top  + \sum_{p,q}^{} (B_\ell)_{p,q} u_p u_q^\top.
				\label{eq:heuristics}
		\end{equation}
		Consequently, for every nonnegative direction $\Delta \geq 0$ in the $J$-channel, the one-sided Fr\'echet derivative at $0^+$ exists and is entrywise nonnegative. Hence, along the ray $\{B_\ell + \delta \Delta \mid \delta \geq 0\}$, the output is entrywise nondecreasing:
		\[
				\frac{\partial Z}{\partial B_\ell} [\Delta] \in \mathbb{R}^{n\times n}_{\geq 0}, \qquad Z(B_\ell + \delta \Delta) - Z(B_\ell) \geq 0 \text{ for all }\delta \geq 0.
		\] 
        Moreover, if $G$ is disconnected, and either (i) $\Delta_{p,p} > 0$ for a block $p$ such that $u_p$ has support in at least two components, or (ii) there exist blocks $p,q$ with $\Delta_{p,q} > 0$ and $u_p, u_q$ supported in different components, then there exist cross component pairs $(i,j)$ with $(\frac{\partial Z}{\partial B_\ell}[\Delta])_{i,j} > 0$.
\end{lemma}

\begin{proof}{}{}
		Since $J_n x = (\textbf{1}^\top x) \textbf{1}$ for $x\in \mathbb{R}^{n}$, we have $X_p J_n X_q^\top  = (X_p \textbf{1}) (X_q \textbf{1})^\top  = u_p u_q^\top $, yielding the displayed decomposition. For $B_\ell \mapsto B_\ell + \delta \Delta$ with $\Delta \geq 0$, the layer scores
		\[
				R_\ell(B_\ell + \delta \Delta) - R_\ell(B_\ell) = \delta \sum_{p,q}^{} \Delta_{p,q} u_p u_q^\top  \geq 0,
		\] 
		 so the one-sided derivative exists and is entrywise nonnegative. Because all subsequent maps are entrywise monotone, this implies $Z(B_\ell + \delta \Delta) - Z(B_\ell) \geq 0$ as stated. 
        
        For the ``moreover'' part, in casse (i), $u_pu_p^\top$ places positive mass on index pairs spnaning the componentns where $u_p > 0$, and in case (ii), $u_p u_q^\top$ (or its transpose) places positive mass across two components supporting $u_p$ and $u_q$. Monotonicity propagates these positives to $D = \frac{\partial Z}{\partial B_\ell}[\Delta]$.
\end{proof}

\begin{proof}[Proof of \Cref{thm:training-suppresses-J-debug}]
		For the population risk $\mathcal{R}(\Theta ) = \mathbb{E}[\mathcal{L}(Z; R)]$, applying definitions gives the directional derivative along $\Delta$ gives
		\[
				D_{B_\ell} \mathcal{R}(\Theta ) [\Delta] = \left<\mathbb{E} \Big[  \frac{\partial \mathcal{L}}{\partial Z} \Big], D \right> _F = \alpha \cdot  \mathbb{E}\Big[ \sum_{i,j}^{} \left( 1 - \frac{R_{i,j}}{\phi_\epsilon(Z_{i,j})}  \right) D_{i,j} \Big].
		\] 
		Separating indices by $R_{i,j}\in \{0,1\}$ proves \eqref{eq:B_ell_dynamics}.

        For the second claim, evaluate \eqref{eq:B_ell_dynamics} at $\Delta = \lvert B_\ell \rvert$. Under the activation premise (\Cref{lem:monotonicity-in-J}) and strict dominance by cross-component penalty, we obtain $ D_{B_\ell} \mathcal{R} (\Theta) [\lvert B_\ell \rvert] = \left<\nabla _{B_\ell}\mathcal{R}(\Theta ), \Delta \right> _F > 0$. Since $\nabla_{B_\ell} \mathcal{R}(\Theta) \geq 0$ and $\lvert B_\ell \rvert \geq 0$, a strictly positive inner product violates the KKT complementary condition $\nabla_{B_\ell} \mathcal{R} (\Theta) \odot B_\ell = 0$ unless $B_\ell = 0$.
\end{proof}

\begin{remark}
    \label{rem:early-training}
    While \Cref{thm:training-suppresses-J-debug} mostly discusses the suppression of $B_\ell$, its (i) in fact reveals a quite interesting, opposite phenomenon: \textbf{early training promotes $B_\ell$}. Before the model learns to pick up easy connected pairs, the corresponding values $\phi_\epsilon(Z_{i,j})\ll 1$. Consequently, the fractions $(1-\phi_\epsilon(Z_{i,j})) / \phi_\epsilon (Z_{i,j})$ are large, making \eqref{eq:B_ell_dynamics} negative. Gradient descent then pushes $B_\ell$ up ``without feeling pressure.'' As training proceeds, these easy connected pairs saturate ($\phi_\epsilon(Z_{i,j})\to 1$), while simultaneously $\Delta$ begins to active cross pairs (the ``moreover'' part of \Cref{lem:monotonicity-in-J}), increasing the $R=0$ term in \eqref{eq:B_ell_dynamics} and potentially flipping the sign. This is when the $J$-channel starts to incur penalty. This explains the transient ``Phase 1'' in \S\ref{ssec:training_dynamics}.
\end{remark}

\begin{remark}
    \label{rem:c8}
		The $B_\ell\otimes J_n$-channel injects rank-one dense terms $u_p u_q^\top $ into the attention core. On disconnected graphs, these terms produce cross-component positives, which the reachability target $R$ labels as negatives. Because disconnected graphs appear with positive probability in the data, the population gradient penalizes every nonnegative direction in the $J$-channel active on cross-component pairs whenever the cross-component penalty dominates within-component reward. Under the same activation and cross-component penalty dominance assumptions, any KKT stationary point must have $B_\ell = 0$. In short: under these conditions, population drives the node-side factor towards locality, i.e., $W_\ell \approx A_\ell \otimes I_n$. 
\end{remark}

\subsection{Which Samples Push Which Channel? (Local \texorpdfstring{$I_n$}{In} vs. Global \texorpdfstring{$J_n$}{Jn})}
\label{app:dynamics-2}

Recall $W_\ell = A_\ell \otimes I_n + B_\ell \otimes J_n$ and \Cref{lem:monotonicity-in-J}. The $I$-channel controls local propagation within components; the $J$-channel couples to the global / mean direction and injects dense rank-one terms. In this section, we first shift to a micro-level perspective, focusing on the effects of individual samples (graphs), and then draw connection to how the training distribution determines the model's eventual behavior (algorithmic vs. heuristic, \S \ref{ssec:training_dynamics}).

We decompose the single-sample loss $\mathcal{L}_G(\Theta ) := \mathcal{L}(\TF_\Theta ^{L}(A); R)$ and examine directional derivatives at a fixed $\Theta $, with the link gradient $\partial \mathcal{L} / \partial Z = \alpha(1 - R/ \phi_\epsilon(Z))$. Throughout, we say a pair $(i,j)$ is \textbf{saturated} if its per-pair loss gradient vanishes; for within-component pairs ($R_{i,j} = 1$) this is equivalent to $\phi_\epsilon(Z_{i,j}) = R_{i,j}$. We say a direction $\Delta$ is \textbf{active} over $(i,j)$ if the correpsonding channel directive $D_{i,j}>0$, where $D$ denotes $\frac{\partial Z}{\partial A_\ell} [\Delta]$ or $\frac{\partial Z}{\partial B_\ell} [\Delta]$ as appropriate. 

Our first main result is the following Theorem, which intuitively claims two things:

\begin{itemize}
    \item \textit{(Within capacity) Small-diameter graphs ``reward'' the local $I$-channel and, if disconnected, penalizes the global $J$-channel if activated.}
    \item \textit{(Beyond capacity) Large-diameter connected graphs demand a global shortcut: the $J$-channel is promoted, while the $I$-channel remains confined to short-range corrections.}
\end{itemize}

\vspace{5pt}

\begin{theorem}[Per-sample pushes by diameter]
		\label{thm:dynamics}
		Fix a layer $\ell$ and nonnegative directions $\Delta_A, \Delta_B \geq 0$ for $A_\ell, B_\ell$, respectively. Assume $B_1 = \hdots = B_L = 0$. 
		\begin{itemize}
				\item [(i)] (Within capacity) If $\mathrm{diam} (G) \leq 3^{L}$, then $D_{A_\ell} \mathcal{L}_G(\Theta ) [\Delta_A] \leq 0$, with strict $<0$ whenever $\Delta_A$ is active on at least one unsaturated within-component pair. If, in addition, $G$ is disconnected, then $D_{B_\ell}\mathcal{L}_G(\Theta ) [\Delta_B] > 0$ if both of the following hold: $\Delta_B$ is active at at least one cross-component pair, and Population-Level Dominance Condition holds.
						
				\item [(ii)] (Beyond capacity) If $\mathrm{diam} (G) > 3^{L}$ and $G$ is connected, then we have $D_{A_\ell} \mathcal{L}_G(\Theta ) [\Delta_A] \leq 0$ where only within-capacity pairs can contribute, and $D_{B_\ell} \mathcal{L}_G(\Theta ) [\Delta_B] < 0$ for $\Delta_B$ that is active on at least one unsaturated pair.
		\end{itemize}
\end{theorem}

\vspace{5pt} 

To prove this Theorem, we split the argument into the following four lemmas, each isolating one ingredient of the dynamics. Firstly, \Cref{lem:local-channel-always-helps} shows that the local $I$-channel is monotone: any nonnegative $A_{\ell}$ cannot increase the loss and is strictly helpful on unsaturated within-component pairs. This lets us treat local corrections as ``harmless,'' while \Cref{lem:help-hurt} analyze the sign of the global $J$-channel (connected vs. disconnected), and \Cref{lem:A-B-function} determines which pairs are ever affected when $B = 0$. Together, they yield the two cases in \Cref{thm:dynamics}.

\vspace{5pt}

\begin{lemma}[Local channel always helps]	
		\label{lem:local-channel-always-helps}
		Assume $B_1 =\hdots = B_L = 0$. For any graph $G$, any layer $\ell$, and any direction $\Delta \geq 0$ in the $I$-channel,
		\begin{equation}
				D_{A_\ell} \mathcal{L}_G(\Theta ) [\Delta] = \left<\frac{\partial \mathcal{L}}{\partial Z} , \frac{\partial Z}{\partial A_\ell} [\Delta] \right> _F \leq 0,
		\end{equation}
		with strict inequality whenever there exists a within-component, unsaturated pair, on which $\Delta$ is active. 
\end{lemma}

\begin{proof}
		From the block decomposition from \eqref{eq:heuristics}, the $I$-channel contributes $\sum_{p,q}^{} \Delta_{p,q} X_p X_q^\top.$, which is block-diagonal with respect to the component partition. Hence $\frac{\partial Z}{\partial A_\ell} [\Delta]$ has support only on pairs $(i,j)$ in the same component. On those pairs, $R_{i,j}=1$, and thus 
		\[
				\left( \frac{\partial \mathcal{L}}{\partial Z}  \right) _{i,j} = \alpha \cdot  \left( 1 - \frac{1}{\phi_\epsilon(Z_{i,j})} \right)  = -\alpha \cdot  \frac{1-\phi_\epsilon(Z_{i,j})}{\phi_\epsilon(Z_{i,j})} \leq  0,
		\] 
		with strict negativity whenever $\phi_\epsilon(Z_{i,j}) < 1$. Entrywise, nonnegativity of the forward map (\Cref{lem:monotonicity-in-J}) gives $\frac{\partial Z}{\partial A_\ell} [\Delta] \geq 0$. Therefore the Frobenius inner product $\leq 0$, and $<0$ under the stated conditions.
\end{proof}

\vspace{5pt}

We now switch from the local $I$-channel to the global $J$-channel and will use that the forward sensitivity in the $J$-channel is entrywise nonnegative, so the sign of the directional derivative is controlled entirely by the per-pair loss gradient. 

\vspace{5pt}

\begin{lemma}[Global channel helps connected graphs and conditionally hurts disconnected graphs]
		\label{lem:help-hurt}
		Fix a layer $\ell$ and a nonnegative direction $\Delta \geq 0$ in the $J$-channel. 

		\begin{itemize}
				\item [(i)] If $G$ is connected, then $D_{B_\ell}\mathcal{L}_G(\Theta ) \leq 0$, with strict $<0$ whenever there exists an unsaturated pair $(i,j)$ (i.e., $\phi_\epsilon(Z_{i,j}) < 1$) on which $\Delta$ is active ($D_{i,j} > 0)$.
				\item [(ii)] If $G$ is disconnected, then 
						\begin{equation}
								D_{B_{\ell}} \mathcal{L}_G(\Theta ) [\Delta] = \alpha \cdot  \Bigg[ \sum_{R_{i,j}=0}^{} D_{i,j} - \sum_{R_{i,j}=1}^{} \frac{1-\phi_\epsilon(Z_{i,j})}{\phi_\epsilon(Z_{i,j})} D_{i,j} \Bigg],
						\end{equation}
						hence $D_{B_\ell}\mathcal{L}_G(\Theta )[\Delta] \geq 0$ whenever the Population-Level Dominance Condition holds. Strict $>0$ holds if the inequality is strict, \textit{and} $\Delta$ is active on at least one cross pair.
		\end{itemize}
\end{lemma}

\begin{proof}{}{}
		By the chain rule, 
		\begin{equation}
				D_{B_\ell} \mathcal{L}_G(\Theta )[\Delta] = \left<\frac{\partial \mathcal{L}}{\partial Z} , \frac{\partial Z}{\partial B_\ell} [\Delta] \right> _{F} = \left<\alpha \cdot \left( 1- \frac{R}{\phi_\epsilon(Z)} \right), D\right> _F.
				\label{eq:chain-rule}
		\end{equation}
		By \Cref{lem:monotonicity-in-J} , $D\geq 0$ entrywise; moreover, $D_{i,j} > 0$ exactly on pairs where $\Delta$ is active. 

		\begin{itemize}
				\item [(i)] If $G$ is connected, then the reachability matrix $R$ is all-ones. Hence $\frac{\partial \mathcal{L}}{\partial Z} = - \alpha (1-\phi_\epsilon(Z)) / \phi_\epsilon(Z) \leq 0$ entrywise, with strict negativity whenever $\phi_\epsilon(Z_{i,j}) < 1$. Pairing with $D\geq 0$ and $D_{i,j}>0$ on active pairs gives $D_{B_\ell} \mathcal{L}_G(\Theta ) [\Delta] \leq 0$ and strict $<0$ under the stated saturation / activation conditions.
				\item [(ii)] If $G$ is disconnected, split \eqref{eq:chain-rule} over $R_{i,j}=0$ and $R_{i,j}=1$ to obtain the displayed identity. Since $D\geq 0$, the stated dominance condition yields $\geq 0$. Strictness requires a cross pair with $D_{i,j}>0$, holds exactly when $\Delta$ is active on at least one cross-component pair. \qedhere
		\end{itemize}
\end{proof}

\vspace{5pt}

We now show that when the global $J$-channel is disabled, the model can only light up within-capacity pairs. Note this is somewhat a converse to \Cref{thm:good-model-parameterization}, where a ``good'' model that only lights up within-capacity pairs necessarily have each $W_\ell[r,s]$ diagonal. The following Lemma isolates the role of the $J$-channel as the only ``nontrivial'' shortcut.

Recall from \Cref{def:capacity_dichotomy}: for a depth $L$ and a graph $G$ with adjacency matrix $A$, we call a pair $(i,j)$ \textbf{within capacity} if $[A^{3^{L}}]_{i,j} > 0$ and \textbf{beyond capacity} otherwise. 

\begin{lemma}[$I$-channel reaches within-capacity pairs; $J$-channel is the only dense shortcut]
		\label{lem:A-B-function}
		At any $\Theta $ with $B_1 = \hdots = B_L = 0$, the output satisfies 
		\[
				Z_{i,j} > 0 \implies  [A^{3^{L}}]_{i,j} > 0.
		\] 
		Equivalently, beyond-capacity pairs receive no positive mass from the $I$-channel alone. In contrast, for any $\ell$ and any $\Delta \geq 0$ in the $J$-channel, $\frac{\partial Z}{\partial B_\ell} [\Delta] \geq 0$ and is strictly positive on active pairs by definition.
\end{lemma}

\begin{proof}
		Since $B_\ell = 0$ implies $h_{\ell-1} W_\ell h_{\ell-1}^\top = \sum_{p,q} (A_\ell)_{p,q} X_p X_q^\top$ from \Cref{lem:monotonicity-in-J}, it is easy to see that they are block-diagonal w.r.t. connected components. Hence \Cref{lem:no-teleportation} applies and support expands by at most a factor of $3$ per layer, and only within-capacity pairs receive mass. The density statement follows from \Cref{lem:monotonicity-in-J}: For any $\Delta \geq 0$ in the $J$-channel, we have $\frac{\partial Z}{\partial B_\ell}  = \sum_{p,q}^{} \Delta_{p,q} u_p u_q^\top u \geq 0$. The strict positiveness characterization follows directly from \Cref{lem:monotonicity-in-J}.
\end{proof}

\vspace{5pt}

With the previous lemmas established, we can now assemble the per-sample sign rules. Intuitively, the $I$-channel makes only local corrections, never hurting the loss and only touching within-capacity pairs when $B = 0$, while the $J$-channel is the sole dense shortcut, helpful on connected graphs but penalized by cross-component pairs when the graph is disconnected.

\vspace{5pt}

\begin{proof}[Proof of Theorem \ref{thm:dynamics}]
		Let $\ell, \Delta_{A}, \Delta_B$ be given as described. Set $D_A = \frac{\partial Z}{\partial A_\ell} [\Delta_A]$ and $D_B = \frac{\partial Z}{\partial B_\ell} [\Delta_B]$. Recall from chain rule
		\[
				D_{(\boldsymbol\cdot )} \mathcal{L}_G(\Theta ) [\boldsymbol\cdot ] = \left<\frac{\partial \mathcal{L}}{\partial Z} , \frac{\partial Z}{\partial (\boldsymbol\cdot )} [\boldsymbol\cdot ] \right> _F = \alpha\cdot  \left<1 - \frac{R}{\phi_\epsilon(Z)}, \frac{\partial Z}{\partial (\boldsymbol\cdot )} [\boldsymbol\cdot ] \right> _F.
		\] 
		\begin{itemize}
				\item [(i)] (Within capacity) By \Cref{lem:local-channel-always-helps}, for any $\Delta_A \geq 0$ the $I$-channel directional derivative is $\leq 0$, with strict inequality under the stated conditions. The result on disconnected graphs $G$ follows from \Cref{lem:help-hurt}.
				\item [(ii)] By \Cref{lem:A-B-function}, with $B = 0$, only within-capacity pairs can be affected by the $I$-channel, so \Cref{lem:local-channel-always-helps} gives $D_{A_\ell}\mathcal{L}_G(\Theta ) [\Delta_A] \leq 0$. Since $G$ is connected and $\mathrm{diam} (G) > 3^{L}$, there will be unsaturated pairs; then \Cref{lem:help-hurt}(i) yields $D_{B_\ell} \mathcal{L}_G(\Theta ) [\Delta_B] < 0$, as claimed. \qedhere 
		\end{itemize}
\end{proof}

\begin{remark}[Population-level consequence under $\mathrm{ER} (n,p)$]
	\label{rmk:population-push}
	Fix a layer $\ell$ and nonnegative directions $\Delta_A, \Delta_B \geq 0$. Partition the graphs into $\mathcal{G}_0 = \{G: \mathrm{diam} (G) \leq 3^{L}\}$ and $\mathcal{G}_1 = \{G: \mathrm{diam} (G) > 3^{L}\}$. Writing the population directional derivatives as mixtures,
	\begin{equation}
		D_{B_\ell} \mathcal{R}(\Theta ) [\Delta_B] = \mathbb{P}(\mathcal{G}_0) \mathbb{E}[D_{B_\ell} \mathcal{L}_G(\Theta ) [\Delta_B]\mid G\in \mathcal{G}_0] + \mathbb{P}(\mathcal{G}_1) \mathbb{E} [D_{B_\ell} \mathcal{L}_G(\Theta ) [\Delta_B] \mid G\in \mathcal{G}_1].
	\end{equation}
	We claim the following on the population gradient.
	
	\begin{itemize}
		\item [(i)] (Local) {From \Cref{lem:local-channel-always-helps}, once the global $J$-channel has been suppressed, the local $I$-channel is consistently promoted until saturation. }				
		\item [(ii)] (Global) {The population gradient along the global $J$-channel is an explicit mixture of two regimes: large, connected graphs beyond capacity that promote the $J$-channel, and small, disconnected graphs within capacity that suppress it whenever cross-component errors persist.} Formally: 
        \vspace{5pt}
		\begin{itemize}
			\item [(ii.a)] If $G$ is connected and $\mathrm{diam} (G) > 3^{L}$, then by \Cref{lem:A-B-function}, every beyond-capacity pair has $Z_{ij}=0$ while $R_{ij}=1$. For those pairs, we have $\partial \mathcal{L} / \partial Z = -\alpha (1-\phi_\epsilon(Z)) / \phi_\epsilon(Z) < 0$. By \Cref{lem:monotonicity-in-J}, the inner product $\left<\partial \mathcal{L} / \partial Z, \partial Z / \partial B_\ell [\Delta_B] \right> _F < 0$ too. Integrating over all beyond-capacity, connected graphs yields
			\begin{equation}
				\mathbb{E} [D_{B_\ell} \mathcal{L}_G (\Theta ) [\Delta_B] \mid G\in \mathcal{G}_1 \ \textit{\&}\  G \text{ connected}] < 0.
			\end{equation}
			\item[(ii.b)] If $G$ is disconnected and $\mathrm{diam}(G)\le 3^L$, then by \Cref{lem:help-hurt}, $D_{B_\ell}\,\mathcal{L}_G(\Theta)[\Delta_B]\ \ge\ 0$
			with strict $>\!0$ if cross-component errors persist (the $\sum_{R=0} D$ term strictly dominates the $\sum_{R=1} (1-\phi_\epsilon(Z)) / \phi_\epsilon(Z) \cdot D$ term), and if $\Delta_B$ is {active on cross pairs} (i.e.\ $D^{(B)}_{ij}>0$ for some $R_{ij}=0$). The latter holds by \Cref{lem:monotonicity-in-J} if $\Delta_B$ is active on at least one cross pair. Integrating thus yields
			\begin{equation}
				\mathbb{E} \big[D_{B_\ell}\,\mathcal{L}_G(\Theta)[\Delta_B]\ \big|\ G\in\mathcal{G}_0\ \textit{\&}\ G\ \text{disconnected}\big]\ \geq\ 0,
			\end{equation}
			and strictly $>\!0$ provided the two additional assumptions above. 
		\end{itemize}
		
	\end{itemize}
\end{remark}

\subsection{Convergence of Projected Gradient Descent to KKT Points}
\label{app:convergence}

In this subsection, we establish that projected gradient descent on the regularized population risk converges to points satisfying the Karush-Kuhn-Tucker (KKT) conditions. This justifies the stationarity assumption underlying \Cref{thm:training-suppresses-J-debug}. We begin by stating the main result.

\begin{reptheorem}{\ref{thm:convergence}}
\label{thm:convergence_appendix}
Let $\mathcal{R}(\Theta) := \mathbb{E}_{G \sim \ER(n,p)}[\mathcal{L}(\TF_{\Theta}^{L}(A_G); R_G)]$ denote the population risk. For $\lambda > 0$, define the regularized objective $\mathcal{R}_\lambda(\Theta) := \mathcal{R}(\Theta) + \frac{\lambda}{2}\|\Theta\|_F^2$. Let $\mathcal{C} := \{(A_\ell, B_\ell)_{\ell} : A_\ell \ge 0, B_\ell \ge 0, \forall \ell\}$ denote the constraint set, and consider the sequence $\{\Theta^{(k)}\}_{k \ge 0}$ generated by projected gradient descent on $\mathcal{R}_\lambda$:
\begin{equation}
    \Theta^{(k+1)} = \Pi_{\mathcal{C}}\left(\Theta^{(k)} - \eta \nabla \mathcal{R}_\lambda(\Theta^{(k)})\right),
\end{equation}
with step size $\eta > 0$ sufficiently small and initialization $\Theta^{(0)} \in \mathcal{C}$ of the form $W_\ell = A_\ell \otimes I + B_\ell \otimes J$. Then every limit point $\Theta^*_\lambda \in \mathcal{C}$ satisfies the KKT conditions:
\begin{equation}
\label{eq:kkt_main}
    \nabla_{B_\ell} \mathcal{R}(\Theta^*_\lambda) + \lambda B_\ell^* \ge 0, \quad B_\ell^* \ge 0, \quad \left(\nabla_{B_\ell} \mathcal{R}(\Theta^*_\lambda) + \lambda B_\ell^*\right) \odot B_\ell^* = 0,
\end{equation}
and analogously for $A_\ell^*$. Moreover, after $K$ iterations, there exists $k < K$ such that $\Theta^{(k)}$ is an $\epsilon$-approximate KKT point with $\epsilon = O(1/K)$.
\end{reptheorem}

The proof adapts the standard convergence analysis for gradient descent on smooth nonconvex functions \citep{Nesterov2004IntroductoryLO} to the projected setting. The argument proceeds in three stages: we first introduce the gradient mapping as the appropriate measure of stationarity for constrained problems, then establish a sufficient decrease property for each iteration, and finally combine these ingredients via a telescoping argument to obtain the convergence rate.

\subsubsection{Preliminaries}

We work in the decomposed parameter space $\Theta = (A_\ell, B_\ell)_\ell$, where $W_\ell = A_\ell \otimes I + B_\ell \otimes J$ as in \Cref{asm:training_dynamics}. By \Cref{thm:population-gradient}, if the initialization lies in this subalgebra, then the gradient $\nabla \mathcal{R}_\lambda(\Theta)$ also decomposes as a direct sum over the $(A_\ell, B_\ell)$ components, and projected gradient descent preserves this structure. The constraint set $\mathcal{C} = \{(A_\ell, B_\ell)_\ell : A_\ell \ge 0, B_\ell \ge 0\}$ is the non-negative orthant in this decomposition, and the projection $\Pi_{\mathcal{C}}$ acts component-wise as $\Pi_{\mathcal{C}}(\Theta) = \max\{0, \Theta\}$.

Recall that a point $\Theta^* \in \mathcal{C}$ satisfies the KKT conditions for minimizing $\mathcal{R}_\lambda$ over $\mathcal{C}$ if
\begin{equation}
    \nabla \mathcal{R}_\lambda(\Theta^*) \ge 0, \quad \Theta^* \ge 0, \quad \text{and} \quad \nabla \mathcal{R}_\lambda(\Theta^*) \odot \Theta^* = 0,
\end{equation}
where $\odot$ denotes the Hadamard product and inequalities hold component-wise. These conditions assert that interior components (where $\Theta^* > 0$) have vanishing gradient, while boundary components (where $\Theta^* = 0$) have non-negative gradient pointing outward from the feasible region.

In unconstrained optimization, the gradient norm $\|\nabla f(\Theta)\|$ measures proximity to stationarity. For constrained problems, the appropriate generalization is the gradient mapping.

\begin{definition}[Gradient Mapping]
\label{def:gradient_mapping}
For step size $\eta > 0$, the gradient mapping at $\Theta \in \mathcal{C}$ is
\begin{equation}
    G_\eta(\Theta) := \frac{1}{\eta}\left(\Theta - \Pi_{\mathcal{C}}\left(\Theta - \eta \nabla \mathcal{R}_\lambda(\Theta)\right)\right).
\end{equation}
\end{definition}

The gradient mapping quantifies the displacement induced by a projected gradient step: the update rule can be written as $\Theta^{(k+1)} = \Theta^{(k)} - \eta G_\eta(\Theta^{(k)})$. When $\mathcal{C}$ is unconstrained, the projection is the identity and $G_\eta(\Theta) = \nabla \mathcal{R}_\lambda(\Theta)$. The following lemma confirms that the gradient mapping vanishes precisely at KKT points.

\begin{lemma}
\label{lem:gradient_mapping_kkt}
For any $\eta > 0$, a point $\Theta^* \in \mathcal{C}$ satisfies $G_\eta(\Theta^*) = 0$ if and only if $\Theta^*$ is a KKT point.
\end{lemma}
\begin{proof}
The condition $G_\eta(\Theta^*) = 0$ is equivalent to $\Theta^* = \Pi_{\mathcal{C}}(\Theta^* - \eta \nabla \mathcal{R}_\lambda(\Theta^*))$. For the non-negative orthant, this becomes
\begin{equation}
    \Theta^* = \max\left\{0, \Theta^* - \eta\nabla \mathcal{R}_\lambda(\Theta^*)\right\}.
\end{equation}
On the interior $\{\Theta^* > 0\}$, equality requires $\nabla \mathcal{R}_\lambda(\Theta^*) = 0$. On the boundary $\{\Theta^* = 0\}$, the condition reduces to $0 = \max\{0, -\eta\nabla \mathcal{R}_\lambda(\Theta^*)\}$, which holds if and only if $\nabla \mathcal{R}_\lambda(\Theta^*) \ge 0$. These are precisely the KKT conditions.
\end{proof}

\subsubsection{Regularity of the Objective}

The convergence analysis requires two properties of the regularized objective: coercivity, which ensures that iterates remain bounded, and smoothness, which enables a descent inequality.

\begin{lemma}[Coercivity]
\label{lem:coercivity}
For any $\lambda > 0$, the sublevel sets of $\mathcal{R}_\lambda$ are bounded: if $\mathcal{R}_\lambda(\Theta) \le c$, then $\|\Theta\|_F \le \sqrt{2c/\lambda}$.
\end{lemma}

\begin{proof}
Since $\mathcal{R}(\Theta) \ge 0$, we have $\mathcal{R}_\lambda(\Theta) \ge \frac{\lambda}{2}\|\Theta\|_F^2$. The bound follows by rearrangement.
\end{proof}

Coercivity is essential: without regularization, the scaling symmetry of ReLU networks ($W_\ell \mapsto \alpha W_\ell$, $W_{\ell+1} \mapsto \alpha^{-1}W_{\ell+1}$) renders the sublevel sets unbounded, and iterates could escape to infinity.

\begin{lemma}[Lipschitz Gradient]
\label{lem:lipschitz}
Under \Cref{asm:training_dynamics}, for any bounded set $\mathcal{B} \subseteq \mathcal{C}$, there exists $L_{\mathcal{B}} > 0$ such that
\begin{equation}
    \|\nabla \mathcal{R}_\lambda(\Theta_1) - \nabla \mathcal{R}_\lambda(\Theta_2)\| \le L_{\mathcal{B}}\|\Theta_1 - \Theta_2\| \quad \text{for all } \Theta_1, \Theta_2 \in \mathcal{B}.
\end{equation}
In particular, $\nabla \mathcal{R}_\lambda$ is Lipschitz continuous on the sublevel set $\{\Theta \in \mathcal{C} : \mathcal{R}_\lambda(\Theta) \le \mathcal{R}_\lambda(\Theta^{(0)})\}$.
\end{lemma}

\begin{proof}
Within the constraint set $\mathcal{C}$, ReLU activations act as the identity. Since the adjacency matrix $A_G \ge 0$ and $W_\ell = A_\ell \otimes I + B_\ell \otimes J \ge 0$ for $(A_\ell, B_\ell) \in \mathcal{C}$, induction shows $h_{\ell-1} W_\ell h_{\ell-1}^\top \ge 0$ at every layer, so $\mathrm{ReLU}$ is the identity throughout the forward pass. The transformer output is therefore a polynomial in $\Theta$, the loss function is smooth (with $\phi_\epsilon \ge \epsilon > 0$ keeping the logarithm away from zero), and the regularization is quadratic. The composition is smooth on $\mathcal{C}$, and smooth functions have Lipschitz gradients on bounded sets. The second statement follows from \Cref{lem:coercivity}, which ensures the sublevel set is bounded.
\end{proof}

\subsubsection{Sufficient Decrease}

The core of the analysis is a progress bound showing that each projected gradient step decreases the objective by an amount proportional to the squared gradient mapping norm.

\begin{lemma}[Sufficient Decrease]
\label{lem:sufficient_decrease}
For step size $\eta \le 1/L$, the iterates satisfy
\begin{equation}
    \mathcal{R}_\lambda(\Theta^{(k+1)}) \le \mathcal{R}_\lambda(\Theta^{(k)}) - \frac{\eta}{2}\|G_\eta(\Theta^{(k)})\|^2.
\end{equation}
\end{lemma}

\begin{proof}
Let $\Theta = \Theta^{(k)}$, $\Theta^+ = \Theta^{(k+1)}$, $g = \nabla \mathcal{R}_\lambda(\Theta)$, and $G = G_\eta(\Theta)$. The descent lemma for $L$-smooth functions gives
\begin{equation}
    \mathcal{R}_\lambda(\Theta^+) \le \mathcal{R}_\lambda(\Theta) + \langle g, \Theta^+ - \Theta \rangle + \frac{L}{2}\|\Theta^+ - \Theta\|^2.
\end{equation}
Since $\Theta^+ - \Theta = -\eta G$, this becomes
\begin{equation}
\label{eq:descent_intermediate}
    \mathcal{R}_\lambda(\Theta^+) \le \mathcal{R}_\lambda(\Theta) - \eta\langle g, G \rangle + \frac{L\eta^2}{2}\|G\|^2.
\end{equation}

It remains to show that $\langle g, G \rangle \ge \|G\|^2$. The projection $\Theta^+ = \Pi_{\mathcal{C}}(\Theta - \eta g)$ satisfies the first-order optimality condition: for all $z \in \mathcal{C}$,
\begin{equation}
    \langle (\Theta - \eta g) - \Theta^+, z - \Theta^+ \rangle \le 0.
\end{equation}
Taking $z = \Theta \in \mathcal{C}$ and using $\Theta - \Theta^+ = \eta G$ yields $\langle \eta G - \eta g, \eta G \rangle \le 0$, which simplifies to $\langle g, G \rangle \ge \|G\|^2$. Substituting into \eqref{eq:descent_intermediate} and using $\eta \le 1/L$ completes the proof.
\end{proof}

\subsubsection{Proof of \Cref{thm:convergence_appendix}}

\begin{proof}
We establish each component of the theorem in turn.

\paragraph{Bounded iterates.}
The sufficient decrease property (\Cref{lem:sufficient_decrease}) implies that $\mathcal{R}_\lambda(\Theta^{(k)})$ is non-increasing, so all iterates lie in the initial sublevel set. By \Cref{lem:coercivity}, this set is bounded: $\|\Theta^{(k)}\|_F \le \sqrt{2\mathcal{R}_\lambda(\Theta^{(0)})/\lambda}$ for all $k$.

\paragraph{Non-asymptotic rate.}
Rearranging \Cref{lem:sufficient_decrease} and summing from $k = 0$ to $K-1$:
\begin{equation}
    \sum_{k=0}^{K-1} \|G_\eta(\Theta^{(k)})\|^2 \le \frac{2}{\eta}\sum_{k=0}^{K-1}\left[\mathcal{R}_\lambda(\Theta^{(k)}) - \mathcal{R}_\lambda(\Theta^{(k+1)})\right] = \frac{2}{\eta}\left[\mathcal{R}_\lambda(\Theta^{(0)}) - \mathcal{R}_\lambda(\Theta^{(K)})\right].
\end{equation}
The right-hand side telescopes. Since $\mathcal{R}_\lambda \ge 0$, we obtain
\begin{equation}
    \sum_{k=0}^{K-1} \|G_\eta(\Theta^{(k)})\|^2 \le \frac{2\mathcal{R}_\lambda(\Theta^{(0)})}{\eta}.
\end{equation}
The minimum of $K$ non-negative terms is at most their average:
\begin{equation}
    \min_{k=0,\ldots,K-1} \|G_\eta(\Theta^{(k)})\|^2 \le \frac{2\mathcal{R}_\lambda(\Theta^{(0)})}{\eta K}.
\end{equation}
Thus, within $K$ iterations, at least one iterate achieves $\|G_\eta(\Theta^{(k)})\|^2 \le \epsilon$ for $\epsilon = O(1/K)$.

\paragraph{Asymptotic convergence.}
Taking $K \to \infty$, the bound $\sum_{k=0}^\infty \|G_\eta(\Theta^{(k)})\|^2 < \infty$ implies $\|G_\eta(\Theta^{(k)})\| \to 0$.


\paragraph{Limit points are KKT.}
Boundedness of the iterates follows from \Cref{lem:coercivity}. That every limit point of projected gradient descent on a smooth function over a closed convex set satisfies the first-order stationarity condition is a standard result; see, e.g., \citet{Bertsekas01031997} (Proposition 2.3.2) or \citet{Beck2017} (Theorem 10.15). The required smoothness holds by \Cref{lem:lipschitz}. Restricting to the $B_\ell$ components and expanding $\nabla \mathcal{R}_\lambda = \nabla \mathcal{R} + \lambda\Theta$ yields \eqref{eq:kkt_main}.
\end{proof}

\begin{remark}[Sufficiency for Main Results]
\Cref{thm:convergence_appendix} establishes convergence to KKT points in the sense of limit points; it does not rule out the possibility that the sequence oscillates between multiple KKT points. Establishing convergence to a unique limit requires additional structure, such as the Kurdyka-\L{}ojasiewicz property \citep{Attouch2013ConvergenceOD}. However, for our purposes, convergence to limit points suffices: \Cref{thm:training-suppresses-J-debug} shows that \emph{any} KKT point satisfying our assumptions has $B_\ell^* = 0$, so the heuristic channel is suppressed regardless of which limit point is approached.
\end{remark}

\subsubsection{Connecting Regularized Convergence to Heuristic Suppression}
\label{app:connecting-regularized}

\Cref{thm:convergence} establishes that projected gradient descent on the regularized objective $\mathcal{R}_\lambda(\Theta) = \mathcal{R}(\Theta) + \frac{\lambda}{2}\|\Theta\|_F^2$ converges to stationary points satisfying the KKT conditions. To complete the picture, we must link this convergence guarantee back to the gradient properties derived in \Cref{thm:training-suppresses-J-debug}, which analyzed the unregularized population risk $\mathcal{R}(\Theta)$.

We now demonstrate that the regularization term $\frac{\lambda}{2}\|\Theta\|_F^2$ works in tandem with the data-driven suppression mechanism. Specifically, for any nonnegative direction $\Delta \geq 0$ in the $J_n$-channel, the directional derivative of the regularized objective satisfies:
\begin{equation}
    D_{B_\ell}\mathcal{R}_\lambda(\Theta)[\Delta] = \underbrace{D_{B_\ell}\mathcal{R}(\Theta)[\Delta]}_{\text{Population Risk Gradient}} + \underbrace{\lambda \langle B_\ell, \Delta \rangle_F}_{\text{Regularization Penalty}}.
\end{equation}
When the cross-component penalty dominates the within-component reward (as characterized in \Cref{thm:training-suppresses-J-debug}), the population risk gradient is already nonnegative ($D_{B_\ell}\mathcal{R}(\Theta)[\Delta] \geq 0$). Since $\lambda > 0$ and $B_\ell \geq 0$, the regularization term is also nonnegative, and strictly positive whenever $B_\ell \neq 0$. Consequently, the regularization only strengthens the inequality, ensuring that the heuristic channel is suppressed at any stationary point.

\begin{corollary}[Regularized KKT Points Suppress the Heuristic Channel]
\label{cor:regularized-suppression}
Assume the setting of \Cref{asm:training_dynamics}. Let $\lambda > 0$ and let $\Theta_\lambda^* = (A_\ell^*, B_\ell^*)_{\ell=1}^L$ be a KKT point of the regularized objective $\mathcal{R}_\lambda(\Theta)$ over the constraint set $\mathcal{C}$. Fix a layer $\ell \in \{1, \ldots, L\}$ and assume the Population-Level Dominance Condition (cf. \Cref{eq:B_ell_dynamics}) holds at $\Theta_\lambda^*$:
\begin{equation}
    \mathbb{E}\left[ \sum_{R_{ij}=0} D_{ij} \right] > \mathbb{E}\left[ \sum_{R_{ij}=1} \frac{1 - \phi_\epsilon(Z_{ij})}{\phi_\epsilon(Z_{ij})} D_{ij} \right],
\end{equation}
where $D = \frac{\partial Z}{\partial B_\ell}[B_\ell^*]$ is the Jacobian along the direction of the learned weights $B_\ell^*$. Then the heuristic channel is fully suppressed: $B_\ell^* = 0$.
\end{corollary}

\begin{proof}
We proceed by contradiction. Suppose $B_\ell^* \neq 0$.

Since $\Theta_\lambda^*$ is a KKT point of $\mathcal{R}_\lambda$ over $\mathcal{C}$, the first-order optimality conditions for the non-negative parameter $B_\ell^*$ require
\begin{equation}
    \left(\nabla_{B_\ell} \mathcal{R}(\Theta_\lambda^*) + \lambda B_\ell^*\right) \odot B_\ell^* = 0.
\end{equation}
Summing over all entries (taking the Frobenius inner product with $B_\ell^*$), we obtain:
\begin{equation}
    \langle \nabla_{B_\ell} \mathcal{R}(\Theta_\lambda^*), B_\ell^* \rangle_F + \lambda \|B_\ell^*\|_F^2 = 0. \label{eq:kkt-sum}
\end{equation}

We now analyze the population gradient term $\langle \nabla_{B_\ell} \mathcal{R}(\Theta_\lambda^*), B_\ell^* \rangle_F$. By \Cref{thm:training-suppresses-J-debug}, this term decomposes into the difference between the expected penalty on disconnected graphs and the expected reward on connected graphs. By the Population-Level Dominance Condition assumed in the Corollary statement, the expected penalty strictly exceeds the expected reward. Therefore, the unregularized gradient contribution is strictly positive:
\begin{equation}
    \langle \nabla_{B_\ell} \mathcal{R}(\Theta_\lambda^*), B_\ell^* \rangle_F > 0.
\end{equation}
Intuitively, this means the data distribution itself is pushing the weights $B_\ell^*$ toward zero.

Substituting this into \eqref{eq:kkt-sum}, we arrive at a contradiction:
\begin{equation}
    \underbrace{\langle \nabla_{B_\ell} \mathcal{R}(\Theta_\lambda^*), B_\ell^* \rangle_F}_{> 0} + \underbrace{\lambda \|B_\ell^*\|_F^2}_{> 0} = 0.
\end{equation}
Both terms on the left-hand side are strictly positive (the second term because $\lambda > 0$ and we assumed $B_\ell^* \neq 0$). Their sum cannot be zero. Thus, we must have $B_\ell^* = 0$.
\end{proof}

\begin{remark}[Role of Regularization]
\label{rem:regularization-role}
The regularization parameter $\lambda > 0$ plays a dual role. First, it ensures coercivity (\Cref{lem:coercivity}), which is necessary to prove the existence of limit points for the training dynamics in \Cref{thm:convergence}. Second, as shown in \Cref{cor:regularized-suppression}, it acts as a strict enforcer of suppression: even if the data-driven gradient were merely zero (a ``tie" between penalty and reward), the regularization force $\lambda B_\ell^*$ would still drive the weights to zero via the stationarity condition.
\end{remark}

\begin{corollary}[Convergence to Algorithmic Solutions]
\label{cor:convergence-algorithmic}
Under \Cref{asm:training_dynamics}, let $\{\Theta^{(k)}\}_{k \geq 0}$ be the sequence generated by projected gradient descent on $\mathcal{R}_\lambda$ with step size $\eta \leq 1/L$ and initialization $\Theta^{(0)} \in \mathcal{C}$. Suppose that at every limit point $\Theta^*$, the Population-Level Dominance Condition from \Cref{cor:regularized-suppression} holds for all layers $\ell$.

Then every limit point satisfies $B_\ell^* = 0$ for all $\ell$, and consequently $W_\ell^* = A_\ell^* \otimes I_n$. The model at any limit point implements the algorithmic $I_n$-channel exclusively and reaches its theoretical capacity of $3^L$.
\end{corollary}

\begin{proof}
\Cref{thm:convergence} guarantees that every limit point $\Theta_\lambda^*$ satisfies the KKT conditions for $\mathcal{R}_\lambda$ over $\mathcal{C}$. Applying \Cref{cor:regularized-suppression} to each layer yields $B_\ell^* = 0$. With $B_\ell^*=0$, the heuristic channel is eliminated. The capacity statement then follows from \Cref{lem:A-B-function}, which establishes that when $B_1 = \cdots = B_L = 0$, the model output satisfies $Z_{ij} > 0$ only if $[A^{3^L}]_{ij} > 0$, achieving the tight capacity bound of \Cref{thm:capacity}.
\end{proof}

%% file: sections/AD_Experiments.tex
\clearpage
\section{Additional Experiment Details and Results} \label{app:experiments}

\subsection{Experiment Details}
\paragraph{Standard Transformers. } When training 2-layer standard Transformers, we adopt the implementation from RoBERTa \citep{liu2019roberta} with single-head per-layer and using normalized ReLU activation function as defined in \Cref{def:roberta_full}. We use a hidden dimension of $d = 512$ to make sure the hidden size is not the blocker for expressivity. We trained on 1 Billion $\ER$ graphs with a batch size of 1000 and $10^6$ steps. Each graph is only seen by the model once to resembling the training regime of modern LLMs. We note that although 1 billion graphs sounds a lot but with $n=20$ nodes, this is far from enumerating all possible graphs: there can be $2^{\binom{n}{2}}$ graphs if we don't consider graph isomorphism. When $n = 20$, this is about more than $10^{57}$ graphs in total, and 1 billion ($10^9$) is only a very small number of training instances. We train with AdamW optimizer with a learning rate of \texttt{1e-4} and weight decay of \texttt{1e-4} and a cosine learning rate decay.

\textbf{Disentangled Transformers. } For 1-layer Disentangled Transformers in \Cref{sec:exp}, we train on a fixed set 4096 i.i.d. samples of $\ER(n=8)$ graphs and running standard \textit{Gradient Descent} without any mini-batching. In this case, we have a learning rate of $0.1$ with cosine learning rate decay. For 2-layer Disentangled Transformers, we train on the same set of 1 billion number of $\ER(n=20)$ graphs as with standard Transformers. For 3-layer models, we train on 1 billion number of $\ER(n=64)$ graphs. Both 2- and 3-layer models are trained with AdamW with a learning rate of \texttt{1e-3}. We would like to note that the hidden dimensions $d_\ell$ of Disentangled Transformerss are fixed to be $d_\ell = 2^\ell n$ rather than a hyper-parameter (see \Cref{def:disentangle}).

\paragraph{Computing Energy Share of $I_n$/$J_n$ Channels.~~} In the experiments on 1-Layer Disentangled Transformers, we compute energy shares of the $A \otimes I_n$ and $B \otimes J_n$ within $\|W\|_F^2$. Here is the formalized versions. We consider the noisy decomposition $W=\hat A\otimes I_n+\hat B\otimes J_n+W_{\epsilon}$, where $W_\epsilon$ is the projection error term.  We define Frobenius-norm energy share on the $I_n$ channel as
\[
\mathsf{EnergyShare}(\hat{A} \otimes I_n, W) = \frac{\langle W,\,\hat A\otimes I_n\rangle}{\|W\|_F^2} = \frac{\|\hat A\otimes I_n\|_F^2 + \langle
\hat A\otimes I_n, \hat B\otimes J_n\rangle + \langle\hat A\otimes I_n, W_\epsilon\rangle}{\|W\|_F^2},
\]
and by symmetry, the \(J_n\)-channel share is
\[
\mathsf{EnergyShare}(\hat B\otimes J_n, W)
= \frac{\langle W,\,\hat B\otimes J_n\rangle}{\|W\|_F^2}
= \frac{\|\hat B\otimes J_n\|_F^2
+ \langle \hat B\otimes J_n,\hat A\otimes I_n\rangle
+ \langle \hat B\otimes J_n, W_\epsilon\rangle}{\|W\|_F^2}.
\]
This is a well-designed quantity because if you expand \(\|W\|_F^2\) you obtain $\langle W, \hat{A} \otimes I_n + \hat{B} \otimes J_n + W_\epsilon\rangle$, and the $I$/$J$-channels' energy shares will sum to one when the projection error $W_\epsilon$ converges to zero.

\subsection{Additional Experiments on Disentangled and Standard Transformers}
In \Cref{fig:dt_n64_model_behavior}, we show the training dynamics of a 3-Layer Disentangled Transformer. In \Cref{fig:dt_weights}, we show the learned weights by Disentangled Transformers. 
\begin{figure}[h]
    \centering
    \includegraphics[width=0.32\linewidth]{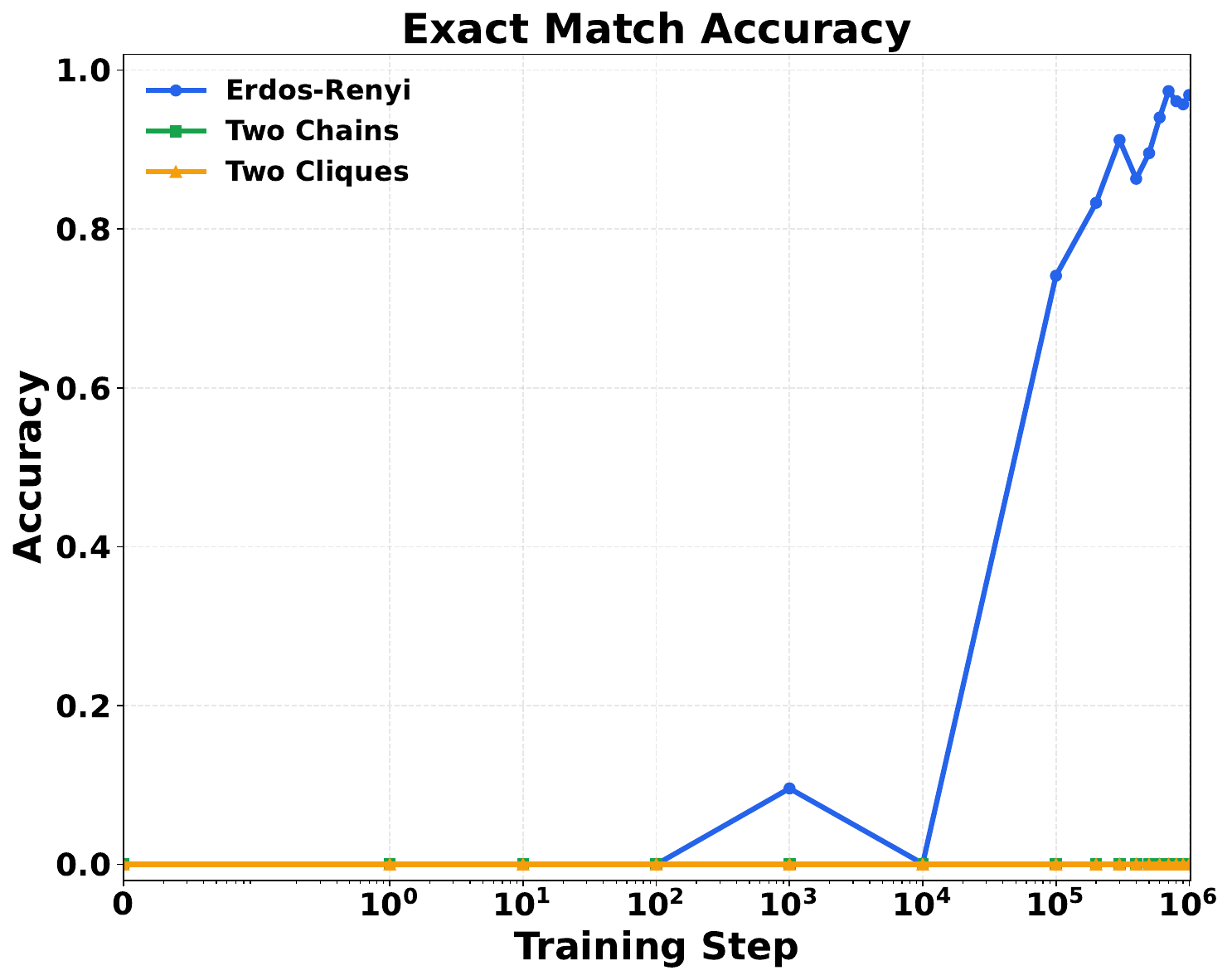}
    \includegraphics[width=0.32\linewidth]{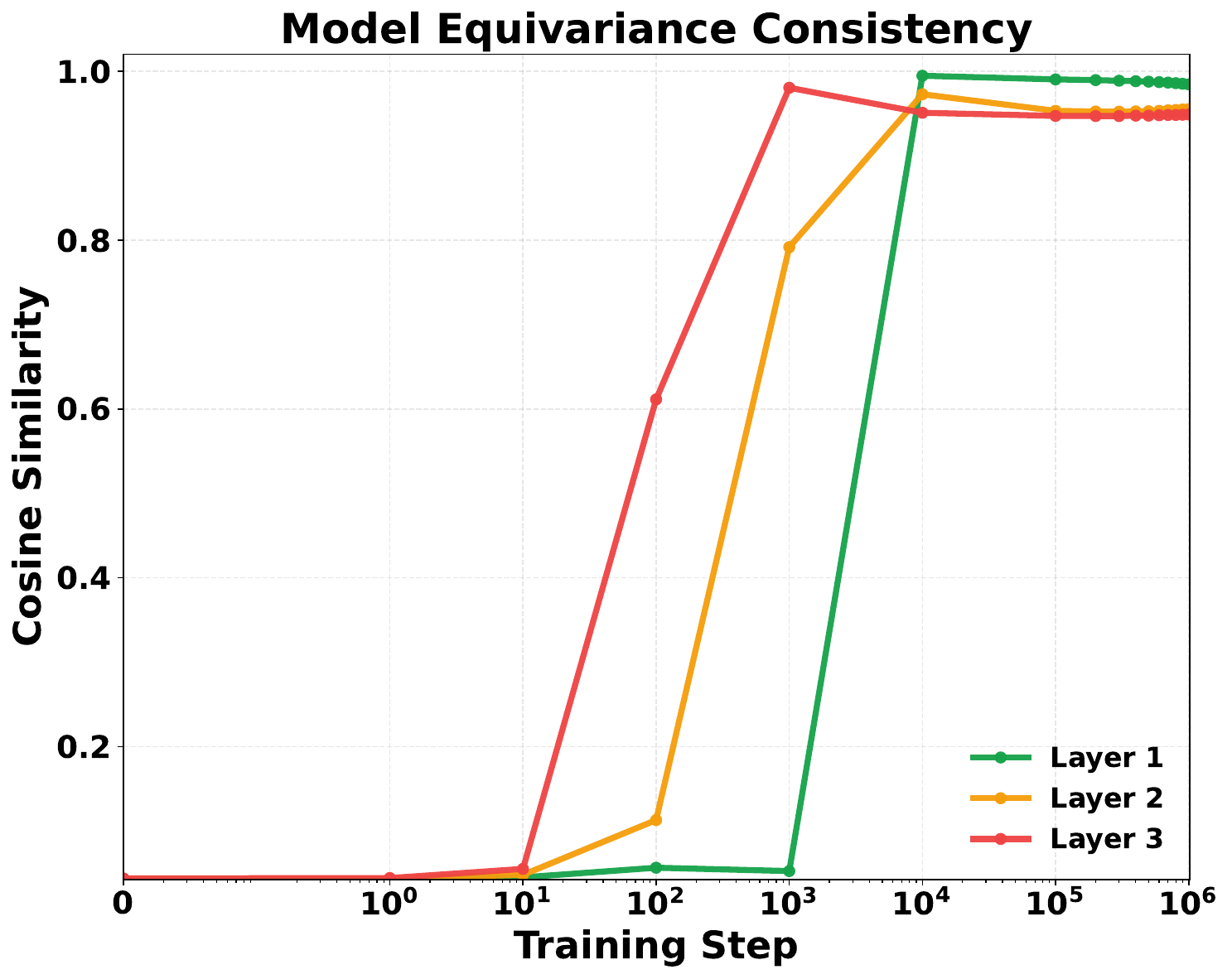}
    \includegraphics[width=0.32\linewidth]{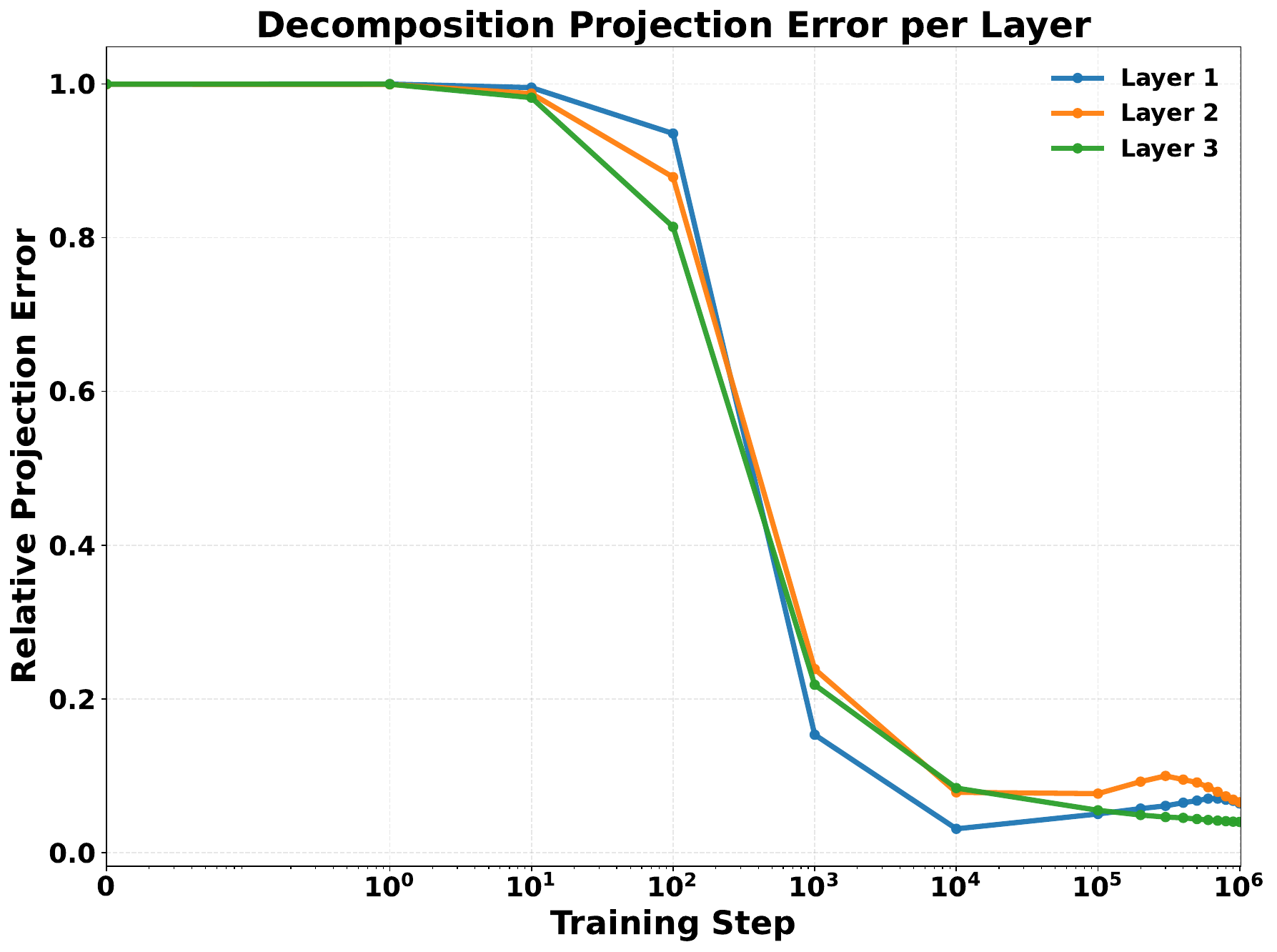}
    \caption{We plot the model behavior of a 3-Layer Disentangled Transformer model trained on $\ER(n=64)$ graphs. They also quickly pick up almost \textit{layer-wise equivariant} properties (measured by Eqn. \ref{eqn:metric_equivariance}). All layers show very small projection error onto the $A \otimes I_n + B \otimes J_n$ decomposition, resonating our theoretical claims in \Cref{thm:w_decomposition}.}
    \label{fig:dt_n64_model_behavior}
\end{figure}

In \Cref{fig:dt_weights}, we show that the trained 2-layer and 3-layer converge to weight spaces $W_\ell = A_\ell \otimes I_n + B_\ell \otimes J_n$ in the particular form echoing \Cref{thm:w_decomposition}. 

\begin{figure}[!h]
    \centering
    \includegraphics[width=0.4\linewidth]{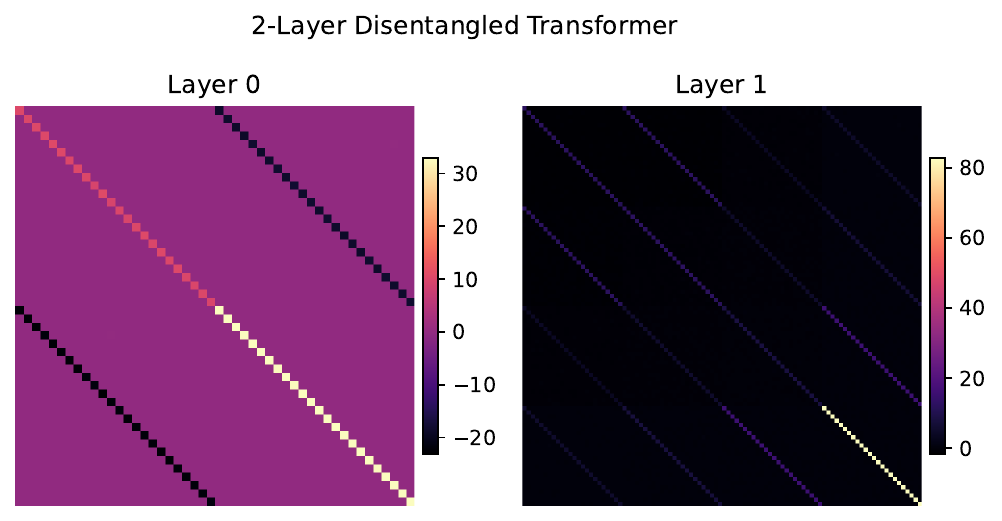}
    \includegraphics[width=0.64\linewidth]{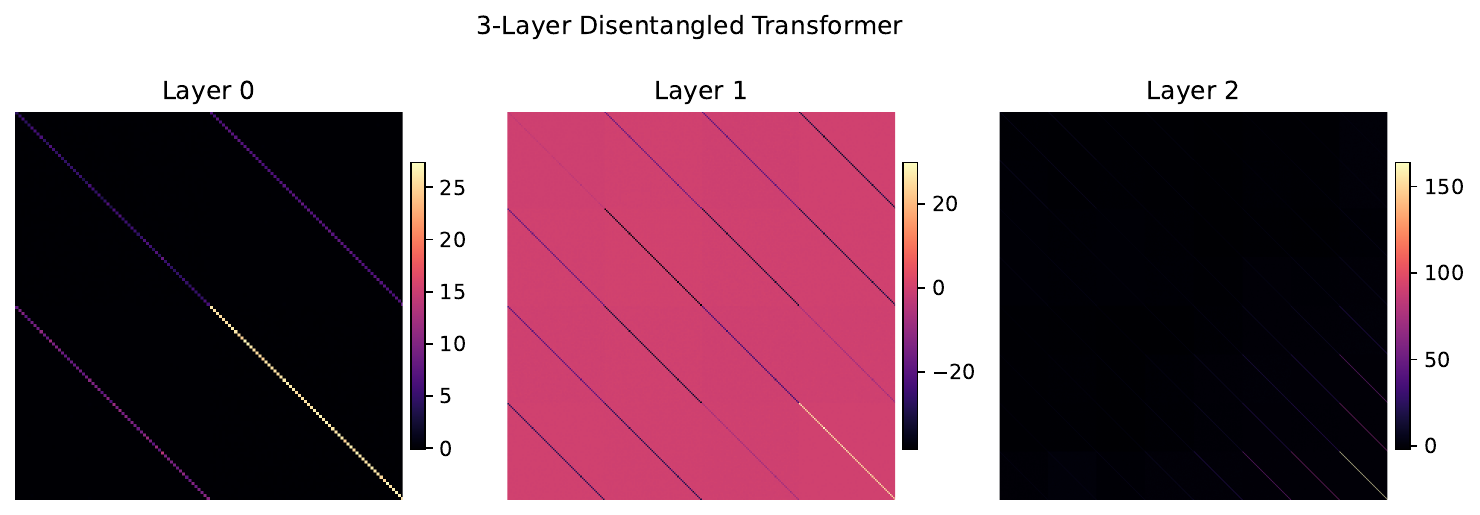}
    \caption{Here we visualize the weights $W_\ell$ learned by a 2-Layer and 3-Layer Disentangled Transformer respectively. All models are randomly initialized \textbf{without} any restriction on parameterization. Resonating \Cref{thm:w_decomposition}, they all converge to a form of $W_\ell = A_\ell \otimes I_n  + B_\ell \otimes J_n$.}
    \label{fig:dt_weights}
\end{figure}

In \Cref{fig:capacity_roberta_2layer}, we show that the capacity theorems (\Cref{thm:capacity}) also transfer to standard 2-layer Transformer models. 

\begin{figure}[!h]
    \centering
    \includegraphics[width=0.5\linewidth]{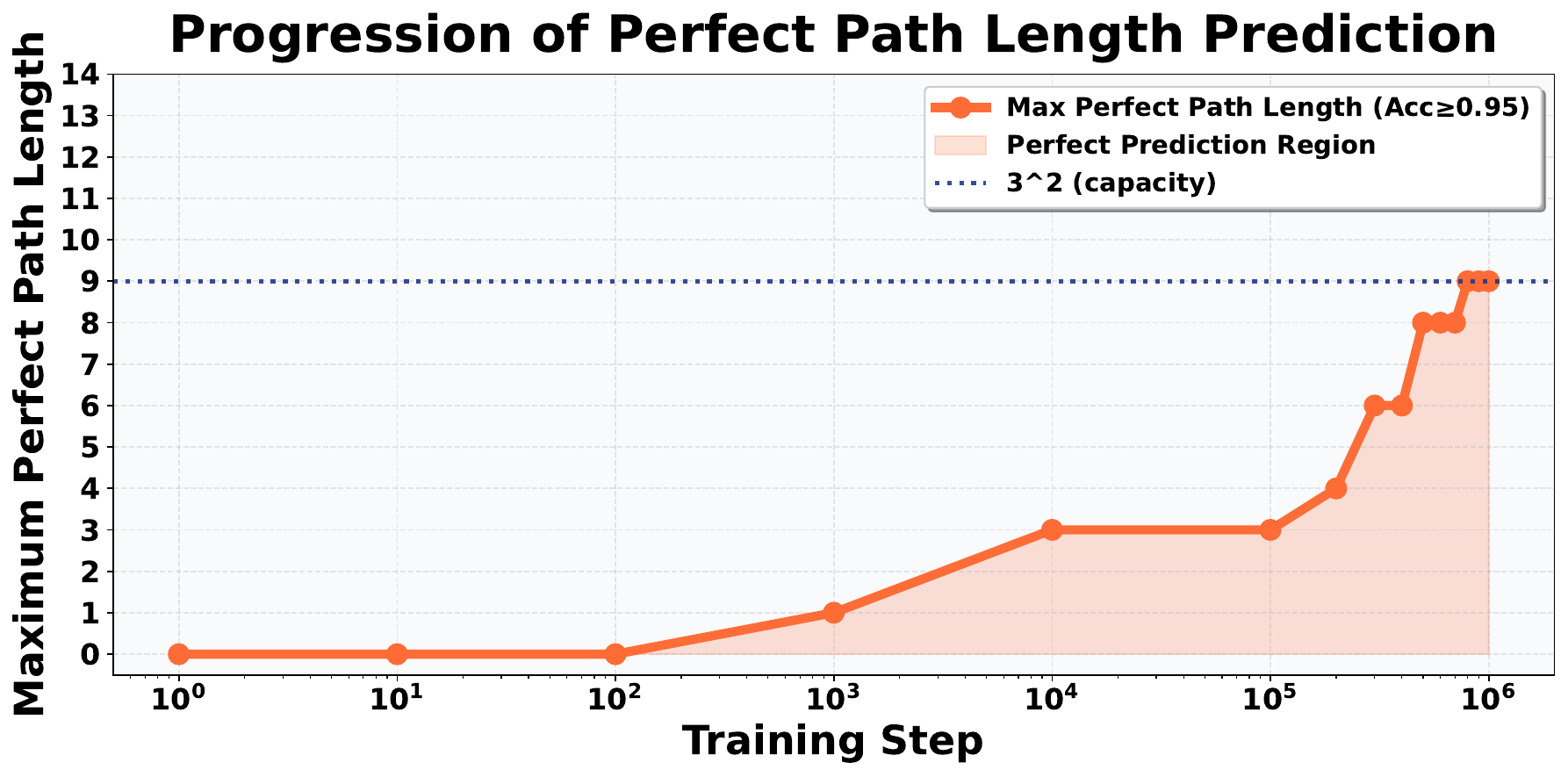}
    \includegraphics[width=0.25\linewidth]{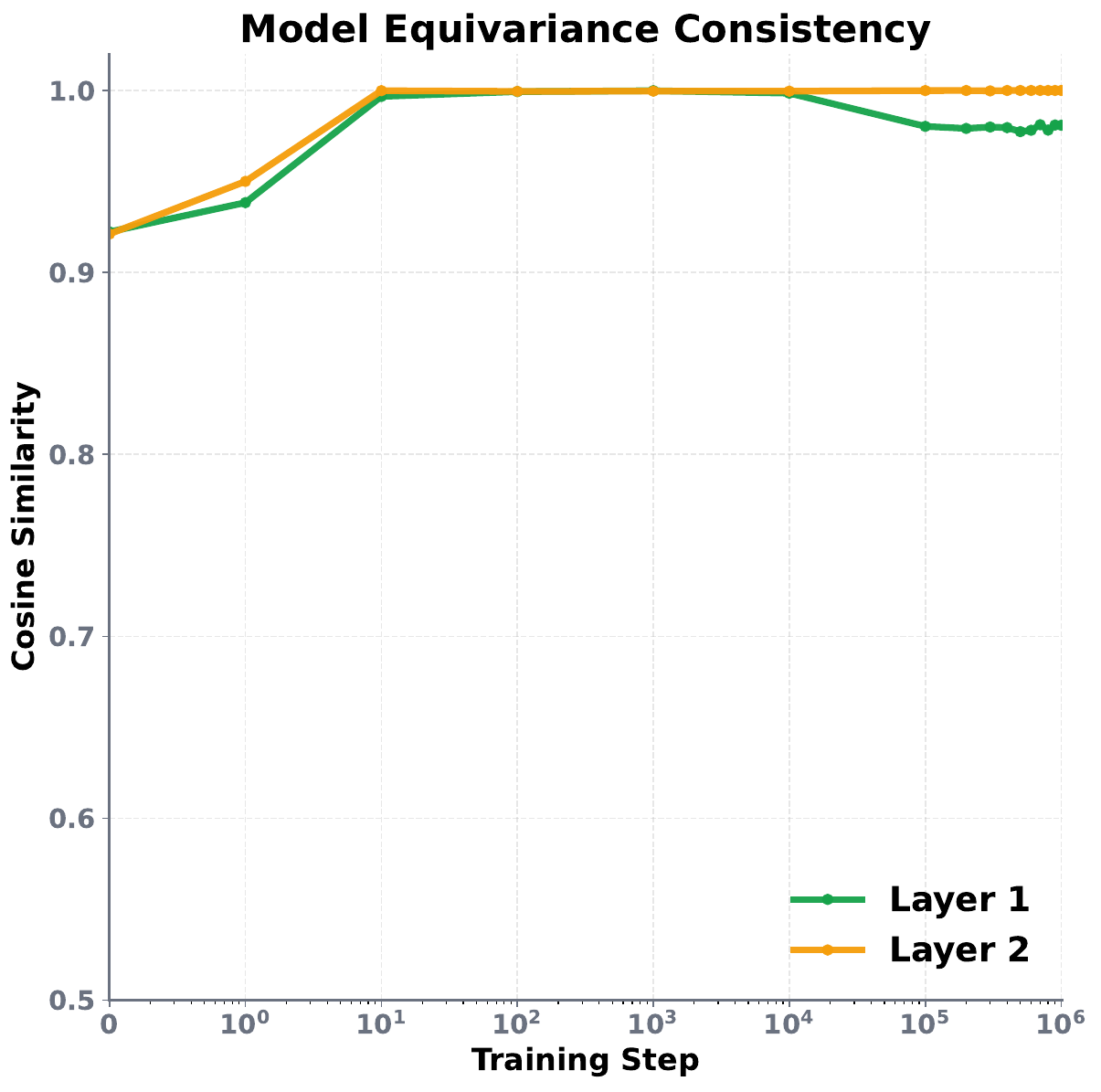}
    \caption{(\textbf{left}) Standard Transformers models studied in \S\ref{ssec:prelim} also hit its capacity wall at $3^L$, showing that our theoretical results transfer beyond the theoretical simplification of Disentangled Transformers. (\textbf{right}) Standard Transformer models also learn an almost layer-wise equivariant solution measured by Eqn. \ref{eqn:metric_equivariance}.}
    \label{fig:capacity_roberta_2layer}
    \vsn
\end{figure}

In \Cref{fig:roberta_restrict_diam_2clique},
we show that when evaluated in $\TClique$ dataset, the one trained with the right data generalize better. 

\begin{figure}[!h]
  \centering
  \includegraphics[width=0.5\linewidth]{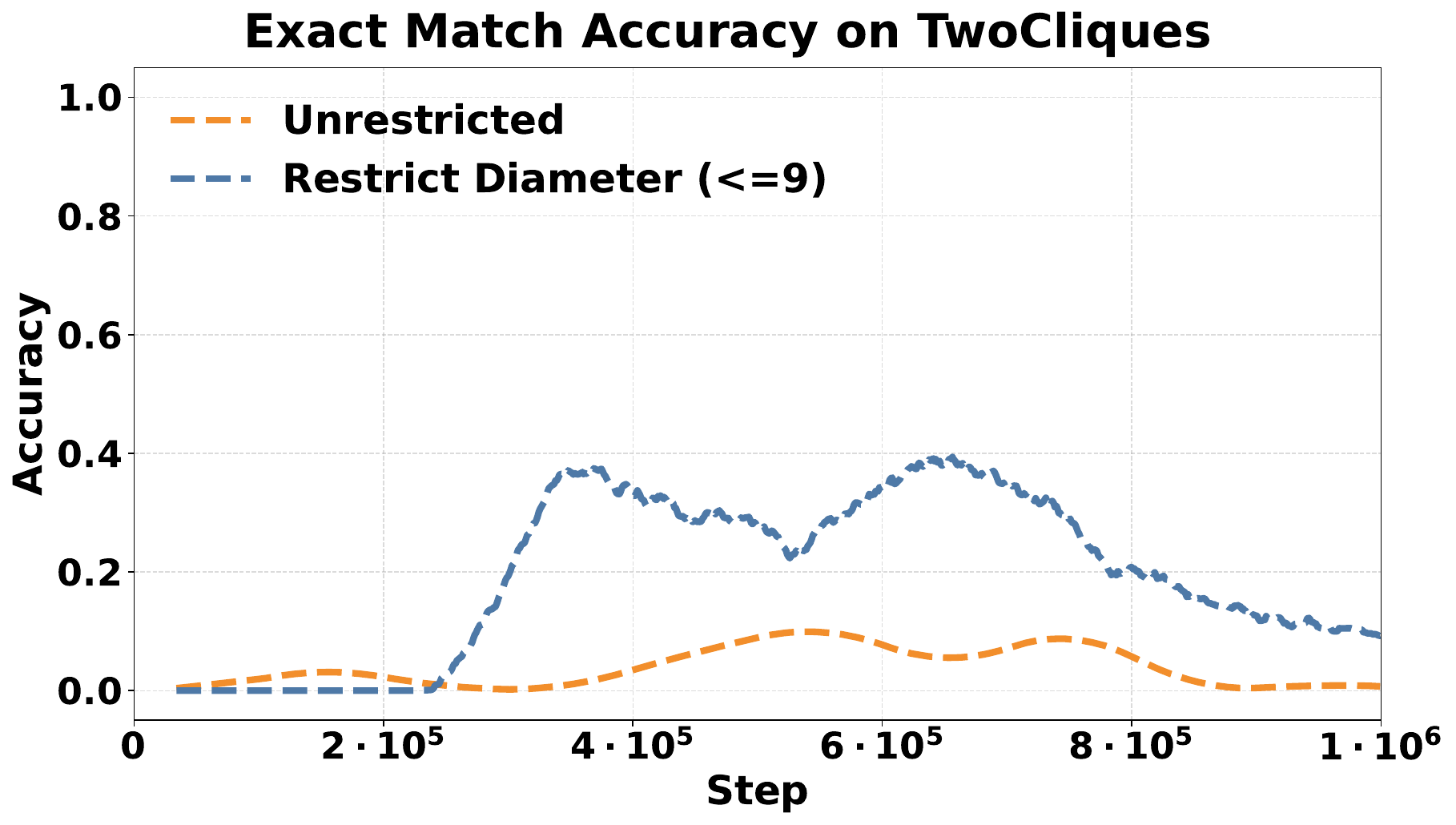}
  \caption{Under the same setup as Fig. \ref{fig:roberta_restrict_diam}, when tested on $\TClique$ graphs, the one trained with \textit{the right data} is able to generalize better.}
 \label{fig:roberta_restrict_diam_2clique}
\end{figure}

\clearpage
\subsection{Scaling Effects of Diameter and Capacity}
\ifpreprint
\else
\begin{figure}[h]
    \centering
    \includegraphics[width=\linewidth]{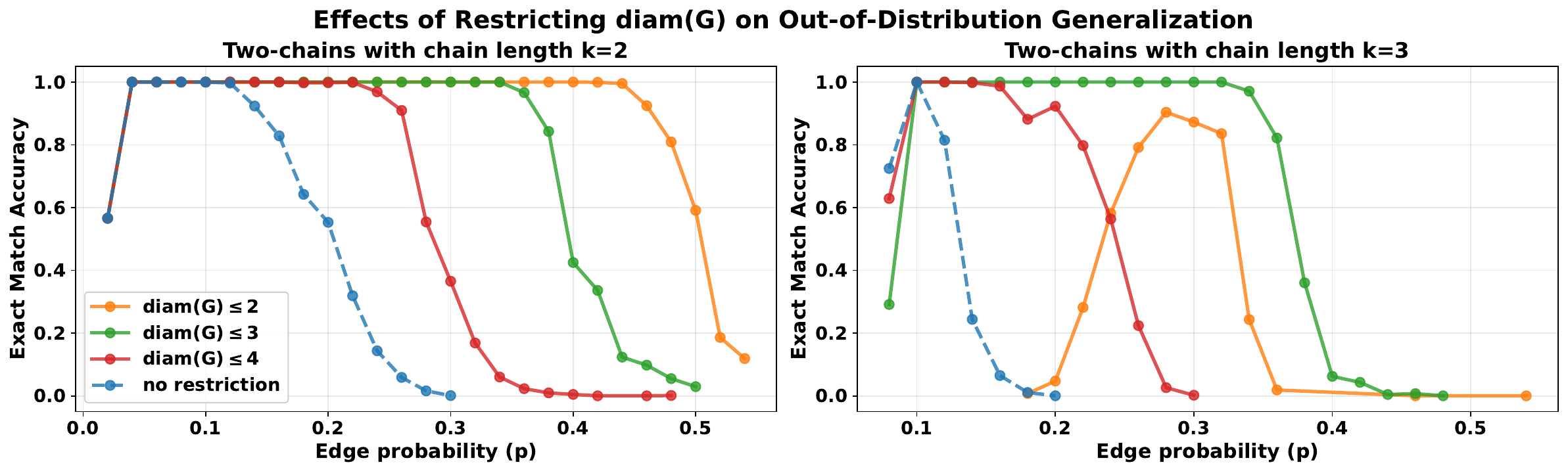}
    \caption{With 1-layer Disentangled Transformers with capacity $\mathsf{Cap} = 3$ following \Cref{thm:capacity}, we vary $d$ such that we restrict our training graphs to have $\diam(G) \leq d$. We also vary the edge probability of our training distribution $\ER(n=8, p = \cdot)$ for generality. We test on $\TChain(n=8, k=\cdot)$ graphs with $k=2$ or $3$ and show the exact match accuracy on configurations where the accuracy is non-zero for readability. We find if the training $d \leq \mathsf{Cap}$, models still learns the algorithmic solution up to problem size $d$ (see $d=2, k=2$ case on the left in {\color{orange}orange}) but \emph{fails to length generalize} (see $d=2, k=3$ in {\color{orange}orange} on the right). On the other hand, if the training $d > \mathsf{Cap}$, model struggles to learn the algorithmic solution (see $d=4$ cases in {\color{red}red} on both $k = 2$ or $3$). The best case overall is when setting $d = \mathsf{Cap} $, i.e., preventing the model from seeing beyond-capacity samples but still preserving at-capacity samples for better generalization. As shown in the {\color{ForestGreen}green} lines, with $d=3$, model achieves balanced testing accuracy on both $k=2$ and $3$.}
    \label{fig:varying_diam}
\end{figure}
\fi

\begin{figure}[h]
    \centering
    \begin{subfigure}{0.66\linewidth}
        \includegraphics[width=\linewidth]{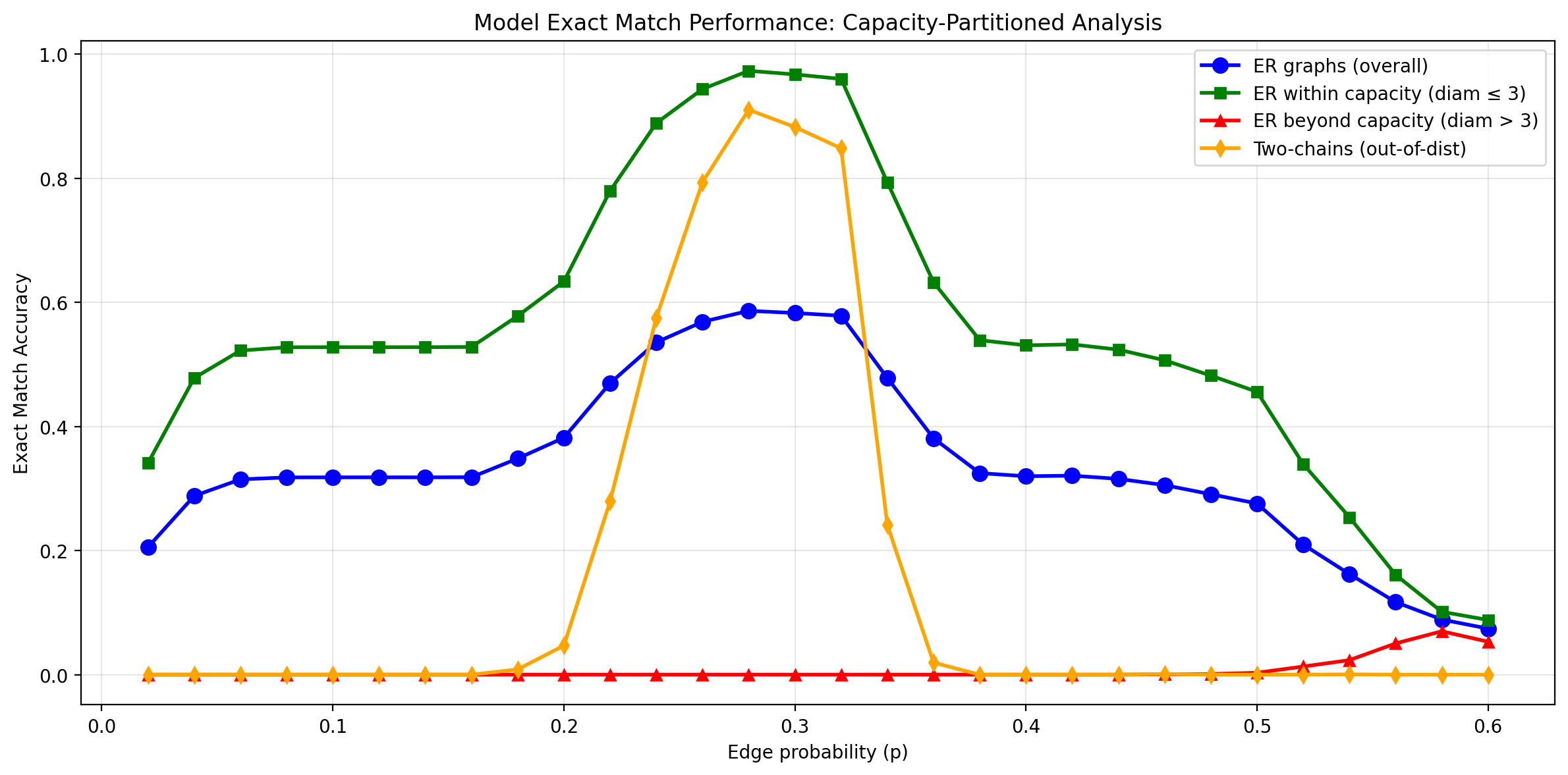}
        \caption{When training 1-layer Disentangled Transformers, instead of restricting training graphs to have diameter at most 3, we restrict {\color{purple}$\mathbf{\diam(G) \leq 2}$} and varying the edge probability in $\ER(n=8, p=p)$ training distribution. When measured by exact match accuracy, restricting $\diam(G) \leq 2$ make the models unable to generalize as well, indicating the importance of \textbf{at-capacity graphs} ($\diam(G) = 3$)}
    \end{subfigure}
    \begin{subfigure}{0.66\linewidth}
        \includegraphics[width=\linewidth]{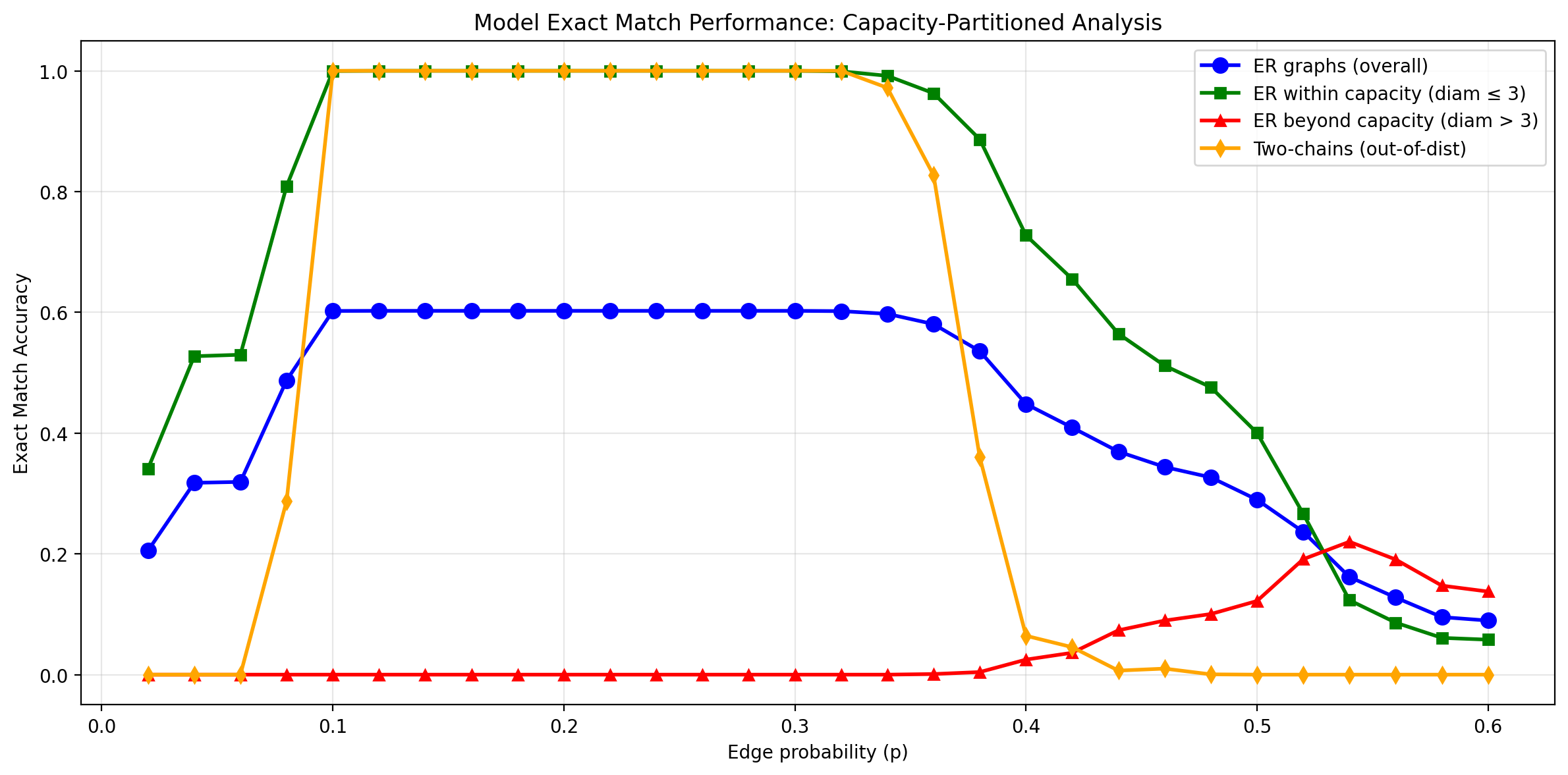}
        \caption{When restricting {\color{purple}$\mathbf{\diam(G) \leq 3}$}, with reasonable $p \in [0.1, 0.32]$, 1-layer Disentangled Transformer can learn the algorithmic channel.}
    \end{subfigure}
    \begin{subfigure}{0.66\linewidth}
        \includegraphics[width=\linewidth]{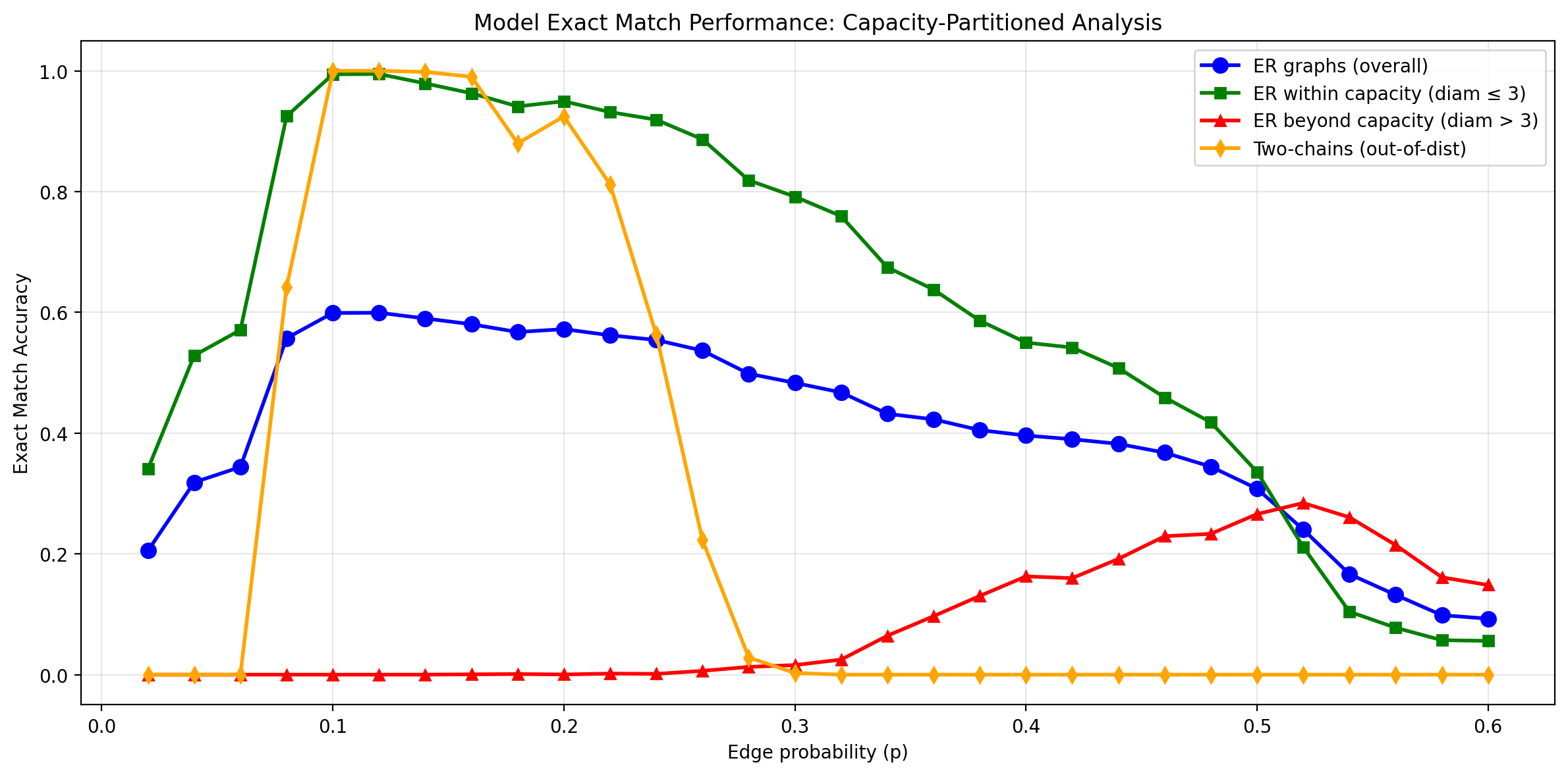}
        \caption{When restricting {\color{purple}$\mathbf{\diam(G) \leq 4}$}, allowing some beyond-capacity graphs, 1-layer Disentangled Transformer struggle to learn the algorithmic channel, and starts to rely on the heuristic $J_n$-channel to make predictions on beyond-capacity graphs (red lines).}
    \end{subfigure}
    \caption{Effects of \textbf{at-capacity graphs} ($\diam(G) = 3^L$) for $L=1$. Without at-capacity graphs, models struggle to learn the algorithmic solution. With beyond-capacity graphs, models weight too much on heuristics. In short, models not only need most graphs within capacity and but also require at-capacity graphs to learn algorithms over heuristics.}
    \label{fig:at-capacity}
\end{figure}

%% file: sections/AE_RelatedWork.tex
\section{Additional Related Work} \label{app:related_work}
\paragraph{Mechanistic Interpretability of Transformers.}
A growing body of work reverse-engineers the \emph{algorithmic circuits} that Transformers learn for tasks like copying, induction, and reasoning \citep{elhage2021framework,olsson2022induction,wang2022ioi,brinkmann2024mechanistic}. These can range from Fourier-style circuits for modular addition \citep{nanda2023grokking,zhou2024fourier} to Newton-like updates for in-context linear regression \citep{fu2024secondorder}. Researchers validate hypotheses by compiling programs into model weights \citep{lindner2023tracr}, decompiling models into code \citep{friedman2023ltp}, and using causal interventions to localize function \citep{chan2022causalscrubbing,meng2022rome,yao2024knowledgecircuits,chang2024localize}. Theoretical work on inductive biases, like a preference for low-sensitivity functions, helps explain why models often favor robust heuristics over exact algorithms \citep{vasudeva2025lowsensitivity}. 

%% file: reference.bib
@inproceedings{nichani2024transformerslearncausalstructure,
    title={How Transformers Learn Causal Structure with Gradient Descent},
    author={Eshaan Nichani and Alex Damian and Jason D. Lee},
    booktitle={Forty-first International Conference on Machine Learning},
    year={2024},
    url={https://openreview.net/forum?id=jNM4imlHZv}
}

@inproceedings{wigderson1992complexity,
    author="Wigderson, Avi",
    editor="Havel, Ivan M.
    and Koubek, V{\'a}clav",
    title="The complexity of graph connectivity",
    booktitle="Mathematical Foundations of Computer Science 1992",
    year="1992",
    publisher="Springer Berlin Heidelberg",
    address="Berlin, Heidelberg",
    pages="112--132",
    isbn="978-3-540-47291-9"
}

@article{warshall1962,
    author = {Warshall, Stephen},
    title = {A Theorem on Boolean Matrices},
    year = {1962},
    issue_date = {Jan. 1962},
    publisher = {Association for Computing Machinery},
    address = {New York, NY, USA},
    volume = {9},
    number = {1},
    issn = {0004-5411},
    url = {https://doi.org/10.1145/321105.321107},
    doi = {10.1145/321105.321107},
    journal = {J. ACM},
    month = jan,
    pages = {11–12},
    numpages = {2}
}

@article{floyd1962,
    author = {Floyd, Robert W.},
    title = {Algorithm 97: Shortest path},
    year = {1962},
    issue_date = {June 1962},
    publisher = {Association for Computing Machinery},
    address = {New York, NY, USA},
    volume = {5},
    number = {6},
    issn = {0001-0782},
    url = {https://doi.org/10.1145/367766.368168},
    doi = {10.1145/367766.368168},
    journal = {Commun. ACM},
    month = jun,
    pages = {345},
    numpages = {1}
}

@inproceedings{niven-kao-2019-probing,
    title = "Probing Neural Network Comprehension of Natural Language Arguments",
    author = "Niven, Timothy  and
      Kao, Hung-Yu",
    editor = "Korhonen, Anna  and
      Traum, David  and
      M{\`a}rquez, Llu{\'i}s",
    booktitle = "Proceedings of the 57th Annual Meeting of the Association for Computational Linguistics",
    month = jul,
    year = "2019",
    address = "Florence, Italy",
    publisher = "Association for Computational Linguistics",
    url = "https://aclanthology.org/P19-1459/",
    doi = "10.18653/v1/P19-1459",
    pages = "4658--4664"
}

@article{geirhos2020shortcutlear,
  title={Shortcut learning in deep neural networks},
  author={Robert Geirhos and J{\"o}rn-Henrik Jacobsen and Claudio Michaelis and Richard S. Zemel and Wieland Brendel and Matthias Bethge and Felix Wichmann},
  journal={Nature Machine Intelligence},
  year={2020},
  volume={2},
  pages={665 - 673},
  url={https://www.nature.com/articles/s42256-020-00257-z}
}

@inproceedings{saxton2018analysing,
title={Analysing Mathematical Reasoning Abilities of Neural Models},
author={David Saxton and Edward Grefenstette and Felix Hill and Pushmeet Kohli},
booktitle={International Conference on Learning Representations},
year={2019},
url={https://openreview.net/forum?id=H1gR5iR5FX},
}

@inproceedings{tang2023lazylearners,
    title = "Large Language Models Can be Lazy Learners: Analyze Shortcuts in In-Context Learning",
    author = "Tang, Ruixiang  and
      Kong, Dehan  and
      Huang, Longtao  and
      Xue, Hui",
    editor = "Rogers, Anna  and
      Boyd-Graber, Jordan  and
      Okazaki, Naoaki",
    booktitle = "Findings of the Association for Computational Linguistics: ACL 2023",
    month = jul,
    year = "2023",
    address = "Toronto, Canada",
    publisher = "Association for Computational Linguistics",
    url = "https://aclanthology.org/2023.findings-acl.284/",
    doi = "10.18653/v1/2023.findings-acl.284",
    pages = "4645--4657"
}

@inproceedings{yuan2024shortcutsuite,
    title = "Do {LLM}s Overcome Shortcut Learning? An Evaluation of Shortcut Challenges in Large Language Models",
    author = "Yuan, Yu  and
      Zhao, Lili  and
      Zhang, Kai  and
      Zheng, Guangting  and
      Liu, Qi",
    editor = "Al-Onaizan, Yaser  and
      Bansal, Mohit  and
      Chen, Yun-Nung",
    booktitle = "Proceedings of the 2024 Conference on Empirical Methods in Natural Language Processing",
    month = nov,
    year = "2024",
    address = "Miami, Florida, USA",
    publisher = "Association for Computational Linguistics",
    url = "https://aclanthology.org/2024.emnlp-main.679/",
    doi = "10.18653/v1/2024.emnlp-main.679",
    pages = "12188--12200"
}

@inproceedings{zhou2024conceptspurious,
    title = "Explore Spurious Correlations at the Concept Level in Language Models for Text Classification",
    author = "Zhou, Yuhang  and
      Xu, Paiheng  and
      Liu, Xiaoyu  and
      An, Bang  and
      Ai, Wei  and
      Huang, Furong",
    editor = "Ku, Lun-Wei  and
      Martins, Andre  and
      Srikumar, Vivek",
    booktitle = "Proceedings of the 62nd Annual Meeting of the Association for Computational Linguistics (Volume 1: Long Papers)",
    month = aug,
    year = "2024",
    address = "Bangkok, Thailand",
    publisher = "Association for Computational Linguistics",
    url = "https://aclanthology.org/2024.acl-long.28/",
    doi = "10.18653/v1/2024.acl-long.28",
    pages = "478--492"
}

@misc{ye2024spurioussurvey,
      title={The Clever Hans Mirage: A Comprehensive Survey on Spurious Correlations in Machine Learning}, 
      author={Wenqian Ye and Luyang Jiang and Eric Xie and Guangtao Zheng and Yunsheng Ma and Xu Cao and Dongliang Guo and Daiqing Qi and Zeyu He and Yijun Tian and Megan Coffee and Zhe Zeng and Sheng Li and Ting-hao and Huang and Ziran Wang and James M. Rehg and Henry Kautz and Aidong Zhang},
      year={2025},
      eprint={2402.12715},
      archivePrefix={arXiv},
      primaryClass={cs.LG},
      url={https://arxiv.org/abs/2402.12715}, 
}

@misc{zou2023universalattack,
      title={Universal and Transferable Adversarial Attacks on Aligned Language Models}, 
      author={Andy Zou and Zifan Wang and Nicholas Carlini and Milad Nasr and J. Zico Kolter and Matt Fredrikson},
      year={2023},
      eprint={2307.15043},
      archivePrefix={arXiv},
      primaryClass={cs.CL},
      url={https://arxiv.org/abs/2307.15043}, 
}

@inproceedings{deng2024contamination,
    title = "Investigating Data Contamination in Modern Benchmarks for Large Language Models",
    author = "Deng, Chunyuan  and
      Zhao, Yilun  and
      Tang, Xiangru  and
      Gerstein, Mark  and
      Cohan, Arman",
    editor = "Duh, Kevin  and
      Gomez, Helena  and
      Bethard, Steven",
    booktitle = "Proceedings of the 2024 Conference of the North American Chapter of the Association for Computational Linguistics: Human Language Technologies (Volume 1: Long Papers)",
    month = jun,
    year = "2024",
    address = "Mexico City, Mexico",
    publisher = "Association for Computational Linguistics",
    url = "https://aclanthology.org/2024.naacl-long.482/",
    doi = "10.18653/v1/2024.naacl-long.482",
    pages = "8706--8719"
}

@inproceedings{li2024latesteval,
author = {Li, Yucheng and Guerin, Frank and Lin, Chenghua},
title = {LatestEval: addressing data contamination in language model evaluation through dynamic and time-sensitive test construction},
year = {2024},
isbn = {978-1-57735-887-9},
publisher = {AAAI Press},
url = {https://doi.org/10.1609/aaai.v38i17.29822},
doi = {10.1609/aaai.v38i17.29822},
booktitle = {Proceedings of the Thirty-Eighth AAAI Conference on Artificial Intelligence and Thirty-Sixth Conference on Innovative Applications of Artificial Intelligence and Fourteenth Symposium on Educational Advances in Artificial Intelligence},
articleno = {2074},
numpages = {8},
series = {AAAI'24/IAAI'24/EAAI'24}
}

@inproceedings{kao2024complexmath,
    title = "Solving for {X} and Beyond: Can Large Language Models Solve Complex Math Problems with More-Than-Two Unknowns?",
    author = "Kao, Kuei-Chun  and
      Wang, Ruochen  and
      Hsieh, Cho-Jui",
    booktitle = "Findings of the Association for Computational Linguistics: EMNLP 2024",
    month = nov,
    year = "2024",
    address = "Miami, Florida, USA",
    publisher = "Association for Computational Linguistics",
    url = "https://aclanthology.org/2024.findings-emnlp.980/",
    doi = "10.18653/v1/2024.findings-emnlp.980",
    pages = "16821--16843"
}

@inproceedings{
zhou2025mathforai,
title={Math for {AI}: On the Generalization of Learning Mathematical Problem Solving},
author={Ruochen Zhou and Minrui Xu and Shiqi Chen and Junteng Liu and Yunqi Li and LIN Xinxin and Zhengyu Chen and Junxian He},
booktitle={The 4th Workshop on Mathematical Reasoning and AI at NeurIPS'24},
year={2024},
url={https://openreview.net/forum?id=xlnvZ85CSo}
}

@inproceedings{saparov2025search,
    title={Transformers Struggle to Learn to Search},
    author={Abulhair Saparov and Srushti Ajay Pawar and Shreyas Pimpalgaonkar and Nitish Joshi and Richard Yuanzhe Pang and Vishakh Padmakumar and Mehran Kazemi and Najoung Kim and He He},
    booktitle={The Thirteenth International Conference on Learning Representations},
    year={2025},
    url={https://openreview.net/forum?id=9cQB1Hwrtw}
}

@article{hahn2020,
    title = "Theoretical Limitations of Self-Attention in Neural Sequence Models",
    author = "Hahn, Michael",
    editor = "Johnson, Mark  and
      Roark, Brian  and
      Nenkova, Ani",
    journal = "Transactions of the Association for Computational Linguistics",
    volume = "8",
    year = "2020",
    address = "Cambridge, MA",
    publisher = "MIT Press",
    url = "https://aclanthology.org/2020.tacl-1.11/",
    doi = "10.1162/tacl_a_00306",
    pages = "156--171"
}

@inproceedings{merrill2025logdepth,
title={A Little Depth Goes a Long Way: The Expressive Power of Log-Depth Transformers},
author={William Merrill and Ashish Sabharwal},
booktitle={The Thirty-ninth Annual Conference on Neural Information Processing Systems},
year={2025},
url={https://openreview.net/forum?id=5pHfYe10iX}
}

@misc{liu2019roberta,
      title={RoBERTa: A Robustly Optimized BERT Pretraining Approach}, 
      author={Yinhan Liu and Myle Ott and Naman Goyal and Jingfei Du and Mandar Joshi and Danqi Chen and Omer Levy and Mike Lewis and Luke Zettlemoyer and Veselin Stoyanov},
      year={2019},
      eprint={1907.11692},
      archivePrefix={arXiv},
      primaryClass={cs.CL},
      url={https://arxiv.org/abs/1907.11692}, 
}

@inproceedings{yun2020universal,
    title={Are Transformers universal approximators of sequence-to-sequence functions?},
    author={Chulhee Yun and Srinadh Bhojanapalli and Ankit Singh Rawat and Sashank Reddi and Sanjiv Kumar},
    booktitle={International Conference on Learning Representations},
    year={2020},
    url={https://openreview.net/forum?id=ByxRM0Ntvr}
}

@article{merrill2023parallelism,
    title = "The Parallelism Tradeoff: Limitations of Log-Precision Transformers",
    author = "Merrill, William  and
      Sabharwal, Ashish",
    journal = "Transactions of the Association for Computational Linguistics",
    volume = "11",
    year = "2023",
    address = "Cambridge, MA",
    publisher = "MIT Press",
    url = "https://aclanthology.org/2023.tacl-1.31/",
    doi = "10.1162/tacl_a_00562",
    pages = "531--545"
}

@article{hao2022hardattention,
    title = "Formal Language Recognition by Hard Attention Transformers: Perspectives from Circuit Complexity",
    author = "Hao, Yiding  and
      Angluin, Dana  and
      Frank, Robert",
    editor = "Roark, Brian  and
      Nenkova, Ani",
    journal = "Transactions of the Association for Computational Linguistics",
    volume = "10",
    year = "2022",
    address = "Cambridge, MA",
    publisher = "MIT Press",
    url = "https://aclanthology.org/2022.tacl-1.46/",
    doi = "10.1162/tacl_a_00490",
    pages = "800--810"
}

@inproceedings{barcelo2024hardattention,
title={Logical Languages Accepted by Transformer Encoders with Hard Attention},
author={Pablo Barcelo and Alexander Kozachinskiy and Anthony Widjaja Lin and Vladimir Podolskii},
booktitle={The Twelfth International Conference on Learning Representations},
year={2024},
url={https://openreview.net/forum?id=gbrHZq07mq}
}

@inproceedings{weiss2021rasp,
  title={Thinking like transformers},
  author={Weiss, Gail and Goldberg, Yoav and Yahav, Eran},
  booktitle={International Conference on Machine Learning},
  pages={11080--11090},
  year={2021},
  organization={PMLR}
}

@inproceedings{zhou2023lengthgen,
    title={What Algorithms can Transformers Learn? A Study in Length Generalization},
    author={Hattie Zhou and Arwen Bradley and Etai Littwin and Noam Razin and Omid Saremi and Joshua M. Susskind and Samy Bengio and Preetum Nakkiran},
    booktitle={The Twelfth International Conference on Learning Representations},
    year={2024},
    url={https://openreview.net/forum?id=AssIuHnmHX}
}

@inproceedings{merrill2024cot,
    title={The Expressive Power of Transformers with Chain of Thought},
    author={William Merrill and Ashish Sabharwal},
    booktitle={The Twelfth International Conference on Learning Representations},
    year={2024},
    url={https://openreview.net/forum?id=NjNGlPh8Wh}
}

@inproceedings{fu2024secondorder,
    title={Transformers Learn to Achieve Second-Order Convergence Rates for In-Context Linear Regression},
    author={Deqing Fu and Tian-Qi Chen and Robin Jia and Vatsal Sharan},
    booktitle={The Thirty-eighth Annual Conference on Neural Information Processing Systems},
    year={2024},
    url={https://openreview.net/forum?id=L8h6cozcbn}
}

@inproceedings{zhou2024fourier,
    title={Pre-trained Large Language Models Use Fourier Features to Compute Addition},
    author={Tianyi Zhou and Deqing Fu and Vatsal Sharan and Robin Jia},
    booktitle={The Thirty-eighth Annual Conference on Neural Information Processing Systems},
    year={2024},
    url={https://openreview.net/forum?id=i4MutM2TZb}
}

@inproceedings{vasudeva2025lowsensitivity,
    title={Transformers Learn Low Sensitivity Functions: Investigations and Implications},
    author={Bhavya Vasudeva and Deqing Fu and Tianyi Zhou and Elliott Kau and Youqi Huang and Vatsal Sharan},
    booktitle={The Thirteenth International Conference on Learning Representations},
    year={2025},
    url={https://openreview.net/forum?id=4ikjWBs3tE}
}

@article{elhage2021framework,
   title={A Mathematical Framework for Transformer Circuits},
   author={Elhage, Nelson and Nanda, Neel and Olsson, Catherine and Henighan, Tom and Joseph, Nicholas and Mann, Ben and Askell, Amanda and Bai, Yuntao and Chen, Anna and Conerly, Tom and DasSarma, Nova and Drain, Dawn and Ganguli, Deep and Hatfield-Dodds, Zac and Hernandez, Danny and Jones, Andy and Kernion, Jackson and Lovitt, Liane and Ndousse, Kamal and Amodei, Dario and Brown, Tom and Clark, Jack and Kaplan, Jared and McCandlish, Sam and Olah, Chris},
   year={2021},
   journal={Transformer Circuits Thread},
   note={https://transformer-circuits.pub/2021/framework/index.html}
}

@misc{olsson2022induction,
   title={In-context Learning and Induction Heads},
   author={Olsson, Catherine and Elhage, Nelson and Nanda, Neel and Joseph, Nicholas and DasSarma, Nova and Henighan, Tom and Mann, Ben and Askell, Amanda and Bai, Yuntao and Chen, Anna and Conerly, Tom and Drain, Dawn and Ganguli, Deep and Hatfield-Dodds, Zac and Hernandez, Danny and Johnston, Scott and Jones, Andy and Kernion, Jackson and Lovitt, Liane and Ndousse, Kamal and Amodei, Dario and Brown, Tom and Clark, Jack and Kaplan, Jared and McCandlish, Sam and Olah, Chris},
   year={2022},
   journal={Transformer Circuits Thread},
   note={https://transformer-circuits.pub/2022/in-context-learning-and-induction-heads/index.html}
}

@inproceedings{lindner2023tracr,
    title={Tracr: Compiled Transformers as a Laboratory for Interpretability},
    author={David Lindner and Janos Kramar and Sebastian Farquhar and Matthew Rahtz and Thomas McGrath and Vladimir Mikulik},
    booktitle={Thirty-seventh Conference on Neural Information Processing Systems},
    year={2023},
    url={https://openreview.net/forum?id=tbbId8u7nP}
}

@inproceedings{meng2022rome,
title={Locating and Editing Factual Associations in {GPT}},
author={Kevin Meng and David Bau and Alex J Andonian and Yonatan Belinkov},
booktitle={Advances in Neural Information Processing Systems},
editor={Alice H. Oh and Alekh Agarwal and Danielle Belgrave and Kyunghyun Cho},
year={2022},
url={https://openreview.net/forum?id=-h6WAS6eE4}
}

@inproceedings{
wang2022ioi,
title={Interpretability in the Wild: a Circuit for Indirect Object Identification in {GPT}-2 Small},
author={Kevin Ro Wang and Alexandre Variengien and Arthur Conmy and Buck Shlegeris and Jacob Steinhardt},
booktitle={The Eleventh International Conference on Learning Representations },
year={2023},
url={https://openreview.net/forum?id=NpsVSN6o4ul}
}

@inproceedings{brinkmann2024mechanistic,
    title = "A Mechanistic Analysis of a Transformer Trained on a Symbolic Multi-Step Reasoning Task",
    author = "Brinkmann, Jannik  and
      Sheshadri, Abhay  and
      Levoso, Victor  and
      Swoboda, Paul  and
      Bartelt, Christian",
    editor = "Ku, Lun-Wei  and
      Martins, Andre  and
      Srikumar, Vivek",
    booktitle = "Findings of the Association for Computational Linguistics: ACL 2024",
    month = aug,
    year = "2024",
    address = "Bangkok, Thailand",
    publisher = "Association for Computational Linguistics",
    url = "https://aclanthology.org/2024.findings-acl.242/",
    doi = "10.18653/v1/2024.findings-acl.242",
    pages = "4082--4102"
}

@inproceedings{friedman2023ltp,
    title={Learning Transformer Programs},
    author={Dan Friedman and Alexander Wettig and Danqi Chen},
    booktitle={Thirty-seventh Conference on Neural Information Processing Systems},
    year={2023},
    url={https://openreview.net/forum?id=Pe9WxkN8Ff}
}

@article{chan2022causalscrubbing,
	title={Causal scrubbing, a method for rigorously testing interpretability hypotheses},
	author={Chan, Lawrence and Garriga-Alonso, Adrià and Goldwosky-Dill, Nicholas and Greenblatt, Ryan and Nitishinskaya, Jenny and Radhakrishnan, Ansh and Shlegeris, Buck and Thomas, Nate},
	year={2022},
	journal={AI Alignment Forum},
	note={\url{https://www.alignmentforum.org/posts/JvZhhzycHu2Yd57RN/causal-scrubbing-a-method-for-rigorously-testing}}
}

@inproceedings{yao2024knowledgecircuits,
    title={Knowledge Circuits in Pretrained Transformers},
    author={Yunzhi Yao and Ningyu Zhang and Zekun Xi and Mengru Wang and Ziwen Xu and Shumin Deng and Huajun Chen},
    booktitle={The Thirty-eighth Annual Conference on Neural Information Processing Systems},
    year={2024},
    url={https://openreview.net/forum?id=YVXzZNxcag}
}

@inproceedings{nanda2023grokking,
title={Progress measures for grokking via mechanistic interpretability},
author={Neel Nanda and Lawrence Chan and Tom Lieberum and Jess Smith and Jacob Steinhardt},
booktitle={The Eleventh International Conference on Learning Representations },
year={2023},
url={https://openreview.net/forum?id=9XFSbDPmdW}
}

@inproceedings{chang2024localize,
    title = "Do Localization Methods Actually Localize Memorized Data in {LLM}s? A Tale of Two Benchmarks",
    author = "Chang, Ting-Yun  and
      Thomason, Jesse  and
      Jia, Robin",
    editor = "Duh, Kevin  and
      Gomez, Helena  and
      Bethard, Steven",
    booktitle = "Proceedings of the 2024 Conference of the North American Chapter of the Association for Computational Linguistics: Human Language Technologies (Volume 1: Long Papers)",
    month = jun,
    year = "2024",
    address = "Mexico City, Mexico",
    publisher = "Association for Computational Linguistics",
    url = "https://aclanthology.org/2024.naacl-long.176/",
    doi = "10.18653/v1/2024.naacl-long.176",
    pages = "3190--3211"
}

@inproceedings{
chen2024unveiling,
title={Unveiling Induction Heads: Provable Training Dynamics and Feature Learning in Transformers},
author={Siyu Chen and Heejune Sheen and Tianhao Wang and Zhuoran Yang},
booktitle={The Thirty-eighth Annual Conference on Neural Information Processing Systems},
year={2024},
url={https://openreview.net/forum?id=4fN2REs0Ma}
}

@inproceedings{fu2024isobench,
    title={IsoBench: Benchmarking Multimodal Foundation Models on Isomorphic Representations},
    author={Deqing Fu and Ruohao Guo and Ghazal Khalighinejad and Ollie Liu and Bhuwan Dhingra and Dani Yogatama and Robin Jia and Willie Neiswanger},
    booktitle={First Conference on Language Modeling},
    year={2024},
    url={https://openreview.net/forum?id=KZd1EErRJ1}
}

@inproceedings{sanford2024graphreasoning,
title={Understanding Transformer Reasoning Capabilities via Graph Algorithms},
author={Clayton Sanford and Bahare Fatemi and Ethan Hall and Anton Tsitsulin and Mehran Kazemi and Jonathan Halcrow and Bryan Perozzi and Vahab Mirrokni},
booktitle={The Thirty-eighth Annual Conference on Neural Information Processing Systems},
year={2024},
url={https://openreview.net/forum?id=AfzbDw6DSp}
}

@article{Attouch2013ConvergenceOD,
  title={Convergence of descent methods for semi-algebraic and tame problems: proximal algorithms, forward–backward splitting, and regularized Gauss–Seidel methods},
  author={H{\'e}dy Attouch and J{\'e}r{\^o}me Bolte and Benar Fux Svaiter},
  journal={Mathematical Programming},
  year={2013},
  volume={137},
  pages={91-129},
  url={https://api.semanticscholar.org/CorpusID:6597608}
}

@book{Nesterov2004IntroductoryLO,
  author    = {Yurii Nesterov},
  title     = {Introductory Lectures on Convex Optimization: A Basic Course},
  series    = {Applied Optimization},
  volume    = {87},
  publisher = {Springer},
  address   = {New York, NY},
  year      = {2004},
  doi       = {10.1007/978-1-4419-8853-9},
  isbn      = {978-1-4020-7553-7},
  url       = {https://link.springer.com/book/10.1007/978-1-4419-8853-9}
}

@article{Bertsekas01031997,
author = {D P Bertsekas},
title = {Nonlinear Programming},
journal = {Journal of the Operational Research Society},
volume = {48},
number = {3},
pages = {334--334},
year = {1997},
publisher = {Taylor \& Francis},
doi = {10.1057/palgrave.jors.2600425},
URL = {https://doi.org/10.1057/palgrave.jors.2600425},
eprint ={https://doi.org/10.1057/palgrave.jors.2600425}
}

@book{Beck2017,
author = {Beck, Amir},
title = {First-Order Methods in Optimization},
publisher = {Society for Industrial and Applied Mathematics},
year = {2017},
doi = {10.1137/1.9781611974997},
address = {Philadelphia, PA},
edition   = {},
URL = {https://epubs.siam.org/doi/abs/10.1137/1.9781611974997},
eprint = {https://epubs.siam.org/doi/pdf/10.1137/1.9781611974997}
}

@inproceedings{
feng2023towards,
title={Towards Revealing the Mystery behind Chain of Thought: A Theoretical Perspective},
author={Guhao Feng and Bohang Zhang and Yuntian Gu and Haotian Ye and Di He and Liwei Wang},
booktitle={Thirty-seventh Conference on Neural Information Processing Systems},
year={2023},
url={https://openreview.net/forum?id=qHrADgAdYu}
}

@book{hamilton2020graph,
  title={Graph representation learning},
  author={Hamilton, William L},
  year={2020},
  publisher={Morgan \& Claypool Publishers}
}

@misc{mccoy2023embersllm,
      title={Embers of Autoregression: Understanding Large Language Models Through the Problem They are Trained to Solve}, 
      author={R. Thomas McCoy and Shunyu Yao and Dan Friedman and Matthew Hardy and Thomas L. Griffiths},
      year={2023},
      eprint={2309.13638},
      archivePrefix={arXiv},
      primaryClass={cs.CL},
      url={https://arxiv.org/abs/2309.13638}, 
}

@inproceedings{
mirzadeh2025gsmsymbolic,
title={{GSM}-Symbolic: Understanding the Limitations of Mathematical Reasoning in Large Language Models},
author={Seyed Iman Mirzadeh and Keivan Alizadeh and Hooman Shahrokhi and Oncel Tuzel and Samy Bengio and Mehrdad Farajtabar},
booktitle={The Thirteenth International Conference on Learning Representations},
year={2025},
url={https://openreview.net/forum?id=AjXkRZIvjB}
}
